\batchmode
\makeatletter
\def\input@path{{/Users/alexsherstinsky/Documents/RNNAndLSTMFundamentals/Elsevier/ElsevierPhysicaDNonlinearPhenomenaJournalSpecialIssueOnMachineLearningAndDynamicalSystems2020/Revision/arXiv07312023/}}
\makeatother
\documentclass[conference,onecolumn,sort]{IEEEtran}
\usepackage[T1]{fontenc}
\usepackage[latin9]{inputenc}
\usepackage{amsmath}
\usepackage{amsthm}
\usepackage{amssymb}
\usepackage{graphicx}
\usepackage[numbers]{natbib}
\usepackage[pdftex,unicode=true,
 bookmarks=true,bookmarksnumbered=true,bookmarksopen=true,bookmarksopenlevel=1,
 breaklinks=false,pdfborder={0 0 0},pdfborderstyle={},backref=false,colorlinks=false]
 {hyperref}
\hypersetup{pdftitle={Fundamentals of Recurrent Neural Network (RNN) and Long Short-Term Memory (LSTM) Network},
 pdfauthor={Alex Sherstinsky},
 pdfpagelayout=OneColumn, pdfnewwindow=true, pdfstartview=XYZ, plainpages=false}

\makeatletter
\theoremstyle{definition}
\newtheorem{example}{\protect\examplename}
\theoremstyle{plain}
\newtheorem{prop}{\protect\propositionname}
\ifx\proof\undefined\
  \newenvironment{proof}[1][\proofname]{\par
    \normalfont\topsep6\p@\@plus6\p@\relax
    \trivlist
    \itemindent\parindent
    \item[\hskip\labelsep
          \scshape
      #1]\ignorespaces
  }{%
    \endtrivlist\@endpefalse
  }
  \providecommand{\proofname}{Proof}
\fi
\theoremstyle{remark}
\newtheorem{rem}{\protect\remarkname}

\usepackage[caption=false,font=footnotesize]{subfig}
\usepackage{mathtools}
\usepackage{scrlayer-scrpage}
\clearpairofpagestyles
\@ifundefined{showcaptionsetup}{}{%
 \PassOptionsToPackage{caption=false}{subfig}}
\usepackage{subfig}
\makeatother

\providecommand{\examplename}{Example}
\providecommand{\propositionname}{Proposition}
\providecommand{\remarkname}{Remark}

\begin{document}
\title{%
\begin{minipage}[t]{1\columnwidth}%
\texttt{\small{}Published in the Elsevier journal \textquotedblleft Physica
D: Nonlinear Phenomena\textquotedblright , Volume 404, March 2020:}\vspace{-5mm}
\texttt{\small{}\\Special Issue on Machine Learning and Dynamical
Systems}\vspace{5mm}
\\Fundamentals of Recurrent Neural Network (RNN) and Long Short-Term
Memory (LSTM) Network\vspace{-5mm}
\end{minipage}}
\author{\IEEEauthorblockN{Alex~Sherstinsky}\IEEEauthorblockA{shers@alum.mit.edu\\
https://www.linkedin.com/in/alexsherstinsky}}
\IEEEspecialpapernotice{}
\IEEEaftertitletext{}
\maketitle
\begin{abstract}
Because of their effectiveness in broad practical applications, LSTM
networks have received a wealth of coverage in scientific journals,
technical blogs, and implementation guides. However, in most articles,
the inference formulas for the LSTM network and its parent, RNN, are
stated axiomatically, while the training formulas are omitted altogether.
In addition, the technique of ``unrolling'' an RNN is routinely
presented without justification throughout the literature. The goal
of this tutorial is to explain the essential RNN and LSTM fundamentals
in a single document. Drawing from concepts in Signal Processing,
we formally derive the canonical RNN formulation from differential
equations. We then propose and prove a precise statement, which yields
the RNN unrolling technique. We also review the difficulties with
training the standard RNN and address them by transforming the RNN
into the ``Vanilla LSTM''\footnote{The nickname ``Vanilla LSTM'' symbolizes this model's flexibility
and generality \citep{Greff2015LSTMAS}.} network through a series of logical arguments. We provide all equations
pertaining to the LSTM system together with detailed descriptions
of its constituent entities. Albeit unconventional, our choice of
notation and the method for presenting the LSTM system emphasizes
ease of understanding. As part of the analysis, we identify new opportunities
to enrich the LSTM system and incorporate these extensions into the
Vanilla LSTM network, producing the most general LSTM variant to date.
The target reader has already been exposed to RNNs and LSTM networks
through numerous available resources and is open to an alternative
pedagogical approach. A Machine Learning practitioner seeking guidance
for implementing our new augmented LSTM model in software for experimentation
and research will find the insights and derivations in this treatise
valuable as well.
\end{abstract}

\IEEEpeerreviewmaketitle{\thispagestyle{plain} \cfoot{\thepage}}

\section{Introduction}

Since the original 1997 LSTM paper \citep{hochreiter1997long}, numerous
theoretical and experimental works have been published on the subject
of this type of an RNN, many of them reporting on the astounding results
achieved across a wide variety of application domains where data is
sequential. The impact of the LSTM network has been notable in language
modeling, speech-to-text transcription, machine translation, and other
applications \citep{journals/entropy/LinT17}. Inspired by the impressive
benchmarks reported in the literature, some readers in academic and
industrial settings decide to learn about the Long Short-Term Memory
network (henceforth, ``the LSTM network'') in order to gauge its
applicability to their own research or practical use-case. All major
open source machine learning frameworks offer efficient, production-ready
implementations of a number of RNN and LSTM network architectures.
Naturally, some practitioners, even if new to the RNN/LSTM systems,
take advantage of this access and cost-effectiveness and proceed straight
to development and experimentation. Others seek to understand every
aspect of the operation of this elegant and effective system in greater
depth. The advantage of this lengthier path is that it affords an
opportunity to build a certain degree of intuition that can prove
beneficial during all phases of the process of incorporating an open
source module to suit the needs of their research effort or a business
application, preparing the dataset, troubleshooting, and tuning.

In a common scenario, this undertaking balloons into reading numerous
papers, blog posts, and implementation guides in search of an ``A
through Z'' understanding of the key principles and functions of
the system, only to find out that, unfortunately, most of the resources
leave one or more of the key questions about the basics unanswered.
For example, the Recurrent Neural Network (RNN), which is the general
class of a neural network that is the predecessor to and includes
the LSTM network as a special case, is routinely simply stated without
precedent, and unrolling is presented without justification. Moreover,
the training equations are often omitted altogether, leaving the reader
puzzled and searching for more resources, while having to reconcile
disparate notation used therein. Even the most oft-cited and celebrated
primers to date have fallen short of providing a comprehensive introduction.
The combination of descriptions and colorful diagrams alone is not
actionable, if the architecture description is incomplete, or if important
components and formulas are absent, or if certain core concepts are
left unexplained.

As of the timeframe of this writing, a single self-contained primer
that provides a clear and concise explanation of the Vanilla LSTM
computational cell with well-labeled and logically composed schematics
that go hand-in-hand with the formulas is still lacking. The present
work is motivated by the conviction that a unifying reference, conveying
the basic theory underlying the RNN and the LSTM network, will benefit
the Machine Learning (ML) community.

The present article is an attempt to fill in this gap, aiming to serve
as the introductory text that the future students and practitioners
of RNN and LSTM network can rely upon for learning all the basics
pertaining to this rich system. With the emphasis on using a consistent
and meaningful notation to explain the facts and the fundamentals
(while removing mystery and dispelling the myths), this backgrounder
is for those inquisitive researchers and practitioners who not only
want to know ``how'', but also to understand ``why''. 

We focus on the RNN first, because the LSTM network is a type of an
RNN, and since the RNN is a simpler system, the intuition gained by
analyzing the RNN applies to the LSTM network as well. Importantly,
the canonical RNN equations, which we derive from differential equations,
serve as the starting model that stipulates a perspicuous logical
path toward ultimately arriving at the LSTM system architecture.

The reason for taking the path of deriving the canonical RNN equations
from differential equations is that even though RNNs are expressed
as difference equations, differential equations have been indispensable
for modeling neural networks and continue making a profound impact
on solving practical data processing tasks with machine learning methods.
On one hand, leveraging the established mathematical theories from
differential equations in the continuous-time domain has historically
led to a better understanding of the evolution of the related difference
equations, since the difference equations are obtained from the corresponding
original differential equations through discretization of the differential
operators acting on the underlying functions \citep{sanger:optimal1989,scd/mit/Sherstinsky1994,journals/corr/LiaoP16,journals/statistical_learning/HaberR17pub.1092697518,conf/icml/LuZLD18,journals/corr/abs-1804-04272,Sherstinsky2018fundamentals,conf/neurips2018/Sherstinsky-NeurIPS2018-CRACT-3}.
On the other hand, considering the existing deep neurally-inspired
architectures as the numerical methods for solving their respective
differential equations aided by the recent advances in memory-efficient
implementations has helped to successfully stabilize very large models
at lower computational costs compared to their original versions \citep{conf/nips/CicconeGMOG18,conf/aaai/ChangMHRBH18,chang2018antisymmetricrnn}.
Moreover, differential equations defined on the continuous time domain
are a more natural fit for modeling certain real-life scenarios than
the difference equations defined over the domain of evenly-discretized
time intervals \citep{conf/nips/ChenRBD18,rubanova2019latent}.

Our primary aspiration for this document, particularly for the sections
devoted to the Vanilla LSTM system and its extensions, is to fulfill
all of the following requirements:

\vspace{-\topsep}
\begin{enumerate}
\item Intuitive -- the notation and semantics of variables must be descriptive,
explicitly and unambiguously mapping to their respective purposes
in the system.
\item Complete -- the explanations and derivations must include both the
inference equations (``forward pass'' or ``normal operation'')
and the training equations (``backward pass''), and account for
all components of the system.
\item General -- the treatment must concern the most inclusive form of
the LSTM system (i.e., the ``Vanilla LSTM''), specifically including
the influence of the cell's state on control nodes (``pinhole connections'').
\item Illustrative -- the description must include a complete and clearly
labeled cell diagram as well as the sequence diagram, leaving nothing
to imagination or guessing (i.e., the imperative is: strive to minimize
cognitive strain, do not leave anything as an ``exercise for the
reader'' -- everything should be explained and made explicit).
\item Modular -- the system must be described in such a way that the LSTM
cell can be readily included as part of a pluggable architecture,
both horizontally (``deep sequence'') and vertically (``deep representation'').
\item Vector notation -- the equations should be expressed in the matrix
and vector form; it should be straightforward to plug the equations
into a matrix software library (such as $\mathtt{numpy}$) as written,
instead of having to iterate through indices.
\end{enumerate}
In all sources to date, one or more of the elements in the above list
is not addressed\footnote{An article co-authored by one of the LSTM inventors provides a self-contained
summary of the embodiment of an RNN, though not at an introductory
level \citep{Greff2015LSTMAS}.} \citep{graves05nn,graves:05ijcnn,phd/de/Graves2008,conf/icml/SutskeverMH11,sundermeyer2012neural,journals/corr/Graves13,sutskever2014sequence,conf/interspeech/SakSB14,DBLP:journals/corr/Lipton15,AndrejKarpathyUnreasonableEffectivenessRNN2015,ChristopherOlahUnderstandingLSTM2015,journals/corr/PalangiDSGHCSW15,journals/corr/KannanKRKTMCLGY16,journals/corr/ZhouCWLX16,paulrenvoise2017rnn,EdwinChenExploringLSTMs2017,ArunMallyaLSTMForwardBackwardPass2017,ArunMallyaLSTMIntroduction2017,ArunMallyaLSTMVariants2017,VarunaJayasiriNumpyLSTM2017,journals/corr/abs-1801-01078}.
Hence, to serve as a comprehensive introduction, the present tutorial
captures all the essential details. The practice of using a succinct
vector notation and meaningful variable names as well as including
the intermediate steps in formulas is designed to build intuition
and make derivations easy to follow.

The rest of this document is organized as follows. Section \ref{sec:The-Roots-of-RNN}
gives a principled background behind RNN systems. Then\\ Section
\ref{sec:RNN-Unfolding/Unrolling} formally arrives at RNN unrolling
by proving a precise statement concerning approximating long sequences
by a series of shorter, independent sub-sequences (segments). Section
\ref{sec:RNN-Training-Difficulties} presents the RNN training mechanism
based on the technique, known as ``Back Propagation Through Time'',
and explores the numerical difficulties, which occur when training
on long sequences. To remedy these problems, Section \ref{sec:From-RNN-to-Vanilla-LSTM-Network}
methodically constructs the Vanilla LSTM cell from the canonical RNN
system (derived in Section \ref{sec:The-Roots-of-RNN}) by reasoning
through the ways of making RNN more robust. Section \ref{sec:The-Vanilla-LSTM-Network-Mechanism-in-Detail}
provides a detailed explanation of all aspects of the Vanilla LSTM
cell. Even though this section is intended to be self-contained, familiarity
with the material covered in the preceding sections will be beneficial.
The Augmented LSTM system, which embellishes the Vanilla LSTM system
with the new computational components, identified as part of the exercise
of transforming the RNN to the LSTM network, is presented in Section
\ref{sec:Extensions-to-the-Vanilla-LSTM-Network}. Section \ref{sec:Conclusions-and-Future-Work}
summarizes the covered topics and proposes future projects.

\section{The Roots of RNN\label{sec:The-Roots-of-RNN}}

In this section, we will derive the Recurrent Neural Network (RNN)
from differential equations \citep{Sherstinsky2018fundamentals,conf/neurips2018/Sherstinsky-NeurIPS2018-CRACT-3}.
Let $\vec{s}(t)$ be the value of the $d$-dimensional state signal
vector and consider the general nonlinear first-order non-homogeneous
ordinary differential equation, which describes the evolution of the
state signal as a function of time, $t$:
\begin{align}
\frac{d\vec{s}(t)}{dt} & =\vec{f}(t)+\vec{\phi}\label{eq:General1stOrderNonlinearNonhomogeneousDiffEqSampledStateDomain}
\end{align}
where $\vec{f}(t)$ is a $d$-dimensional vector-valued function of
time, $t\in\mathbb{R}^{+}$, and $\vec{\phi}$ is a constant $d$-dimensional
vector.

One canonical form of $\vec{f}(t)$ is:
\begin{align}
\vec{f}(t) & =\vec{h}\left(\vec{s}(t),\vec{x}(t)\right)\label{eq:CanonicalAutonomousNonlinearNonhomogeneousSampledStateDomainStateRate}
\end{align}
where $\vec{x}(t)$ is the $d$-dimensional input signal vector and
$\vec{h}\left(\vec{s}(t),\vec{x}(t)\right)$ is a vector-valued function
of vector-valued arguments. 

The resulting system,
\begin{align}
\frac{d\vec{s}(t)}{dt} & =\vec{h}\left(\vec{s}(t),\vec{x}(t)\right)+\vec{\phi}\label{eq:CanonicalAutonomous1stOrderNonlinearNonhomogeneousDiffEqSampledStateDomain}
\end{align}
comes up in many situations in physics, chemistry, biology, and engineering
\citep{STR94a,conf/amcc/Wang17a}. 

In certain cases, one starts with $s$ and $x$ as entirely ``analog''
quantities (i.e., functions not only of time, $t$, but also of another
independent continuous variable, $\vec{\xi},$ denoting the coordinates
in multi-dimensional space). Using this notation, the intensity of
an input video signal displayed on a flat $2$-dimensional screen
would be represented as $x(\vec{\xi},t)$ with $\vec{\xi}\in\mathbb{R}^{2}$.
Sampling $x(\vec{\xi},t)$ on a uniform $2$-dimensional grid converts
this signal to the representation $x(\vec{i},t)$, where $\vec{i}$
is now a discrete $2$-dimensional index. Finally, assembling the
values of $x(\vec{i},t)$ for all permutations of the components of
the index, $\vec{i}$, into a column vector, produces $\vec{x}(t)$
as originally presented in Equation \ref{eq:CanonicalAutonomous1stOrderNonlinearNonhomogeneousDiffEqSampledStateDomain}
above.

One special case of $\vec{f}(t)$ in Equation \ref{eq:CanonicalAutonomousNonlinearNonhomogeneousSampledStateDomainStateRate}
is:
\begin{align}
\vec{f}(t) & =\vec{a}(t)+\vec{b}(t)+\vec{c}(t)\label{eq:CanonicalMLThreeTermsAnalogWarpedInputAdditiveModel}
\end{align}
whose constituent terms, $\vec{a}(t)$, $\vec{b}(t)$, and $\vec{c}(t)$,
are $d$-dimensional vector-valued functions of time, $t$. Equation
\ref{eq:CanonicalMLThreeTermsAnalogWarpedInputAdditiveModel} is called
the ``Additive Model'' in Brain Dynamics research literature, because
it adds the terms, possibly nonlinear, that determine the rate of
change of neuronal activities, or potentials, $\vec{s}(t)$. As a
cornerstone of neural network research, the abstract form of the Additive
Model in Equation \ref{eq:CanonicalMLThreeTermsAnalogWarpedInputAdditiveModel}
has been particularized in many ways, including incorporating the
effects of delays, imposing ``shunting'' (or ``saturating'') bounds
on the state of the system, and other factors. Biologically motivated
uses of the Additive Model span computational analyses of vision,
decision making, reinforcement learning, sensory-motor control, short-term
and long-term memory, and the learning of temporal order in language
and speech \citep{Grossberg:2013}. It has also been noted that the
Additive Model generalizes the Hopfield model \citep{Hopfield84},
which, while rooted in biological plausibility, has been influential
in physics and engineering \citep{grossberg:historical1988,Grossberg:2013}.
In fact, a simplified and discretized form of the Additive Model played
a key role in linking the nonlinear dynamical systems governing morphogenesis,
one of the fundamental aspects of developmental biology, to a generalized
version of the Hopfield network \citep{Hopfield84}, and applying
it to an engineering problem in image processing \citep{scd/mit/Sherstinsky1994,journals/tip/SherstinskyP96}.

Consider a saturating Additive Model in Equation \ref{eq:CanonicalMLThreeTermsAnalogWarpedInputAdditiveModel}
with the three constituent terms, $\vec{a}(t)$, $\vec{b}(t)$, and
$\vec{c}(t)$, defined as follows:
\begin{align}
\vec{a}(t) & =\sum_{k=0}^{K_{s}-1}\vec{a}_{k}(\vec{s}(t-\tau_{s}(k)))\label{eq:CanonicalMLAnalogStateComponent}\\
\vec{b}(t) & =\sum_{k=0}^{K_{r}-1}\vec{b}_{k}(\vec{r}(t-\tau_{r}(k)))\label{eq:CanonicalMLWarpedStateComponent}\\
\vec{r}(t-\tau_{r}(k)) & =G\left(\vec{s}(t-\tau_{r}(k))\right)\label{eq:CanonicalMLWarpedStateActivation}\\
\vec{c}(t) & =\sum_{k=0}^{K_{x}-1}\vec{c}_{k}(\vec{x}(t-\tau_{x}(k)))\label{eq:CanonicalMLExternalDrivingForceInputComponent}
\end{align}
where $\vec{r}(t)$, the readout signal vector, is a warped version
of the state signal vector, $\vec{s}(t)$. A popular choice for the
element-wise nonlinear, saturating, and invertible ``warping'' (or
``activation'') function, $G(z)$, is an optionally scaled and/or
shifted form of the hyperbolic tangent. Then the resulting system,
obtained by substituting Equations \ref{eq:CanonicalMLAnalogStateComponent}
-- \ref{eq:CanonicalMLExternalDrivingForceInputComponent} into Equation
\ref{eq:CanonicalMLThreeTermsAnalogWarpedInputAdditiveModel} and
inserting into Equation \ref{eq:General1stOrderNonlinearNonhomogeneousDiffEqSampledStateDomain},
becomes:
\begin{align}
\frac{d\vec{s}(t)}{dt} & =\sum_{k=0}^{K_{s}-1}\vec{a}_{k}(\vec{s}(t-\tau_{s}(k)))+\sum_{k=0}^{K_{r}-1}\vec{b}_{k}(\vec{r}(t-\tau_{r}(k)))+\sum_{k=0}^{K_{x}-1}\vec{c}_{k}(\vec{x}(t-\tau_{x}(k)))+\vec{\phi}\label{eq:CanonicalML}\\
\vec{r}(t-\tau_{r}(k)) & =G\left(\vec{s}(t-\tau_{r}(k))\right)\label{eq:CanonicalMLWarpedState}
\end{align}
Equation \ref{eq:CanonicalML} is a nonlinear ordinary delay differential
equation (DDE) with discrete delays. Delay is a common feature of
many processes in biology, chemistry, mechanics, physics, ecology,
and physiology, among others, whereby the nature of the processes
dictates the use of delay equations as the only appropriate means
of modeling. In engineering, time delays often arise in feedback loops
involving sensors and actuators \citep{journals/vibcontrol/Kyrychko072010}.

Hence, the time rate of change of the state signal in Equation \ref{eq:CanonicalML}
depends on three main components plus the constant (``bias'') term,
$\vec{\phi}$. The first (``analog'') component, $\sum_{k=0}^{K_{s}-1}\vec{a}_{k}(\vec{s}(t-\tau_{s}(k)))$,
is the combination of up to $K_{s}$ time-shifted (by the delay time
constants, $\tau_{s}(k)$) functions, $\vec{a}_{k}(\vec{s}(t))$,
where the term ``analog'' underscores the fact that each $\vec{a}_{k}(\vec{s}(t))$
is a function of the (possibly time-shifted) state signal itself (i.e.,
not the readout signal, which is the warped version of the state signal).
The second component, $\sum_{k=0}^{K_{r}-1}\vec{b}_{k}(\vec{r}(t-\tau_{r}(k)))$,
is the combination of up to $K_{r}$ time-shifted (by the delay time
constants, $\tau_{r}(k)$) functions, $\vec{b}_{k}(\vec{r}(t))$,
of the readout signal, given by Equation \ref{eq:CanonicalMLWarpedState},
the warped (binary-valued in the extreme) version of the state signal.
The third component, $\sum_{k=0}^{K_{x}-1}\vec{c}_{k}(\vec{x}(t-\tau_{x}(k)))$,
representing the external input, is composed of the combination of
up to $K_{x}$ time-shifted (by the delay time constants, $\tau_{x}(k)$)
functions, $\vec{c}_{k}(\vec{x}(t))$, of the input signal\footnote{The entire input signal, $\vec{c}(t)$, in Equation \ref{eq:CanonicalMLExternalDrivingForceInputComponent}
is sometimes referred to as the ``external driving force'' (or,
simply, the ``driving force'') in physics.}.

The rationale behind choosing a form of the hyperbolic tangent as
the warping function is that the hyperbolic tangent possesses certain
useful properties. On one hand, it is monotonic and negative-symmetric
with a quasi-linear region, whose slope can be regulated \citep{metropolis53}.
On the other hand, it is bipolarly-saturating (i.e., bounded at both
the negative and the positive limits of its domain). The quasi-linear
mode aides in the design of the system's parameters and in interpreting
its behavior in the ``small signal'' regime (i.e., when $\left\Vert \vec{s}(t)\right\Vert \ll1)$.
The bipolarly-saturating (``squashing'') aspect, along with the
proper design of the internal parameters of the functions $\vec{a}_{k}(\vec{s}(t))$
and $\vec{b}_{k}(\vec{r}(t))$, helps to keep the state of the system
(and, hence, its output) bounded. The dynamic range of the state signals
is generally unrestricted, but the readout signals are guaranteed
to be bounded, while still carrying the state information with low
distortion in the quasi-linear mode of the warping function (the ``small
signal'' regime). If the system, described by Equation \ref{eq:CanonicalML}
and Equation \ref{eq:CanonicalMLWarpedState}, is stable, then the
state signals are bounded as well \citep{journals/tcas/SherstinskyPicardIEEETCAS98}.

The time delay terms on the right hand side of Equation \ref{eq:CanonicalML}
comprise the ``memory'' aspects of the system. They enable the quantity
holding the instantaneous time rate of change of the state signal,
$\frac{d\vec{s}(t)}{dt}$, to incorporate contributions from the state,
the readout, and the input signal values, measured at different points
in time, relative to the current time, $t$. Qualitatively, these
temporal elements enrich the expressive power of the model by capturing
causal and/or contextual information.

In neural networks, the time delay is an intrinsic part of the system
and also one of the key factors that determines the dynamics\footnote{In neural networks, time delay occurs in the interaction between neurons;
it is induced by the finite switching speed of the neuron and the
communication time between neurons \citep{journals/vibcontrol/Kyrychko072010,phd/ua/Ostroverkhyi2010}.}. Much of the pioneering research in recurrent networks during the
1970s and the 1980s was founded on the premise that neuron processes
and interactions could be expressed as systems of coupled DDEs \citep{Hopfield84,Grossberg:2013}.
Far from the actual operation of the human brain, based on what was
already known at the time, these ``neurally inspired'' mathematical
models have been shown to exhibit sufficiently interesting emerging
behaviors for both, advancing the knowledge and solving real-world
problems in various practical applications. While the major thrust
of research efforts was concerned primarily with continuous-time networks,
it was well understood that the learning procedures could be readily
adapted to discrete systems, obtained from the original differential
equations through sampling. We will also follow the path of sampling
and discretization for deriving the RNN equations \citep{conf/neurips2018/Sherstinsky-NeurIPS2018-CRACT-3}.
Over the span of these two decades, pivotal and lasting contributions
were made in the area of training networks containing interneurons\footnote{This term from neuroanatomy provides a biological motivation for considering
networks containing multiple ``hidden'' layers, essentially what
is dubbed ``deep networks'' and ``deep learning'' today.} with ``Error Back Propagation'' (or ``Back Propagation of Error'',
or ``Back Propagation'' for short), a special case of a more general
error gradient computation procedure. To accommodate recurrent networks,
both continuous-time and discrete-time versions of ``Back Propagation
Through Time'' have been developed on the foundation of Back Propagation
and used to train the weights and time delays of these networks to
perform a wide variety of tasks \citep{Jordan:86,Pineda87a,pineda:recandhigher87b,Pearlmutter89,elman:90a,pearlmutter:dynamic}.
We will rely on Back Propagation Through Time for training the systems
analyzed in this paper.

\pagebreak The contribution of each term on the right hand side of
Equation \ref{eq:CanonicalML} to the overall system is qualitatively
different from that of the others. The functions, $\vec{a}_{k}(\vec{s}(t-\tau_{s}(k)))$,
of the (``analog'') state signal in the first term have a strong
effect on the stability of the system, while the functions, $\vec{b}_{k}(\vec{r}(t-\tau_{r}(k)))$,
of the (bounded) readout signal in the second term capture most of
the interactions that shape the system's long-term behavior. If warranted
by the modeling requirements of the biological or physical system
and/or of the specific datasets and use-cases in an engineering setting,
the explicit inclusion of non-zero delay time constants in these terms
provides the necessary weighting flexibility in the temporal domain
(e.g., to account for delayed neural interactions) \citep{Vries:91}.
Thus, the parameters, $K_{s}$, $K_{r}$, and $K_{x}$ representing
the counts of the functions, $\vec{a}_{k}(\vec{s}(t-\tau_{s}(k)))$,
$\vec{b}_{k}(\vec{r}(t-\tau_{r}(k)))$, and $\vec{c}_{k}(\vec{x}(t-\tau_{x}(k)))$,
respectively (and the counts of the associated delay time constants,
$\tau_{s}(k)$, $\tau_{r}(k)$, and $\tau_{x}(k)$, respectively,
of these functions), in the system equations are chosen (or estimated
by an iterative procedure) accordingly. 

Suppose that $\vec{a}_{k}(\vec{s}(t-\tau_{s}(k)))$, $\vec{b}_{k}(\vec{r}(t-\tau_{r}(k)))$,
and $\vec{c}_{k}(\vec{x}(t-\tau_{x}(k)))$ are linear functions of
$\vec{s}$, $\vec{r}$, and $\vec{x}$, respectively. Then Equation
\ref{eq:CanonicalML} becomes a nonlinear DDE with linear (matrix-valued)
coefficients:
\begin{align}
\frac{d\vec{s}(t)}{dt} & =\sum_{k=0}^{K_{s}-1}A_{k}(\vec{s}(t-\tau_{s}(k)))+\sum_{k=0}^{K_{r}-1}B_{k}(\vec{r}(t-\tau_{r}(k)))+\sum_{k=0}^{K_{x}-1}C_{k}(\vec{x}(t-\tau_{x}(k)))+\vec{\phi}\label{eq:LinearOperatorsML}
\end{align}
Furthermore, if the matrices, $A_{k}$, $B_{k}$, and $C_{k}$, are
circulant (or block circulant), then the matrix-vector multiplication
terms in Equation \ref{eq:LinearOperatorsML} can be expressed as
convolutions in the space of the elements of $\vec{s}$, $\vec{r}$,
$\vec{x}$, and $\vec{\phi}$, each indexed by $\vec{i}$:
\begin{align}
\frac{ds(\vec{i},t)}{dt} & =\sum_{k=0}^{K_{s}-1}a_{k}(\vec{i})\ast s(\vec{i},t-\tau_{s}(k))+\sum_{k=0}^{K_{r}-1}b_{k}(\vec{i})\ast r(\vec{i},t-\tau_{r}(k))+\sum_{k=0}^{K_{x}-1}c_{k}(\vec{i})\ast x(\vec{i},t-\tau_{x}(k))+\phi(\vec{i})\label{eq:ConvolutionallyCoupledML}
\end{align}
The index, $\vec{i}$, is $1$-dimensional if the matrices, $A_{k}$,
$B_{k}$, and $C_{k}$, are circulant and multi-dimensional if they
are block circulant\footnote{For example, the $2$-dimensional shape of $\vec{i}$ is appropriate
for image processing tasks.}. 

The summations of time delayed terms in Equation \ref{eq:ConvolutionallyCoupledML}
represent convolutions in the time domain with finite-sized kernels,
consisting of the spatial convolutions $a_{k}(\vec{i})\ast s(\vec{i})$,
$b_{k}(\vec{i})\ast r(\vec{i})$, and $c_{k}(\vec{i})\ast x(\vec{i})$
as the coefficients for the three temporal components, respectively.
In fact, if the entire data set (e.g., the input data set, $\vec{x}(t)$)
is available a priori for all time ahead of the application of Equation
\ref{eq:ConvolutionallyCoupledML}, then some of the corresponding
time delays (e.g., $\tau_{x}(k)$) can be negative, thereby allowing
the incorporation of ``future'' information for computing the state
of the system at the present time, $t$. This will become relevant
further down in the analysis.

Before proceeding, it is interesting to note that earlier studies
linked the nonlinear dynamical system, formalized in Equation \ref{eq:CanonicalML}
(with $K_{s}=K_{r}=K_{x}=1$ and all $\tau_{s}$, $\tau_{r}$, and
$\tau_{x}$ set to zero), to the generalization of a type of neural
networks\footnote{As mentioned earlier, a more appropriate phrase would be ``neurally
inspired'' networks.}. Specifically the variant, in which the functions $\vec{a}_{k}(\vec{s}(t))$,
$\vec{b}_{k}(\vec{r}(t))$, and $\vec{c}_{k}(\vec{x}(t))$ are linear
operators as in Equation \ref{eq:LinearOperatorsML} (with $K_{s}=K_{r}=K_{x}=1$
and all $\tau_{s}$, $\tau_{r}$, and $\tau_{x}$ set to zero) was
shown to include the Continuous Hopfield Network \citep{Hopfield84}
as a special case. Its close relative, in which these operators are
further restricted to be convolutional as in Equation \ref{eq:ConvolutionallyCoupledML}
(again, with $K_{s}=K_{r}=K_{x}=1$ and all $\tau_{s}$, $\tau_{r}$,
and $\tau_{x}$ set to zero), was shown to include the Cellular Neural
Network \citep{chua:ieeecs88a,chua:ieeecs88b} as a special case \citep{scd/mit/Sherstinsky1994,journals/tcas/SherstinskyPicardIEEETCAS98,journals/tip/SherstinskyP96}.

Applying the simplifications:

\begin{equation}\label{eq:LinearOperatorsMLSimplifications}\begin{rcases} K_{s} & = 1\\\tau_{s}(0) & = 0\\A_{0} & = A\\K_{r} & = 1\\\tau_{r}(0) & = \tau_{0}\\B_{0} & = B\\K_{x} & = 1\\\tau_{x}(0) & = 0\\C_{0} & = C \end{rcases}\end{equation}

(some of which will be later relaxed) to Equation \ref{eq:LinearOperatorsML}
turns it into:
\begin{align}
\frac{d\vec{s}(t)}{dt} & =A\vec{s}(t)+B\vec{r}(t-\tau_{0})+C\vec{x}(t)+\vec{\phi}\label{eq:LinearOperatorsML1WarpedStateDelayTerm}
\end{align}
Equation \ref{eq:LinearOperatorsML}, Equation \ref{eq:ConvolutionallyCoupledML},
and, hence, Equation \ref{eq:LinearOperatorsML1WarpedStateDelayTerm}
are nonlinear first-order non-homogeneous DDEs. A standard numerical
technique for evaluating these equations, or, in fact, any embodiments
of Equation \ref{eq:General1stOrderNonlinearNonhomogeneousDiffEqSampledStateDomain},
is to discretize them in time and compute the values of the input
signals and the state signals at each time sample up to the required
total duration, thereby performing numerical integration.

Denoting the duration of the sampling time step as $\triangle T$
and the index of the time sample as $n$ in the application of the
backward Euler discretization rule\footnote{The backward Euler method is a stable discretization rule used for
solving ordinary differential equations numerically \citep{lntrefethen1996}.
A simple way to express this rule is to substitute the forward finite
difference formula into the definition of the derivative, relax the
requirement $\triangle T\rightarrow0$, and evaluate the function
on the right hand side (i.e., the quantity that the derivative is
equal to) at time, $t+\triangle T$.} to Equation \ref{eq:LinearOperatorsML1WarpedStateDelayTerm} yields\footnote{It is straightforward to extend the application of the discretization
rule to the full Equation \ref{eq:LinearOperatorsML}, containing
any or all the time delay terms and their corresponding matrix coefficients,
without the above simplifications. So there is no loss of generality. }:
\begin{align}
t & =n\triangle T\\
\frac{d\vec{s}(t)}{dt} & \thickapprox\frac{\vec{s}(n\triangle T+\triangle T)-\vec{s}(n\triangle T)}{\triangle T}\\
A\vec{s}(t)+B\vec{r}(t-\tau_{0})+C\vec{x}(t)+\vec{\phi} & =A\vec{s}(n\triangle T)+B\vec{r}(n\triangle T-\tau_{0})+C\vec{x}(n\triangle T)+\vec{\phi}\\
A\vec{s}(t+\triangle T)+B\vec{r}(t+\triangle T-\tau_{0})+C\vec{x}(t+\triangle T)+\vec{\phi} & =A\vec{s}(n\triangle T+\triangle T)+B\vec{r}(n\triangle T+\triangle T-\tau_{0})+C\vec{x}(n\triangle T+\triangle T)+\vec{\phi}\\
\frac{\vec{s}(n\triangle T+\triangle T)-\vec{s}(n\triangle T)}{\triangle T} & \thickapprox A\vec{s}(n\triangle T+\triangle T)+B\vec{r}(n\triangle T+\triangle T-\tau_{0})+C\vec{x}(n\triangle T+\triangle T)+\vec{\phi}\label{eq:TimeDiscretizedLinearOperatorsML1WarpedStateDelayTermExplicitUnfactoredDeltaT}
\end{align}
Now set the delay, $\tau_{0}$, equal to the single time step. This
can be interpreted as storing the value of the readout signal into
memory at every time step to be used in the above equations at the
next time step. After a single use, the memory storage can be overwritten
with the updated value of the readout signal to be used at the next
time step, and so forth\footnote{Again, the additional terms, containing similarly combined time delayed
input signals (as will be shown to be beneficial later in this paper)
and state signals, can be included in the discretization, relaxing
the above simplifications as needed to suit the requirements of the
particular problem at hand. }. Thus, setting $\tau_{0}=\triangle T$ and replacing the approximation
sign with an equal sign for convenience in Equation \ref{eq:TimeDiscretizedLinearOperatorsML1WarpedStateDelayTermExplicitUnfactoredDeltaT}
gives:
\begin{align}
\frac{\vec{s}(n\triangle T+\triangle T)-\vec{s}(n\triangle T)}{\triangle T} & =A\vec{s}(n\triangle T+\triangle T)+B\vec{r}(n\triangle T)+C\vec{x}(n\triangle T+\triangle T)+\vec{\phi}\\
\frac{\vec{s}((n+1)\triangle T)-\vec{s}(n\triangle T)}{\triangle T} & =A\vec{s}((n+1)\triangle T)+B\vec{r}(n\triangle T)+C\vec{x}((n+1)\triangle T)+\vec{\phi}\\
\vec{s}((n+1)\triangle T)-\vec{s}(n\triangle T) & =\triangle T\left(A\vec{s}((n+1)\triangle T)+B\vec{r}(n\triangle T)+C\vec{x}((n+1)\triangle T)+\vec{\phi}\right)\label{eq:TimeDiscretizedLinearOperatorsML1WarpedStateDelayTermExplicitFactoredDeltaT}
\end{align}
After performing the discretization, all measurements of time in Equation
\ref{eq:TimeDiscretizedLinearOperatorsML1WarpedStateDelayTermExplicitFactoredDeltaT}
become integral multiples of the sampling time step, $\triangle T$.
Now, $\triangle T$ can be dropped from the arguments, which leaves
the time axis dimensionless. Hence, all the signals are transformed
into sequences, whose domain is the discrete index, $n$, and Equation
\ref{eq:LinearOperatorsML1WarpedStateDelayTerm} turns into a nonlinear
first-order non-homogeneous difference equation \citep{AVO89}:
\begin{align}
\vec{s}[n+1]-\vec{s}[n] & =\triangle T\left(A\vec{s}[n+1]+B\vec{r}[n]+C\vec{x}[n+1]+\vec{\phi}\right)\\
\vec{s}[n+1] & =\vec{s}[n]+\triangle T\left(A\vec{s}[n+1]+B\vec{r}[n]+C\vec{x}[n+1]+\vec{\phi}\right)\nonumber \\
\left(I-(\triangle T)A\right)\vec{s}[n+1] & =\vec{s}[n]+\left((\triangle T)B\right)\vec{r}[n]+\left((\triangle T)C\right)\vec{x}[n+1]+(\triangle T)\vec{\phi}\label{eq:DiscreteLinearOperatorsML1WarpedStateDelayTerm}
\end{align}
Defining: 
\begin{align}
W_{s} & =\left(I-(\triangle T)A\right)^{-1}\label{eq:StateWeightMatrixForCanonicalRNN}
\end{align}
and multiplying both sides of Equation \ref{eq:DiscreteLinearOperatorsML1WarpedStateDelayTerm}
by $W_{s}$ leads to:
\begin{align*}
\vec{s}[n+1] & =W_{s}\vec{s}[n]+\left((\triangle T)W_{s}B\right)\vec{r}[n]+\left((\triangle T)W_{s}C\right)\vec{x}[n+1])+\left((\triangle T)W_{s}\vec{\phi}\right)
\end{align*}
which after shifting the index, $n$, forward by $1$ step becomes:
\begin{align}
\vec{s}[n] & =W_{s}\vec{s}[n-1]+\left((\triangle T)W_{s}B\right)\vec{r}[n-1]+\left((\triangle T)W_{s}C\right)\vec{x}[n]+\left((\triangle T)W_{s}\vec{\phi}\right)\nonumber \\
\vec{r}[n] & =G(\vec{s}[n])
\end{align}
Defining two additional weight matrices and a bias vector,
\begin{align}
W_{r} & =(\triangle T)W_{s}B\\
W_{x} & =(\triangle T)W_{s}C\\
\vec{\theta}_{s} & =(\triangle T)W_{s}\vec{\phi}
\end{align}
transforms the above system into the canonical Recurrent Neural Network
(RNN) form:
\begin{align}
\vec{s}[n] & =W_{s}\vec{s}[n-1]+W_{r}\vec{r}[n-1]+W_{x}\vec{x}[n]+\vec{\theta}_{s}\label{eq:CanonicalRNN}\\
\vec{r}[n] & =G(\vec{s}[n])\label{eq:CanonicalRNNWarpedState}
\end{align}
\begin{figure}[tph]
\includegraphics[scale=0.47]{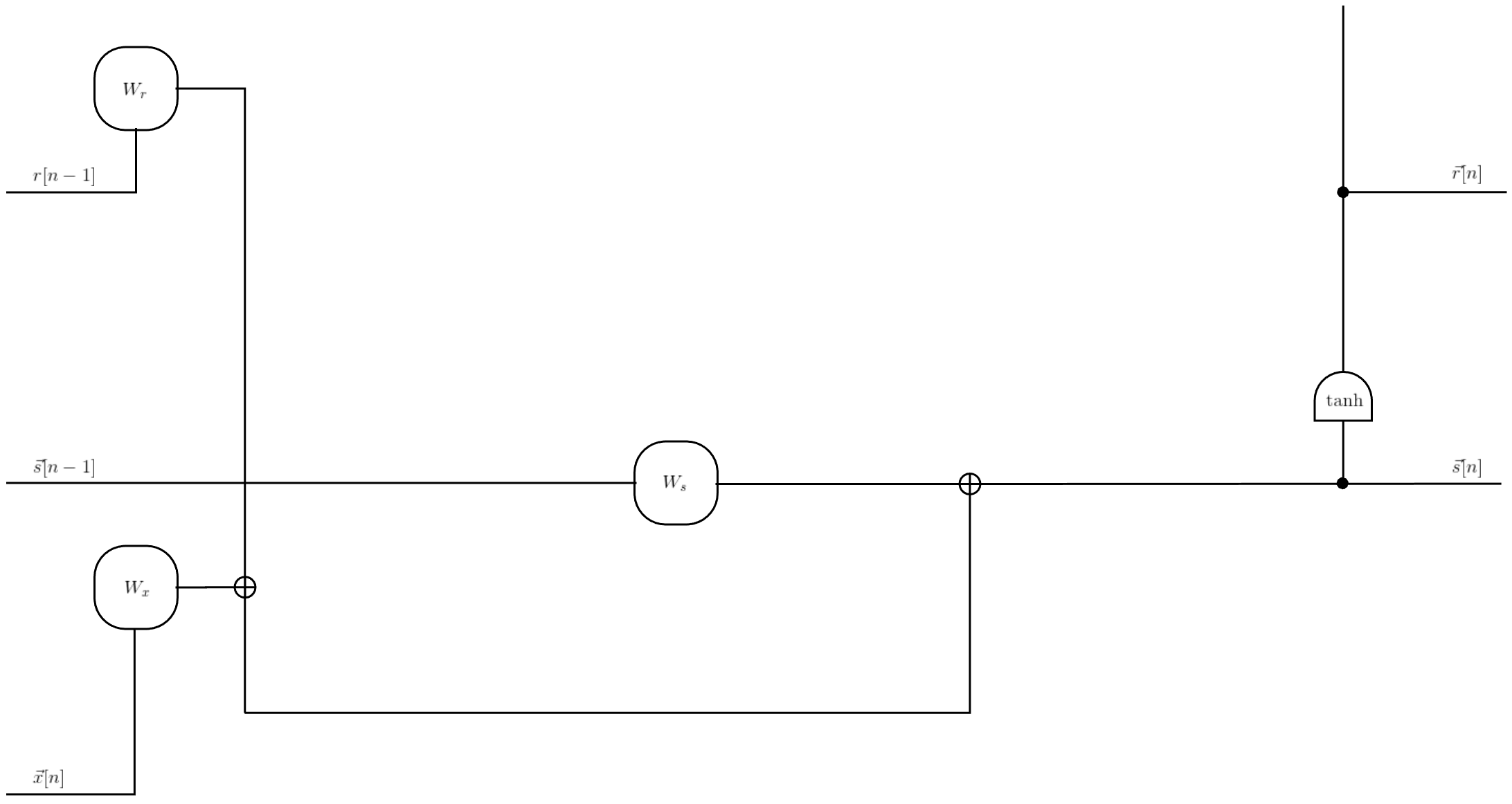}

\caption{Canonical RNN cell. The bias parameters, $\vec{\theta}_{s},$ have
been omitted from the figure for brevity. It can be assumed to be
included without the loss of generality by appending an additional
element, always set to $1$, to the input signal vector, $\vec{x}[n]$,
and increasing the row dimensions of $W_{x}$ by $1$.\label{fig:CanonicalRNNCellSchematics}}
\end{figure}

The RNN formulation in Equation \ref{eq:CanonicalRNN}, diagrammed
in Figure \ref{fig:CanonicalRNNCellSchematics}, will be later logically
evolved into the LSTM system. Before that, it is beneficial to introduce
the process of ``unrolling''\footnote{The terms ``unrolling'' and ``unfolding'' are used interchangeably
in association with RNN systems.} and the notion of a ``cell'' of an RNN. These concepts will be
simpler to describe using the standard RNN definition, which is derived
next from Equation \ref{eq:CanonicalRNN} based on stability arguments.

\pagebreak{}

For the system in Equation \ref{eq:CanonicalRNN} to be stable, every
eigenvalue of $\hat{W}=W_{s}+W_{r}$ must lie within the complex-valued
unit circle \citep{AVO89,STR94a}. Since there is considerable flexibility
in the choice of the elements of $A$ and $B$ to satisfy this requirement,
setting $\triangle T=1$ for simplicity is acceptable. As another
simplification, let $A$ be a diagonal matrix with large negative
entries (i.e., $a_{ii}\ll0$) on its main diagonal (thereby practically
guaranteeing the stability of Equation \ref{eq:LinearOperatorsML1WarpedStateDelayTerm}).
Then, from Equation \ref{eq:StateWeightMatrixForCanonicalRNN}, $W_{s}\thickapprox-A^{-1}$
will be a diagonal matrix with small positive entries, $\frac{1}{\left|a_{ii}\right|}$,
on its main diagonal, which means that the explicit effect of the
state signal's value from memory, $\vec{s}[n-1]$, on the system's
trajectory will be negligible (the implicit effect through $\vec{r}[n-1]$
will still be present as long as $\left\Vert W_{r}\right\Vert >0$).
Thus, ignoring the first term in Equation \ref{eq:CanonicalRNN},
reduces it to the standard RNN definition:
\begin{align}
\vec{s}[n] & =W_{r}\vec{r}[n-1]+W_{x}\vec{x}[n]+\vec{\theta}_{s}\label{eq:StandardRNN}\\
\vec{r}[n] & =G(\vec{s}[n])\label{eq:StandardRNNWarpedState}
\end{align}
From Equation \ref{eq:StandardRNN}, now only the matrix $\hat{W}\thickapprox W_{r}\thickapprox-A^{-1}B$
is responsible for the stability of the RNN. Consider the best case
scenario, where $B$ is a diagonal matrix ($B=\Lambda_{B}$). With
this simplification, the essential matrix for analyzing the stability
of Equation \ref{eq:StandardRNN} becomes $\tilde{W}=-A^{-1}\Lambda_{B}$,
where $\Lambda_{B}$ is the diagonal matrix (i.e., consisting of only
the eigenvalues of $B$, with the individual eigenvalues, $\lambda_{i}$,
on the main diagonal of $\Lambda_{B}$). Since both $A$ and $\Lambda_{B}$
are diagonal, $\tilde{W}$ is a diagonal matrix with the entries $\mu_{i}=\frac{\lambda_{i}}{\left|a_{ii}\right|}$
on its main diagonal. These quantities become the eigenvalues of the
overall RNN system in Equation \ref{eq:StandardRNN} in the ``small
signal regime'' ($\left\Vert \vec{s}[n]\right\Vert \ll1)$, each
adding the mode of $\left(\mu_{i}\right)^{n}$, multiplied by its
corresponding initial condition, to the trajectory of $\vec{s}[n]$.
A necessary and sufficient condition for stability is that $0<\mu_{i}<1$,
meaning that every eigenvalue, $\lambda_{i}$, of $B$ must satisfy
the condition $0<\lambda_{i}<\left|a_{ii}\right|$. If any $\mu_{i}$
and $\lambda_{i}$ fail to satisfy this condition, the system will
be unstable, causing the elements of $\vec{r}[n]$ to either oscillate
or saturate (i.e., enter the flat regions of the warping nonlinearity)
at some value of the index, $n$.

An alternative to choosing the specific convenient form of $A$ in
Equation \ref{eq:StateWeightMatrixForCanonicalRNN} would be to (somewhat
arbitrarily) treat $W_{s}$, $W_{r}$, $W_{x}$, and $\vec{\theta}_{s}$
in Equation \ref{eq:CanonicalRNN} as mutually independent parameters
and then set $W_{s}=0$ to obtain the standard RNN definition (as
in Equation \ref{eq:StandardRNN}). In this case, the above stability
analysis still applies. In particular, the eigenvalues, $\mu_{i}$,
of $W_{r}$ are subject to the same requirement, $0<\mu_{i}<1$, as
a necessary and sufficient condition for stability. 

Stability considerations will be later revisited in order to justify
the need to evolve the RNN to a more complex system, namely, the LSTM
network.

We have shown that the RNN, as expressed by Equation \ref{eq:CanonicalRNN}
(in the canonical form) or by Equation \ref{eq:StandardRNN} (in the
standard form), essentially implements the backward Euler numerical
integration method for the ordinary DDE in Equation \ref{eq:LinearOperatorsML1WarpedStateDelayTerm}.
This ``forward'' direction of starting in the continuous-time domain
(differential equation) and ending in the discrete-index domain (difference
equation) implies that the original phenomenon being modeled is assumed
to be fundamentally analog in nature, and that it is modeled in the
discrete domain as an approximation for the purpose of realization.
For example, the source signal could be the audio portion of a lecture,
recorded on an analog tape (or on digital media as a finely quantized
waveform and saved in an audio file). The original recording thus
contains the spoken words as well as the intonation, various emphases,
and other vocal modulations that communicate the content in the speaker's
individual way as expressed through voice. The samples generated by
a hypothetical discretization of this phenomenon, governed in this
model by Equation \ref{eq:LinearOperatorsML1WarpedStateDelayTerm},
could be captured as the textual transcript of the speech, saved in
a document containing only the words uttered by the speaker, but ignoring
all the intonation, emotion, and other analogous nuances. In this
scenario, it is the sequence of words in the transcript of the lecture
that the RNN will be employed to reproduce, not the actual audio recording
of the speech. The key subtle point in this scenario is that applying
the RNN as a model implies that the underlying phenomenon is governed
by Equation \ref{eq:LinearOperatorsML1WarpedStateDelayTerm}, whereby
the role of the RNN is that of implementing the computational method
for solving this DDE using the backward Euler discretization rule,
under the restriction that the sampling time step, $\triangle T$,
is equal to the delay, $\tau_{0}$. In contrast, the ``reverse''
direction would be a more appropriate model in situations where the
discrete signal is the natural starting domain, because the phenomenon
originates as a sequence of samples. For example, a written essay
originates as a sequence of words and punctuation, saved in a document.
One can conjure up an analog rendition of this essay as being read
by a narrator, giving life to the words and passages with intonation,
pauses, and other expressions not present in the original text of
the essay. While the starting point depends on how the original source
of data is generated, both the continuous (``forward'') and the
discrete (``reverse'') representations can serve as tools for gaining
insight into the advantages and the limitations of the models under
consideration.

\pagebreak{}

\section{RNN Unfolding/Unrolling\label{sec:RNN-Unfolding/Unrolling}}

It is convenient to use the term ``cell'' when referring to Equation
\ref{eq:CanonicalRNN} and Equation \ref{eq:StandardRNN} in the uninitialized
state. In other words, the sequence has been defined by these equations,
but its terms not yet computed. Then the cell can be said to be ``unfolded''
or ``unrolled'' by specifying the initial conditions on the state
signal, $\vec{s}[n]$, and numerically evaluating Equation \ref{eq:CanonicalRNN}
or\\ Equation \ref{eq:StandardRNN} for a finite range of discrete
steps, indexed by $n$. This process is illustrated in Figure \ref{fig:RNNCellUnrolled}.

\begin{figure}[tph]
\includegraphics[scale=0.44]{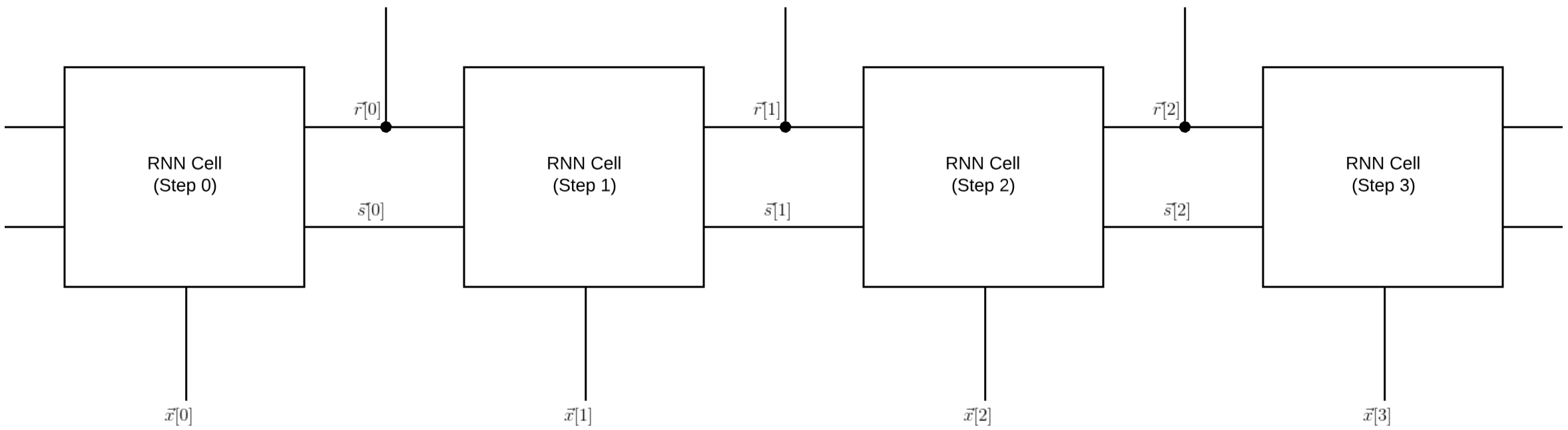}

\caption{Sequence of steps generated by unrolling an RNN cell.\label{fig:RNNCellUnrolled}}
\end{figure}

Both Equation \ref{eq:CanonicalRNN} and Equation \ref{eq:StandardRNN}
are recursive in the state signal, $\vec{s}[n]$. Hence, due to the
repeated application of the recurrence relation as part of the unrolling,
the state signal, $\vec{s}[n]$, at some value of the index, $n$,
no matter how large, encompasses the contributions of the state signal,
$\vec{s}[k]$, and the input signal, $\vec{x}[k]$, for all indices,
$k<n$, ending at $k=0$, the start of the sequence \citep{Jordan:86,elman:90a}.
Because of this attribute, the RNN belongs to the category of the
``Infinite Impulse Response'' (IIR) systems. 

Define the vector-valued unit step function as:
\begin{align}
\vec{u}[n] & =\begin{cases}
\vec{1}, & n\geq0\\
\vec{0}, & n<0
\end{cases}\label{eq:VectorValuedUnitStepFunction}
\end{align}
where $\vec{0}$ and $\vec{1}$ denote vectors, all of whose elements
are equal to $0$ and to $1$, respectively. Then the vector-valued
unit sample function, $\vec{\delta}[n]$, is defined by being $\vec{1}$
at $n=0$, and $\vec{0}$ otherwise. In terms of $\vec{u}[n]$,
\begin{align}
\vec{\delta}[n] & =\vec{u}[n]-\vec{u}[n-1]\label{eq:VectorValuedUnitSampleFunction}
\end{align}

These functions are depicted in Figure \ref{fig:UnitStep1DFunctionPlotAndUnitSample1DFunctionPlot}.

\begin{figure}[tph]
\subfloat[The unit step function.\label{fig:UnitStep1DFunctionPlot}]%
{\includegraphics[scale=0.5]{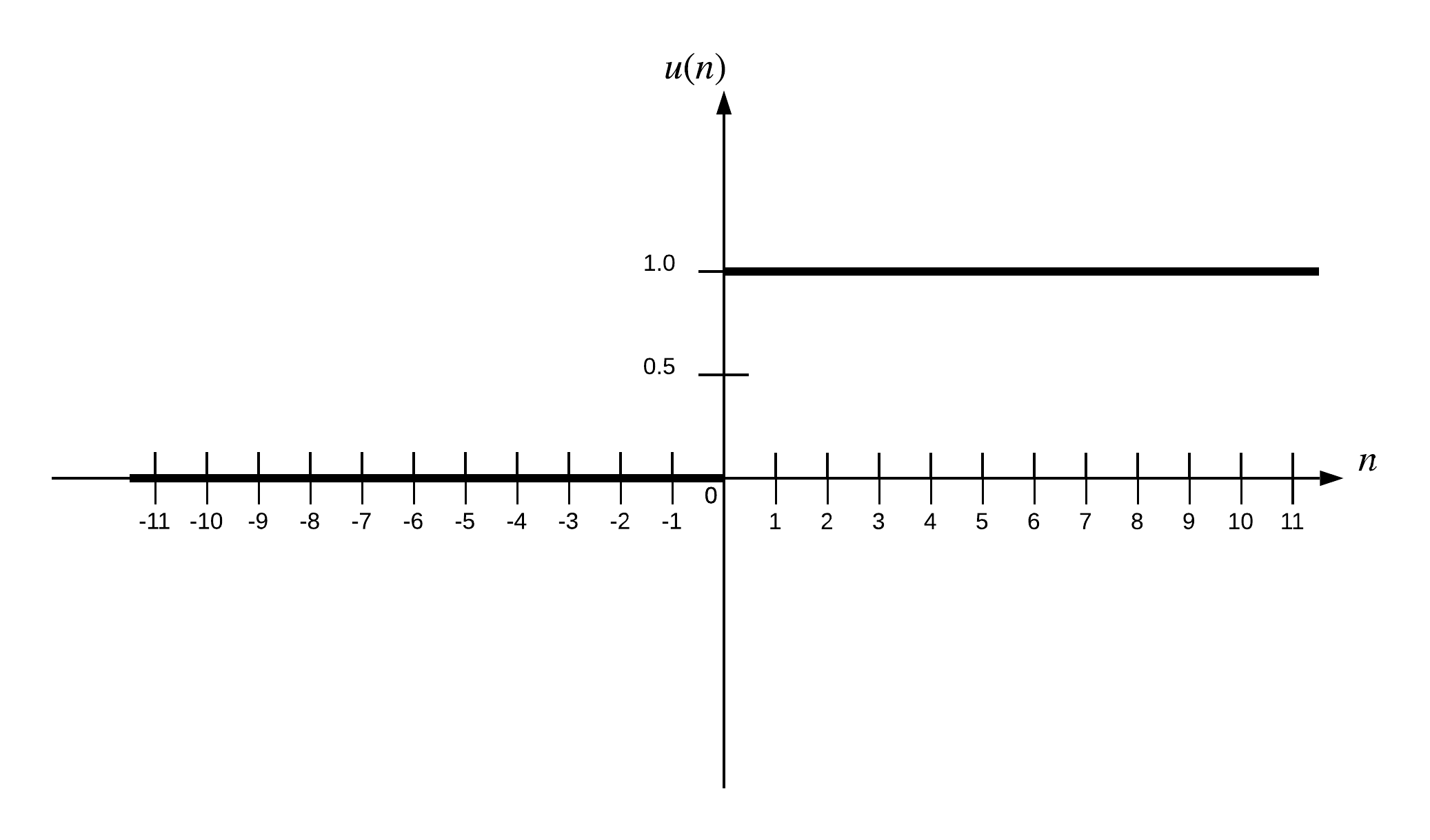}

}\hfill{}%
\subfloat[The unit sample function.\label{fig:UnitSample1DFunctionPlot}]%
{\includegraphics[scale=0.5]{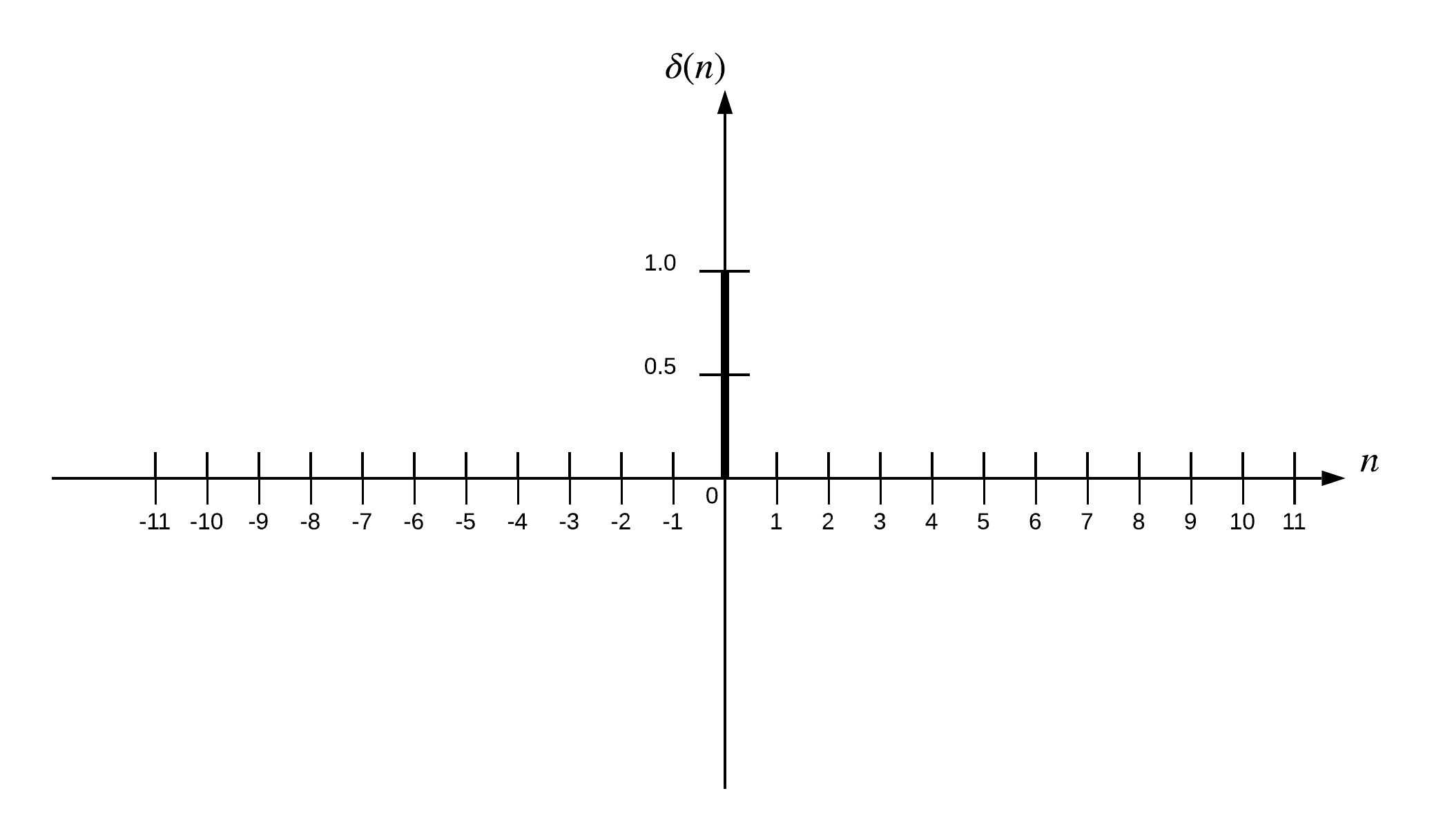}

}\caption{The unit step and the unit sample (\textquotedblleft impulse\textquotedblright )
functions plotted (in one data dimension) against the discrete index,
$n$. \label{fig:UnitStep1DFunctionPlotAndUnitSample1DFunctionPlot}}
\end{figure}

\pagebreak{}
\begin{example}
\label{example:StandardRNNAsIIRSequence}The IIR (i.e., unending)
nature of the sequences, governed by these equations, can be readily
demonstrated by letting $\vec{s}[-1]=\vec{0}$ be the initial condition,
setting $\vec{x}[n]=\vec{\delta}[n]$, the unit sample stimulus (i.e.,
the ``impulse''), and computing the response, $\vec{s}[n]$, to
this ``impulse'' for several values of the index, $n$, in order
to try to recognize a pattern. In the case of Equation \ref{eq:StandardRNN}
with $\vec{\theta}_{s}=\vec{0}$, the sequence of $\vec{s}[n]$ values
will be:
\begin{align}
\vec{s}[n=-1] & =\vec{0}\nonumber \\
\vec{s}[n=0] & =W_{x}\vec{1}\nonumber \\
\vec{s}[n=1] & =W_{r}G(W_{x}\vec{1})\nonumber \\
\vec{s}[n=2] & =W_{r}G(W_{r}G(W_{x}\vec{1}))\nonumber \\
\vec{s}[n=3] & =W_{r}G(W_{r}G(W_{r}G(W_{x}\vec{1})))\nonumber \\
\vec{s}[n=4] & =W_{r}G(W_{r}G(W_{r}G(W_{r}G(W_{x}\vec{1}))))\nonumber \\
 & \cdots\label{eq:StandardRNNExampleIIR}
\end{align}
and so forth. Evidently, it is defined for every positive $n$, even
when the input is only a single impulse at $n=0$. 
\end{example}
In practice, it is desirable to approximate a sequence with an infinite
support (IIR), such as Equation \ref{eq:CanonicalRNN} or Equation
\ref{eq:StandardRNN}, by a ``Finite Impulse Response'' (FIR) sequence.
The rationale is that FIR systems have certain advantages over IIR
systems. One advantage is guaranteed stability -- FIR systems are
intrinsically stable. Another advantage is that FIR systems are realizable
with finite computational resources. An FIR system will take a finite
number of steps to compute the output from the input and will require
a finite number of memory locations for storing intermediate results
and various coefficients. Moreover, the computational complexity and
storage requirements of an FIR system are known at design time.

Denote the sequence of the ``ground truth'' output values by $\vec{v}[n]$
for any value of the index, $n$, and let $N$ be the length of the
sequence, $\vec{v}[n]$, where $N$ can be an arbitrarily large integer
(e.g., the total number of samples in the training set, or the number
of inputs presented to the system for inference over the lifetime
of the system, etc.). Suppose that $\vec{v}[n]\left\rfloor _{0\leq n\leq N-1}\right.$
is subdivided into $M$ non-overlapping varying-length segments with
$K_{m}$ samples per segment, where every $K_{m}$ is finite, and
$K_{m}\leq N$. It can be assumed that $M$ is an integer with $M\geq1$
(if needed, the RNN system in Equation \ref{eq:StandardRNN} can be
``padded'' with extra $\vec{x}[n]=\vec{0}$ input terms for this
to hold).

Formally, let $\vec{v}[n]\left\rfloor _{0\leq n\leq N-1}\right.$
be the sequence of the ground truth output values for any value of
the index, $n$, and assume that there exists a partitioning of $\vec{v}[n]\left\rfloor _{0\leq n\leq N-1}\right.$
into $M$ non-overlapping segments, $\vec{v}_{m}[n]$, $0\leq m\leq M-1$:
\begin{align}
\vec{v}[n]\left\rfloor _{0\leq n\leq N-1}\right. & =\sum_{m=0}^{M-1}\vec{v}_{m}[n]\label{eq:RNNTrueOutputMutuallyExclusiveCollectivelyExhaustivePartitioning}
\end{align}
For subdividing a sequence into $M$ non-overlapping segments, consider
a vector-valued ``rectangular'' window function, $\vec{w}_{0}[n]$,
which has the value of $\vec{1}$ within the window $0\leq n\leq K_{0}-1$
and $\vec{0}$ otherwise. In terms of the vector-valued unit step
function, $\vec{u}[n]$, $\vec{w}_{0}[n]$ is defined as:
\begin{align}
\vec{w}_{0}[n] & =\vec{u}[n]-\vec{u}[n-K_{0}]\label{eq:VectorValuedPrototypeBoxCarWindow}
\end{align}
Combining Equation \ref{eq:VectorValuedUnitSampleFunction} with Equation
\ref{eq:VectorValuedPrototypeBoxCarWindow} provides an alternative
(``sampling'') definition of $\vec{w}_{0}[n]$:
\begin{align}
\vec{w}_{0}[n] & =\vec{u}[n]-\vec{u}[n-1]\nonumber \\
 & +\vec{u}[n-1]-\vec{u}[n-2]\nonumber \\
 & +\vec{u}[n-2]-\vec{u}[n-3]\nonumber \\
 & +\ldots+\nonumber \\
 & +\vec{u}[n-(K_{0}-2)]-\vec{u}[n-(K_{0}-1)]\nonumber \\
 & +\vec{u}[n-(K_{0}-1)]-\vec{u}[n-K_{0}]\nonumber \\
 & =\sum_{k=0}^{K_{0}-1}\vec{\delta}[n-k]\label{eq:VectorValuedSegmentBoxCarSampler}
\end{align}
Then from Equation \ref{eq:VectorValuedSegmentBoxCarSampler}, the
RNN sequence can be sampled in its entirety by the full $N$-samples-long
window:
\begin{align}
\vec{w}[n] & =\vec{u}[n]-\vec{u}[n-N]\nonumber \\
 & =\sum_{m=0}^{M-1}\left(\sum_{k=j(m)}^{j(m)+K_{m}-1}\vec{\delta}[n-k]\right)\label{eq:VectorValuedSequenceBoxCarSampler}\\
 & =\sum_{m=0}^{M-1}\vec{w}_{m}[n]\label{eq:VectorValuedSequenceBoxCarSamplerWindow}
\end{align}
where:
\begin{align}
j(m) & =\begin{cases}
\sum_{i=0}^{m-1}K_{i}, & 1\leq m\leq M-1\\
0, & m=0
\end{cases}\label{eq:VariableLengthSegmentFirstIndex}
\end{align}
and:
\begin{align}
\vec{w}_{m}[n] & =\sum_{k=j(m)}^{j(m)+K_{m}-1}\vec{\delta}[n-k]\label{eq:VectorValuedBoxCarSamplerWindow}
\end{align}
Under the change of indices,
\begin{align*}
l & =k-j(m)\\
k & =j(m)+l\\
l & \longrightarrow k
\end{align*}
Equation \ref{eq:VectorValuedBoxCarSamplerWindow} becomes:
\begin{align}
\vec{w}_{m}[n] & =\sum_{k=0}^{K_{m}-1}\vec{\delta}[n-j(m)-k]\label{eq:VectorValuedShiftedBoxCarSampler}
\end{align}
Equation \ref{eq:VectorValuedShiftedBoxCarSampler} indicates that
each $\vec{w}_{m}[n]$ is a rectangular window, whose size is $K_{m}$
samples. Hence, ``extracting'' a\\ $K_{m}$-samples-long segment
with the index, $m$, from the overall ground truth output sequence,
$\vec{v}[n]\left\rfloor _{0\leq n\leq N-1}\right.$, amounts to multiplying
this sequence by $\vec{w}_{m}[n]$: 
\begin{align}
\vec{v}_{m}[n]\left\rfloor _{0\leq n\leq N-1}\right. & =\vec{w}_{m}[n]\odot\vec{v}[n]\left\rfloor _{0\leq n\leq N-1}\right.\label{eq:RNNTrueOutputMultiplicationByVectorValuedBoxCarSamplerWindow}\\
 & =\begin{cases}
\vec{v}[n], & j(m)\leq n\leq j(m)+K_{m}-1\\
\vec{0}, & \textrm{otherwise}
\end{cases}\label{eq:RNNSegmentTrueOutput}
\end{align}
where $j(m)$ is given by Equation \ref{eq:VariableLengthSegmentFirstIndex}.
According to Equation \ref{eq:RNNSegmentTrueOutput}, the segment-level
ground truth output subsequence,\\ $\vec{v}_{m}[n]\left\rfloor _{0\leq n\leq N-1}\right.$,
in Equation \ref{eq:RNNTrueOutputMultiplicationByVectorValuedBoxCarSamplerWindow}
will have non-zero values for the given value of the segment index,
$m$, where $0\leq m\leq M-1$, only when the index, $n$, is in the
range $j(m)\leq n\leq j(m)+K_{m}-1$. This is in agreement with Equation
\ref{eq:RNNTrueOutputMutuallyExclusiveCollectivelyExhaustivePartitioning}. 

Define $\mathcal{Q}(\left\langle \vec{r}[n]\right\rangle )$ as an
invertible map that transforms an ensemble of the readout signals
of the RNN system, $\left\langle \vec{r}[n]\right\rangle $, into
an ensemble of observable output signals, $\left\langle \vec{y}[n]\right\rangle $,
for $0\leq n\leq N-1$:
\begin{align}
\left\langle \vec{y}[n]\right\rangle  & \equiv\mathcal{Q}(\left\langle \vec{r}[n]\right\rangle )\label{eq:RNNOutputFunctionOfWarpedState}
\end{align}
In addition, define $\mathcal{L}$ as an ``objective function''
(or ``merit function'' \citep{Vinyals:EECS-2013-202}) that measures
the cost of the observable output of the system deviating from the
desired ground truth output values, given the input data, supplied
over the entire range of the values of the index, $n$:
\begin{align}
\mathcal{L}\left(\left\langle \vec{y}[n]\right\rangle \left\rfloor _{0\leq n\leq N-1}\right.,\left\langle \vec{v}[n]\right\rangle \left\rfloor _{0\leq n\leq N-1}\right.\right)\label{eq:RNNObjectiveFunctionOfOutputsEnsemble}
\end{align}
where $\left\langle \vec{y}[n]\right\rangle \left\rfloor _{0\leq n\leq N-1}\right.$
denotes the ensemble of all $N$ members of the sequence of the observable
output variables, $\vec{y}[n]$, and\\ $\left\langle \vec{v}[n]\right\rangle \left\rfloor _{0\leq n\leq N-1}\right.$
denotes the ensemble of all $N$ members of the sequence of the ground
truth output values, $\vec{v}[n]$.

As shorthand, combine all parameters of the standard RNN system in
Equation \ref{eq:StandardRNN} under one symbol, $\Theta$:
\begin{align}
\Theta & \equiv\left\{ W_{r},W_{x},\vec{\theta}_{s}\right\} \label{eq:StandardRNNParameters}
\end{align}
\pagebreak{}
\begin{prop}
\label{Proposition:RNNUnfoldingUnrolling}Given the standard RNN system
in Equation \ref{eq:StandardRNN} parameterized by $\Theta$, defined
in Equation \ref{eq:StandardRNNParameters}, assume that there exists
a value of $\Theta$, at which the objective function, $\mathcal{L}$,
defined in Equation \ref{eq:RNNObjectiveFunctionOfOutputsEnsemble}
for an $N$-samples-long sequence, is close to an optimum as measured
by some acceptable bound. Further, assume that there exist non-zero
finite constants, $M$ and $K_{m}$, such that $K_{m}<N$, where $0\leq m\leq M-1$,
and that the ground truth output sequence, $\vec{v}[n]\left\rfloor _{0\leq n\leq N-1}\right.$,
can be partitioned into mutually independent segment-level ground
truth output subsequences, $\vec{v}_{m}[n]\left\rfloor _{0\leq n\leq N-1}\right.$,
for different values of the segment index, $m$, as specified in Equation
\ref{eq:RNNSegmentTrueOutput}. Then a single, reusable RNN cell,
unrolled for an adjustable number of steps, $K_{m}$,\\ is computationally
sufficient for seeking $\Theta$ that optimizes $\mathcal{L}$ over
the training set and for inferring outputs from unseen inputs.
\end{prop}
\begin{proof}
The objective function in Equation \ref{eq:RNNObjectiveFunctionOfOutputsEnsemble}
computes the error in the system's performance during training, validation,
and testing phases as well as tracks its generalization metrics on
the actual application data during the inference phase. By the assumption,
$\mathcal{L}$ can be optimized. This implies that when $\mathcal{L}$
is acceptably close to an optimum, the observable output ensemble
from the RNN system approximates the ground truth output ensemble
within a commensurately acceptable tolerance bound:
\begin{align}
\left\langle \vec{y}[n]\right\rangle \left\rfloor _{0\leq n\leq N-1}\right. & \approx\left\langle \vec{v}[n]\right\rangle \left\rfloor _{0\leq n\leq N-1}\right.\label{eq:RNNOutputsEnsembleApproximatesTrueOutputsEnsemble}
\end{align}
Segmenting the RNN system's output sequence by the same procedure
as was used in Equation \ref{eq:RNNTrueOutputMultiplicationByVectorValuedBoxCarSamplerWindow}
to segment the ground truth output sequence gives: 
\begin{align}
\vec{y}_{m}[n]\left\rfloor _{0\leq n\leq N-1}\right. & =\vec{w}_{m}[n]\odot\vec{y}[n]\left\rfloor _{0\leq n\leq N-1}\right.\label{eq:RNNOutputMultiplicationByVectorValuedBoxCarSamplerWindow}\\
 & =\begin{cases}
\vec{y}[n], & j(m)\leq n\leq j(m)+K_{m}-1\\
\vec{0}, & \textrm{otherwise}
\end{cases}\label{eq:RNNSegmentOutput}
\end{align}
where $j(m)$ is given by Equation \ref{eq:VariableLengthSegmentFirstIndex}.
According to Equation \ref{eq:RNNSegmentOutput}, the segment-level
output subsequence, $\vec{y}_{m}[n]\left\rfloor _{0\leq n\leq N-1}\right.$,
in Equation \ref{eq:RNNOutputMultiplicationByVectorValuedBoxCarSamplerWindow}
will have non-zero values for the given value of the segment index,
$m$, where $0\leq m\leq M-1$, only when the index, $n$, is in the
range $j(m)\leq n\leq j(m)+K_{m}-1$. 

By the assumption that the segment-level ensembles of the ground truth
output subsequences are mutually independent, the objective function
in Equation \ref{eq:RNNObjectiveFunctionOfOutputsEnsemble} is separable
and can be expressed as a set of $M$ independent segment-level components,
$\left\{ \mathcal{L}_{m}\left(\left\langle \vec{y}_{m}[n]\right\rangle \left\rfloor _{0\leq n\leq N-1}\right.,\left\langle \vec{v}_{m}[n]\right\rangle \left\rfloor _{0\leq n\leq N-1}\right.\right)\right\} $,
$0\leq m\leq M-1$, combined by a suitable function, $\mathcal{C}$:
\begin{align}
\mathcal{L}\left(\left\langle \vec{y}[n]\right\rangle \left\rfloor _{0\leq n\leq N-1}\right.,\left\langle \vec{v}[n]\right\rangle \left\rfloor _{0\leq n\leq N-1}\right.\right) & =\mathcal{C}\left(\left\{ \mathcal{L}_{m}\left(\left\langle \vec{y}_{m}[n]\right\rangle \left\rfloor _{0\leq n\leq N-1}\right.,\left\langle \vec{v}_{m}[n]\right\rangle \left\rfloor _{0\leq n\leq N-1}\right.\right)\right\} \left\rfloor _{0\leq m\leq M-1}\right.\right)\label{eq:RNNObjectiveFunctionCombinationOfSegmentOutputsEnsembles}
\end{align}
Then by Equation \ref{eq:RNNOutputsEnsembleApproximatesTrueOutputsEnsemble}
and Equation \ref{eq:RNNObjectiveFunctionCombinationOfSegmentOutputsEnsembles},
\begin{align}
\left\langle \vec{y}_{m}[n]\right\rangle \left\rfloor _{0\leq n\leq N-1}\right. & \approx\left\langle \vec{v}_{m}[n]\right\rangle \left\rfloor _{0\leq n\leq N-1}\right.\label{eq:RNNSegmentOutputEnsembleApproximatesSegmentTrueOutputEnsemble}
\end{align}
for all values of the segment index, $m$, where $0\leq m\leq M-1$.
In other words, the tracking of the ground truth output by the observable
output of the RNN system at the entire $N$-sample ensemble level
must hold at the $K_{m}$-samples-long segment level, too, for all
segments.

Since $\mathcal{Q}(\left\langle \vec{r}[n]\right\rangle )$ is invertible,
\begin{align}
\left\langle \vec{r}_{m}[n]\right\rangle \left\rfloor _{0\leq n\leq N-1}\right. & =\mathcal{Q}^{-1}\left(\left\langle \vec{y}_{m}[n]\right\rangle \left\rfloor _{0\leq n\leq N-1}\right.\right)\label{eq:RNNSegmentWarpedStateEnsembleInverseFunctionOfSegmentOutputEnsemble}
\end{align}
and since the warping function, $G(z)$, in Equation \ref{eq:StandardRNNWarpedState}
is invertible, then for any value of the sample index, $n$,
\begin{align}
\vec{s}_{m}[n]\left\rfloor _{0\leq n\leq N-1}\right. & =G^{-1}\left(\vec{r}_{m}[n]\left\rfloor _{0\leq n\leq N-1}\right.\right)\label{eq:RNNSegmentStateInverseFunctionOfSegmentWarpedState}
\end{align}
According to Equation \ref{eq:RNNOutputsEnsembleApproximatesTrueOutputsEnsemble},
Equation \ref{eq:RNNSegmentWarpedStateEnsembleInverseFunctionOfSegmentOutputEnsemble},
and Equation \ref{eq:RNNSegmentStateInverseFunctionOfSegmentWarpedState},
$\left\langle \vec{y}_{m}[n]\right\rangle \left\rfloor _{0\leq n\leq N-1}\right.$,
$\left\langle \vec{r}_{m}[n]\right\rangle \left\rfloor _{0\leq n\leq N-1}\right.$,
and $\vec{s}_{m}[n]\left\rfloor _{0\leq n\leq N-1}\right.$ are all
functions of random variables. Let $\left\langle \vec{v}_{m}[n]\right\rangle \left\rfloor _{0\leq n\leq N-1}\right.$
and $\left\langle \vec{v}_{l}[n]\right\rangle \left\rfloor _{0\leq n\leq N-1}\right.$
be the ground truth output subsequence ensembles, belonging to any
two segments, whose indices are $m$ and $l$, respectively, with
$m\neq l$. By the assumption, $\left\langle \vec{v}_{m}[n]\right\rangle \left\rfloor _{0\leq n\leq N-1}\right.$
and $\left\langle \vec{v}_{l}[n]\right\rangle \left\rfloor _{0\leq n\leq N-1}\right.$
are independent random variables. Because the functions of independent
variables are also independent, it follows that at the segment level
the observable output signal subsequence ensembles, $\left\langle \vec{y}_{m}[n]\right\rangle \left\rfloor _{0\leq n\leq N-1}\right.$
and $\left\langle \vec{y}_{l}[n]\right\rangle \left\rfloor _{0\leq n\leq N-1}\right.$,
are independent, the readout signal subsequence ensembles, $\left\langle \vec{r}_{m}[n]\right\rangle \left\rfloor _{0\leq n\leq N-1}\right.$
and $\left\langle \vec{r}_{l}[n]\right\rangle \left\rfloor _{0\leq n\leq N-1}\right.$,
are independent, and the state signal subsequences, $\vec{s}_{m}[n]\left\rfloor _{0\leq n\leq N-1}\right.$
and $\vec{s}_{l}[n]\left\rfloor _{0\leq n\leq N-1}\right.$, are independent. 

The mutual independence of the state signal subsequences, $\vec{s}_{m}[n]\left\rfloor _{0\leq n\leq N-1}\right.$,
for different values of the segment index, $m$, places a restriction
on the initial conditions of these subsequences. Specifically, the
initial condition for the state signal subsequence of one segment
cannot be a function of samples belonging to either the state signal
subsequence or the input signal subsequence of another segment for
any value of the index, $n$. 

Performing the element-wise multiplication of the input sequence,
$\vec{x}[n]\left\rfloor _{0\leq n\leq N-1}\right.$, by the sampling
window, $\vec{w}_{m}[n]$, extracts a segment-level input sequence
with the index, $m$:
\begin{align}
\vec{x}_{m}[n]\left\rfloor _{0\leq n\leq N-1}\right. & =\vec{w}_{m}[n]\odot\vec{x}[n]\left\rfloor _{0\leq n\leq N-1}\right.\label{eq:RNNInputMultiplicationByVectorValuedBoxCarSamplerWindow}\\
 & =\begin{cases}
\vec{x}[n], & j(m)\leq n\leq j(m)+K_{m}-1\\
\vec{0}, & \textrm{otherwise}
\end{cases}\label{eq:RNNSegmentInput}
\end{align}
where $j(m)$ is given by Equation \ref{eq:VariableLengthSegmentFirstIndex}.
According to Equation \ref{eq:RNNSegmentInput}, the segment-level
input subsequence, $\vec{x}_{m}[n]\left\rfloor _{0\leq n\leq N-1}\right.$,
in Equation \ref{eq:RNNInputMultiplicationByVectorValuedBoxCarSamplerWindow}
will have non-zero values for the given value of the segment index,
$m$, where $0\leq m\leq M-1$, only when the index, $n$, is in the
range $j(m)\leq n\leq j(m)+K_{m}-1$. 

Due to recursion, the members of the state signal sequence, $\vec{s}_{m}[n]\left\rfloor _{0\leq n\leq N-1}\right.$,
in an RNN system can in general depend on the entire input signal
sequence assembly. However, since under the present assumptions the
segment-level state signal subsequences, $\vec{s}_{m}[n]\left\rfloor _{0\leq n\leq N-1}\right.$
and $\vec{s}_{l}[n]\left\rfloor _{0\leq n\leq N-1}\right.$, belonging
to different segments, are independent, the dependency of $\vec{s}_{m}[n]\left\rfloor _{0\leq n\leq N-1}\right.$
on the input signal must be limited to the same segment-level subsequence
(i.e., with segment index, $m$). If we define $\mathcal{F}\left(\left\langle \vec{x}_{m}[n]\right\rangle \left\rfloor _{j(m)\leq n\leq j(m)+K_{m}-1}\right.\right)$
as a map that transforms the segment-level input signal subsequence
assembly, $\left\langle \vec{x}_{m}[n]\right\rangle \left\rfloor _{0\leq n\leq N-1}\right.$,
into the segment-level state signal subsequence assembly, $\left\langle \vec{s}_{m}[n]\right\rangle \left\rfloor _{j(m)\leq n\leq j(m)+K_{m}-1}\right.$,
then the standard RNN system definition in Equation \ref{eq:StandardRNN}
at the $K_{m}$-samples-long segment level can be expressed as:
\begin{align}
\left\langle \vec{s}_{m}[n]\right\rangle \left\rfloor _{j(m)\leq n\leq j(m)+K_{m}-1}\right. & =\mathcal{F}\left(\left\langle \vec{x}_{m}[n]\right\rangle \left\rfloor _{j(m)\leq n\leq j(m)+K_{m}-1}\right.\right)\label{eq:RNNSegmentAllStateVariablesFunctionOfAllSegmentInputVariables}
\end{align}
Hence, for any $0\leq m,l\leq M-1$ with $m\neq l$, the restriction,
\begin{align}
\vec{s}_{m}[n=j(m)-1] & \perp\begin{cases}
\left\langle \vec{s}_{l}[n]\right\rangle , & j(l)\leq n\leq j(l)+K_{l}-1\\
\left\langle \vec{x}_{l}[n]\right\rangle , & j(l)\leq n\leq j(l)+K_{l}-1
\end{cases}\label{eq:RNNSegmenStateInitialConditionRestriction}
\end{align}
must be enforced in order to satisfy the independence of the segment-level
state signal subsequences. The only way to achieve this is to set
$\vec{s}_{m}[n=j(m)-1]$ to a random vector or to $\vec{0}$. The
latter choice is adopted here for simplicity.

Thus, substituting Equation \ref{eq:RNNSegmentInput} and Equation
\ref{eq:RNNSegmenStateInitialConditionRestriction} into Equation
\ref{eq:StandardRNN} yields the RNN system equations for an individual
segment:
\begin{align}
\vec{s}_{m}[n] & =\begin{cases}
W_{r}\vec{r}_{m}[n-1]+W_{x}\vec{x}_{m}[n]+\vec{\theta}_{s}, & j(m)\leq n\leq j(m)+K_{m}-1\\
\vec{0}, & \textrm{otherwise}
\end{cases}\label{eq:StandardRNNSegment}\\
\vec{r}_{m}[n] & =\begin{cases}
G(\vec{s}_{m}[n]), & j(m)\leq n\leq j(m)+K_{m}-1\\
\vec{0}, & \textrm{otherwise}
\end{cases}\label{eq:StandardRNNSegmentWarpedState}\\
\vec{s}_{m}[n=j(m)-1] & =\vec{0}\label{eq:RNNSegmentStateInitialCondition}\\
0 & \leq m\leq M-1\label{eq:RNNInterSegmentIndex}
\end{align}

\vspace{-2mm}
Making the index substitution, \vspace{-2mm}
\begin{align*}
n & \longrightarrow n+j(m)
\end{align*}
shifts the segment-level subsequences, $\vec{r}_{m}[n]\left\rfloor _{0\leq n\leq N-1}\right.$,
$\vec{s}_{m}[n]\left\rfloor _{0\leq n\leq N-1}\right.$, and $\vec{x}_{m}[n]\left\rfloor _{0\leq n\leq N-1}\right.$,
by $-j(m)$ samples:
\begin{align}
\vec{\tilde{x}}_{m}[n] & \equiv\vec{x}_{m}[n+j(m)]\left\rfloor _{0\leq n\leq N-1}\right.\nonumber \\
 & =\begin{cases}
\vec{x}[n+j(m)], & j(m)\leq n+j(m)\leq j(m)+K_{m}-1\\
\vec{0}, & \textrm{otherwise}
\end{cases}\nonumber \\
 & =\begin{cases}
\vec{x}[n+j(m)], & 0\leq n\leq K_{m}-1\\
\vec{0}, & \textrm{otherwise}
\end{cases}\label{eq:RNNSegmentInputShifted}\\
\vec{\tilde{r}}_{m}[n] & \equiv\vec{r}_{m}[n+j(m)]\left\rfloor _{0\leq n\leq N-1}\right.\label{eq:RNNSegmentWarpedStateShifted}\\
\vec{\tilde{s}}_{m}[n] & \equiv\vec{s}_{m}[n+j(m)]\left\rfloor _{0\leq n\leq N-1}\right.\nonumber \\
 & =\begin{cases}
W_{r}\vec{r}_{m}[n+j(m)-1]+W_{x}\vec{x}_{m}[n+j(m)]+\vec{\theta}_{s}, & j(m)\leq n+j(m)\leq j(m)+K_{m}-1\\
\vec{0}, & \textrm{otherwise}
\end{cases}\nonumber \\
 & =\begin{cases}
W_{r}\vec{\tilde{r}}_{m}[n-1]+W_{x}\vec{\tilde{x}}_{m}[n]+\vec{\theta}_{s}, & 0\leq n\leq K_{m}-1\\
\vec{0}, & \textrm{otherwise}
\end{cases}\label{eq:StandardRNNSegmentShifted}\\
\vec{\tilde{s}}_{m}[n=-1] & \equiv\vec{s}_{m}[n+j(m)=j(m)-1]=\vec{0}\label{eq:RNNSegmentStateInitialConditionShifted}
\end{align}
Simplified, these equations reduce to the form of the standard RNN
system, unrolled for $K_{m}$ steps, for any segment with the index,
$m$, where $0\leq m\leq M-1$:
\begin{align}
\vec{\tilde{s}}_{m}[n=-1] & =\vec{0}\label{eq:StandardRNNStepSegmentStateInitialCondition}\\
\vec{\tilde{s}}_{m}[n] & =\begin{cases}
W_{r}\vec{\tilde{r}}_{m}[n-1]+W_{x}\vec{\tilde{x}}_{m}[n]+\vec{\theta}_{s}, & 0\leq n\leq K_{m}-1\\
\vec{0}, & \textrm{otherwise}
\end{cases}\label{eq:StandardRNNStepSegment}\\
\vec{\tilde{r}}_{m}[n] & =\begin{cases}
G(\vec{\tilde{s}}_{m}[n]), & 0\leq n\leq K_{m}-1\\
\vec{0}, & \textrm{otherwise}
\end{cases}\label{eq:StandardRNNStepSegmentWarpedState}\\
\vec{\tilde{x}}_{m}[n] & =\begin{cases}
\vec{x}[n+j(m)], & 0\leq n\leq K_{m}-1\\
\vec{0}, & \textrm{otherwise}
\end{cases}\label{eq:StandardRNNStepSegmentInput}
\end{align}
It follows that the shifted segment-level state signal subsequences,
$\vec{\tilde{s}}_{m}[n]\left\rfloor _{0\leq n\leq K_{m}-1}\right.$,
for different values of the segment index, $m$, where $0\leq m\leq M-1$,
are mutually independent. In addition, from Equation \ref{eq:StandardRNNStepSegmentWarpedState},
Equation \ref{eq:StandardRNNStepSegment}, and Equation \ref{eq:StandardRNNStepSegmentInput},
the non-zero values of the resulting sequences, $\vec{\tilde{r}}_{m}[n]\left\rfloor _{0\leq n\leq K_{m}-1}\right.$,
$\vec{\tilde{s}}_{m}[n]\left\rfloor _{0\leq n\leq K_{m}-1}\right.$,
and $\vec{\tilde{x}}_{m}[n]\left\rfloor _{0\leq n\leq K_{m}-1}\right.$,
are confined to $0\leq n\leq K_{m}-1$ for any value of the segment
index, $m$, where $0\leq m\leq M-1$.

As the sample index, $n,$ traverses the segment-level range, $0\leq n\leq K_{m}-1$,
for every segment with the index, $m$, where $0\leq m\leq M-1$,
the input subsequence, $\vec{\tilde{x}}_{m}[n]$, takes on all the
available values of the input sequence, $\vec{x}[n]$, segment by
segment. Similarly to the original RNN system in Equation \ref{eq:StandardRNN},
the input signal, $\vec{\tilde{x}}_{m}[n]$ (the external driving
force), is the only independent variable of the RNN system, unrolled
for $K_{m}$ steps, in Equation \ref{eq:StandardRNNStepSegment}.
Together with the mutual independence of $\vec{\tilde{s}}_{m}[n]\left\rfloor _{0\leq n\leq K_{m}-1}\right.$
for different segments, this makes the computations of the RNN system,
unrolled for $K_{m}$ steps, generic for all segments. The only signal
that retains the dependence on the segment index, $m$, is the input.
Dropping the segment subscript, $m$, from the variables representing
the state signal and the readout signal results in the following prototype
formulation of the RNN system, unrolled for $K_{m}$ steps:\vspace{2mm}
\begin{align}
\vec{\tilde{s}}[n=-1] & =\vec{0}\label{eq:StandardRNNStepPrototypeStateInitialCondition}\\
\vec{\tilde{s}}[n] & =\begin{cases}
W_{r}\vec{\tilde{r}}[n-1]+W_{x}\vec{\tilde{x}}_{m}[n]+\vec{\theta}_{s}, & 0\leq n\leq K_{m}-1\\
\vec{0}, & \textrm{otherwise}
\end{cases}\label{eq:StandardRNNStepPrototype}\\
\vec{\tilde{r}}[n] & =\begin{cases}
G(\vec{\tilde{s}}[n]), & 0\leq n\leq K_{m}-1\\
\vec{0}, & \textrm{otherwise}
\end{cases}\label{eq:StandardRNNStepPrototypeWarpedState}\\
\vec{\tilde{x}}_{m}[n] & =\begin{cases}
\vec{x}[n+j(m)], & 0\leq n\leq K_{m}-1\\
\vec{0}, & \textrm{otherwise}
\end{cases}\label{eq:StandardRNNStepPrototypeInput}\\
0 & \leq m\leq M-1\label{eq:StandardRNNStepPrototypeInterSegmentIndex}
\end{align}
where $j(m)$ is given by Equation \ref{eq:VariableLengthSegmentFirstIndex}.

The same prototype variable-length RNN computation, unrolled for $K_{m}$
steps, can process all segments, one at a time. After initializing
the segment's state signal using Equation \ref{eq:StandardRNNStepPrototypeStateInitialCondition}
and selecting the input samples for the segment using Equation \ref{eq:StandardRNNStepPrototypeInput},
Equation \ref{eq:StandardRNNStepPrototype} and Equation \ref{eq:StandardRNNStepPrototypeWarpedState}
are evaluated for $K_{m}$ steps, $0\leq n\leq K_{m}-1$. This procedure
can then be applied to the next segment using the same computational
module, and then to the next segment, and so on, until the inputs
comprising all $M$ segments have been processed. Moreover, the mutual
independence of the segments facilitates parallelism, whereby the
computation of $\vec{\tilde{s}}_{m}[n]$ , $\vec{\tilde{r}}_{m}[n]$,
and $\vec{y}_{m}[n]$ for $0\leq n\leq K_{m}-1$ can be carried out
for all $M$ segments concurrently.\pagebreak{}
\end{proof}
\vspace{2mm}

\begin{rem}
\label{Proposition:RNNUnfoldingUnrolling:Remark:NoAttentionNetworksNoMultipleLSTMNetworks}Proposition
\ref{Proposition:RNNUnfoldingUnrolling} and its proof do not formally
address advanced RNN architectures, such as Gated Recurrent Unit (GRU),
Attention Networks, and complex models comprised of multiple LSTM
networks.
\end{rem}
\vspace{2mm}

\begin{rem}
\label{Proposition:RNNUnfoldingUnrolling:Remark:NoInputIndependenceRequirement}While
the segment-level state signal subsequences, $\vec{\tilde{s}}_{m}[n]\left\rfloor _{0\leq n\leq K_{m}-1}\right.$,
are independent, there is no independence requirement on the input
signal subsequences, $\vec{\tilde{x}}_{m}[n]$, belonging to the different
segments. It has been shown that dependencies in the input signal
can be de-correlated by appropriately trained hidden layer weights,
$W_{x}$, thus maintaining the independence of $\vec{\tilde{s}}_{m}[n]\left\rfloor _{0\leq n\leq K_{m}-1}\right.$
\citep{sanger:optimal1989,journals/tnn/SherstinskyP96}.
\end{rem}
\vspace{2mm}

\begin{rem}
\label{Proposition:RNNUnfoldingUnrolling:Remark:Undersampling}It
is important to emphasize that no IIR-to-FIR conversion method is
optimal in the absolute sense. Finding an optimal FIR approximation
to an IIR system can only be done with respect to a certain measure
of fidelity and performance. In practice, one must settle for an approximation
to the ideal form of the output signal. The success of an approximation
technique depends on the degree to which the resulting FIR system
can be adapted to fit the specific data distribution and achieve acceptable
quality metrics for the given application's requirements. 

Unrolling (or unfolding) for a finite number of steps is a standard,
straightforward technique for approximating RNNs by FIR sequences.
However, due to the truncation inherent in limiting the number of
steps, the resulting unfolded RNN model introduces artificial discontinuities
in the approximated version of the target output sequence. In general,
the more steps are included in the unrolled RNN subsequence, the closer
it can get to the desired output samples, but the less efficient the
system becomes, due to the increased number of computations. Nevertheless,
if the underlying distribution governing the application generates
the sequence under consideration as a series of independent segments
(subsequences), then by Proposition \ref{Proposition:RNNUnfoldingUnrolling},
an unfolded RNN model aligned with each segment can be trained to
reproduce outputs from inputs in a way that aims to satisfy the appropriate
criteria of merit faithfully. In this sense, Proposition \ref{Proposition:RNNUnfoldingUnrolling}
for RNNs loosely resembles in spirit the Sampling Theorem in the field
of Discrete-Time Signal Processing \citep{AVO89}. The method of unrolling
is also applicable to the scenarios where attention can be restricted
to those ``present'' sections of the output that are influenced
to an arbitrarily small extent by the portions of the ``past'' or
the ``future'' of the input beyond some arbitrarily large but finite
step \citep{scd/mit/Bose1956}. As a matter of fact, in certain situations,
the raw data set may be amenable to pre-processing, without losing
much of the essential information. Suppose that after a suitable cleanup,
treating the overall sequence as a collection of independent segments
becomes a reasonable assumption. Then by Proposition \ref{Proposition:RNNUnfoldingUnrolling},
the adverse effects of truncation can be reduced by adjusting the
number of samples comprising the window of the unrolled RNN system.
Moreover, whenever Proposition \ref{Proposition:RNNUnfoldingUnrolling}
applies, the segments can be processed by the unrolled RNN system
in any order (because they are assumed to be independent). This flexibility
is utilized by the modules that split the original data set into segments
and feed the batches of segmented training samples to the computational
core of system.

Conversely, if the assumptions of Proposition \ref{Proposition:RNNUnfoldingUnrolling}
are violated, then truncating the unrolling will prevent the model
from adequately fitting the ground truth output. To illustrate this
point, suppose that an RNN system, unrolled for a relatively few steps,
is being used to fit the target sequence that exhibits extremely long-range
dependencies. The unrolled RNN subsequence will be trained under the
erroneous assumptions, expecting the ground truth to be a series of
short independent subsequences. However, because of its relatively
narrow window, this RNN subsequence will not be able to encompass
enough samples to capture the dependencies present in the actual data.
Under-sampling the distribution will limit the flow of information
from the training samples to the parameters of the model, leaving
it in the constant state of making poor predictions. As a symptom,
the model will repeatedly encounter unexpected variations during training,
causing the objective function to oscillate, never converging to an
adequate optimum. During inference, the generated sequence will suffer
from severe jitter and distortion when compared to the expected output.
\end{rem}
\vspace{2mm}

\begin{rem}
\label{Proposition:RNNUnfoldingUnrolling:Remark:WeakDependenceForModeling}According
to Proposition \ref{Proposition:RNNUnfoldingUnrolling}, the RNN unrolling
technique is justified by partitioning a single output sequence into
multiple independent subsequences and placing restrictions on the
initialization of the state between subsequences. However, adhering
to these conditions may be problematic in terms of modeling sequences
in practical applications. Oftentimes, the output subsequences exhibit
some inter-dependence and/or the initial state of one subsequence
is influenced by the final state of another subsequence. In practice,
if the choice for the initial conditions of the state of subsequences
is consistent with the process by which the application generates
the samples of the input sequence, then a favorable subdivision of
the output sequence into acceptably independent subsequences can be
found empirically through experimentation and statistical analysis.\pagebreak{}
\end{rem}

\section{RNN Training Difficulties\label{sec:RNN-Training-Difficulties}}

Proposition \ref{Proposition:RNNUnfoldingUnrolling} establishes that
Equations \ref{eq:StandardRNNStepPrototypeStateInitialCondition}
-- \ref{eq:StandardRNNStepPrototypeInterSegmentIndex} together with
Equation \ref{eq:VariableLengthSegmentFirstIndex} specify the truncated
unrolled RNN system that realizes the standard RNN system, given by
Equation \ref{eq:StandardRNN} and Equation \ref{eq:StandardRNNWarpedState}.
We now segue to the analysis of the training technique for obtaining
the weights in the truncated unrolled RNN system, with the focus on
Equation \ref{eq:StandardRNNStepPrototype} and Equation \ref{eq:StandardRNNStepPrototypeWarpedState}.

Once the infinite RNN sequence in Equation \ref{eq:StandardRNN} is
truncated (or unrolled to a finite length), the resulting system,
given in Equation \ref{eq:StandardRNNStepPrototype}, becomes inherently
stable. However, RNN systems are problematic in practice, despite
their stability. During training, they suffer from the well-documented
issues, known as ``vanishing gradients'' and ``exploding gradients''
\citep{hochreiter1997long,Hochreiter01gradientflow,pascanu2013difficulty}.
These difficulties become pronounced when the dependencies in the
target subsequence span a large number of samples, requiring the window
of the unrolled RNN model to be commensurately wide in order to capture
these long-range dependencies.

Truncated unrolled RNN systems, such as Equation \ref{eq:StandardRNNStepPrototype},
are commonly trained using ``Back Propagation Through Time'' (BPTT),
which is the ``Back Propagation'' technique adapted for sequences
\citep{Werbos:88gasmarket,werbos:bptt1990,sutskever2012training,phd/ca/pascanu2014}.
The essence of Back Propagation is the repeated application of the
chain rule of differentiation. Computationally, the action of unrolling
Equation \ref{eq:StandardRNN} for $K$ steps amounts to converting
its associated directed graph having a delay and a cycle, into a directed
acyclic graph (DAG) corresponding to Equation \ref{eq:StandardRNNStepPrototype}.
For this reason, while originally Back Propagation was restricted
to feedforward networks only, subsequently, it has been successfully
applied to recurrent networks by taking advantage of the very fact
that for every recurrent network there exists an equivalent feedforward
network with identical behavior for a finite number of steps \citep{rhw1985,RumelMc:86,reason:MinPap90a}.

As a supervised training algorithm, BPTT utilizes the available $\vec{\tilde{x}}_{m}[n]$
and $\vec{\tilde{r}}[n]$ data pairs (or the respective pairs of some
mappings of these quantities) in the training set to compute the parameters
of the system, $\Theta$, defined in Equation \ref{eq:StandardRNNParameters},
so as to optimize an objective function, $E$, which depends on the
readout  signal, $\vec{\tilde{r}}[n]$, at one or more values of the
index, $n$. If Gradient Descent (or another ``gradient type'' algorithm)
is used to optimize $E$, then BPTT provides a consistent procedure
for deriving the elements of $\frac{\partial E}{\partial\Theta}$
through a repeated application of the chain rule\footnote{The name ``Back Propagation Through Time'' reflects the origins
of recurrent neural networks in continuous-time domain and differential
equations. While not strictly accurate, given the discrete nature
of the RNN system under consideration, ``Back Propagation Through
Time'' is easy to remember, carries historical significance, and
should not be the source of confusion.}.

By assuming that the conditions of Proposition \ref{Proposition:RNNUnfoldingUnrolling}
apply, the objective function, $E$, takes on the same form for all
segments. Let us now apply BPTT to Equation \ref{eq:StandardRNNStepPrototype}.
Suppose that $E$ depends on the readout  signal, $\vec{\tilde{r}}[n]$,
at some specific value of the index, $n$. Then it is reasonable to
wish to measure the total gradient of $E$ with respect to $\vec{\tilde{r}}[n]$:
\begin{align}
\vec{\chi}[n] & \equiv\vec{\nabla}_{\vec{\tilde{r}}[n]}E=\frac{\partial E}{\partial\vec{\tilde{r}}[n]}\label{eq:StandardRNNStepPrototypeDelEDelWarpedStateTotal}
\end{align}
Since $\vec{\tilde{r}}[n]$ is explicitly dependent on $\vec{\tilde{s}}[n]$,
it follows that $\vec{\tilde{s}}[n]$ also influences $E$, and one
should be interested in measuring the total gradient of $E$ with
respect to $\vec{\tilde{s}}[n]$:
\begin{align}
\vec{\psi}[n] & \equiv\vec{\nabla}_{\vec{\tilde{s}}[n]}E=\frac{\partial E}{\partial\vec{\tilde{s}}[n]}\label{eq:StandardRNNStepPrototypeDelEDelStateTotal}
\end{align}
Quite often in practice, the overall objective function is defined
as the sum of separate contributions involving the readout  signal,
$\vec{\tilde{r}}[n]$, at each individual value of the index, $n$:
\begin{align}
E & =\sum_{n=0}^{K_{m}-1}E(\vec{\tilde{r}}[n])\label{eq:StandardRNNStepPrototypeObjectiveFunctionSumOfImmediateWarpedStateComponents}
\end{align}
Because of the presence of the individual penalty terms, $E(\vec{\tilde{r}}[n])$,
in Equation \ref{eq:StandardRNNStepPrototypeObjectiveFunctionSumOfImmediateWarpedStateComponents}
for the overall objective function of the system, it may be tempting
to use the chain rule directly with respect to $\vec{\tilde{r}}[n]$
in isolation and simply conclude that $\vec{\chi}[n]$ in Equation
\ref{eq:StandardRNNStepPrototypeDelEDelWarpedStateTotal} is equal
to $\frac{\partial E(\vec{\tilde{r}}[n])}{\partial\vec{\tilde{r}}[n]}\odot\frac{dG_{d}(\vec{z})}{d\vec{z}}\rfloor_{z=\vec{\tilde{s}}[n]}$,
where the $\odot$ operator denotes the element-wise vector product.
However, this would miss an important additional component of the
gradient with respect to the state signal. The subtlety is that for
an RNN, the state signal, $\vec{\tilde{s}}[n]$, at $n=k$ also influences
the state signal, $\vec{\tilde{s}}[n]$, at $n=k+1$ \citep{phd/de/Graves2008,pascanu2013difficulty}.
The dependency of $\vec{\tilde{s}}[n+1]$ on $\vec{\tilde{s}}[n]$
through $\vec{\tilde{r}}[n]$ becomes apparent by rewriting Equation
\ref{eq:StandardRNNStepPrototype} at the index, $n+1$: 
\begin{align}
\vec{\tilde{s}}[n+1] & =W_{r}\vec{\tilde{r}}[n]+W_{x}\vec{\tilde{x}}_{m}[n+1]+\vec{\theta}_{s}\label{eq:StandardRNNStepPrototypeIcrementedIndex}\\
\vec{\tilde{r}}[n] & =G(\vec{\tilde{s}}[n])\nonumber \\
\vec{\tilde{r}}[n+1] & =G(\vec{\tilde{s}}[n+1])\nonumber 
\end{align}
Hence, accounting for both dependencies, while applying the chain
rule, gives the expressions for the total partial derivative of the
objective function with respect to the readout  signal and the state
signal at the index, $n$:
\begin{align}
\vec{\chi}[n] & =\frac{\partial E(\vec{\tilde{r}}[n])}{\partial\vec{\tilde{r}}[n]}+W_{r}\vec{\psi}[n+1]\label{eq:StandardRNNStepPrototypeDelEDelWarpedState}\\
\vec{\psi}[n] & =\vec{\chi}[n]\odot\frac{dG(\vec{z})}{d\vec{z}}\rfloor_{z=\vec{\tilde{s}}[n]}\label{eq:StandardRNNStepPrototypeDelEDelState}\\
 & =\left(\frac{\partial E(\vec{\tilde{r}}[n])}{\partial\vec{\tilde{r}}[n]}+W_{r}\vec{\psi}[n+1]\right)\odot\frac{dG(\vec{z})}{d\vec{z}}\rfloor_{z=\vec{\tilde{s}}[n]}\label{eq:StandardRNNStepPrototypeDelEDelStateExpanded}
\end{align}
Equation \ref{eq:StandardRNNStepPrototypeDelEDelWarpedState} and
Equation \ref{eq:StandardRNNStepPrototypeDelEDelStateExpanded} show
that the total partial derivatives of the objective function form
two sequences, which progress in the ``backward'' direction of the
index, $n$. These sequences represent the dual counterparts of the
sequence generated by unrolling Equation \ref{eq:StandardRNNStepPrototype}
in the ``forward'' direction of the index, $n$. Therefore, just
as Equation \ref{eq:StandardRNNStepPrototype} requires the initialization
of the segment's state signal using Equation \ref{eq:StandardRNNStepPrototypeStateInitialCondition},
the sequence formed by the total partial derivative of the objective
function with respect to the state signal (commonly designated as
the ``the error gradient'') requires that Equation \ref{eq:StandardRNNStepPrototypeDelEDelWarpedState}
must also be initialized:
\begin{align}
\vec{\psi}[n=K_{m}] & =\vec{0}\label{eq:StandardRNNStepPrototypeDelEDelStateInitialCondition}
\end{align}
Applying the chain rule to Equation \ref{eq:StandardRNNStepPrototype}
and using Equation \ref{eq:StandardRNNStepPrototypeDelEDelStateExpanded},
gives the expressions for the derivatives of the model's parameters:
\begin{align}
\frac{\partial E}{\partial\Theta}[n] & =\left\{ \frac{\partial E}{\partial W_{r}}[n],\frac{\partial E}{\partial W_{x}}[n],\frac{\partial E}{\partial\vec{\theta}_{s}}[n]\right\} \label{eq:StandardRNNStepPrototypeDelEDelTheta}\\
\frac{\partial E}{\partial W_{r}}[n] & =\vec{\psi}[n]\vec{\tilde{r}}^{T}[n-1]\label{eq:StandardRNNStepPrototypeDelEDelWeightMatrixWarpedState}\\
\frac{\partial E}{\partial W_{x}}[n] & =\vec{\psi}[n]\vec{\tilde{x}}_{m}^{T}[n]\label{eq:StandardRNNStepPrototypeDelEDelWeightMatrixInput}\\
\frac{\partial E}{\partial\vec{\theta}_{s}}[n] & =\vec{\psi}[n]\label{eq:StandardRNNStepPrototypeDelEDelBiasVector}\\
\frac{dE}{d\Theta} & =\sum_{n=0}^{K_{m}-1}\frac{\partial E}{\partial\Theta}[n]\label{eq:StandardRNNDEDModelParametersSegmentTotal}
\end{align}
Note that for an RNN cell, unrolled for $K_{m}$ steps in order to
cover a segment containing $K_{m}$ training samples, the same set
of the model parameters, $\Theta$, is shared by all the steps. This
is because $\Theta$ is the parameter of the RNN system as a whole.
Consequently, the total derivative of the objective function, $E$,
with respect to the model parameters, $\Theta$, has to include the
contributions from all steps of the unrolled sequence. This is captured
in Equation \ref{eq:StandardRNNDEDModelParametersSegmentTotal}, which
can now be used as part of optimization by Gradient Descent. Another
key observation is that according to Equation \ref{eq:StandardRNNStepPrototypeDelEDelWeightMatrixWarpedState},
Equation \ref{eq:StandardRNNStepPrototypeDelEDelWeightMatrixInput},
and Equation \ref{eq:StandardRNNStepPrototypeDelEDelBiasVector},
all of the quantities essential for updating the parameters of the
system, $\Theta$, during training are directly proportional to $\vec{\psi}[n]$. 

When the RNN system is trained using BPTT, the error gradient signal
flows in the reverse direction of the index, $n$, from that of the
sequence itself. Let $\left\langle \vec{\psi}[k]\right\rangle \left\rfloor _{0\leq k<n}\right.$
denote all terms of the sequence, each of whose elements, $\vec{\psi}[k]$,
is the gradient of $E$ with respect to the state signal, $\vec{\tilde{s}}[k]$,
at the index, $k$, for all $k<n$, ending at $\vec{\psi}[k=0]$,
the start of the sequence. Then Equation \ref{eq:StandardRNNStepPrototypeDelEDelStateExpanded}
reveals that $\vec{\psi}[n]$, the gradient of $E$ with respect to
the state signal, $\vec{\tilde{s}}[n]$, at some value of the index,
$n$, no matter how large, can influence the entire ensemble, $\left\langle \vec{\psi}[k]\right\rangle \left\rfloor _{0\leq k<n}\right.$.
Furthermore, by Proposition \ref{Proposition:RNNUnfoldingUnrolling},
$\vec{\psi}[n]$ depends on the truncated ensemble, $\left\langle \vec{\psi}[k]\right\rangle \left\rfloor _{n<k\leq K_{m}-1}\right.$.
Thus, of a particular interest is the fraction of $\vec{\psi}[n]$
that is retained from back propagating $\vec{\psi}[l]$, where $l\gg n$.
This component of the gradient of the objective function is responsible
for adjusting the model's parameters, $\Theta$, in a way that uses
the information available at one sample to reduce the cost of the
system making an error at a distant sample. If these types of contributions
to $\vec{\psi}[n]\left\rfloor _{0\leq n\leq K_{m}-1}\right.$ are
well-behaved numerically, then the model parameters learned by using
the Gradient Descent optimization procedure will able to incorporate
the long-range interactions among the samples in the RNN window effectively
during inference.

Expanding the recursion in Equation \ref{eq:StandardRNNStepPrototypeDelEDelStateExpanded}
from the step with the index, $n,$ to the step with the index, $l\leq K_{m}-1$,
where $l\gg n$, gives:
\begin{align}
\frac{\partial\vec{\psi}[n]}{\partial\vec{\psi}[l]} & =\prod_{k=n+1}^{l}W_{r}\odot\frac{dG(\vec{z})}{d\vec{z}}\rfloor_{z=\vec{\tilde{s}}[k]}\label{eq:StandardRNNStepPrototypeDelEDelStateFlowRateLongRange}
\end{align}
From Equation \ref{eq:StandardRNNStepPrototypeDelEDelStateFlowRateLongRange},
the magnitude of the overall Jacobian matrix, $\frac{\partial\vec{\psi}[n]}{\partial\vec{\psi}[l]}$,
depends on the product of $l-n$ individual Jacobian matrices, $W_{r}\odot\frac{dG(\vec{z})}{d\vec{z}}\rfloor_{z=\vec{\tilde{s}}[k]}$\footnote{By convention, the element-wise multiplication by a vector is equivalent
to the multiplication by a diagonal matrix, in which the elements
of the vector occupy the main diagonal.}. Even though the truncated unrolled RNN system is guaranteed to be
stable by design, since in the case of long-range interactions the
unrolled window size, $K_{m}$, and the distance between the samples
of interest, $l-n$, are both large, the stability analysis is helpful
in estimating the magnitude of $\frac{\partial\vec{\psi}[n]}{\partial\vec{\psi}[l]}$
in Equation \ref{eq:StandardRNNStepPrototypeDelEDelStateFlowRateLongRange}.
If all eigenvalues, $\mu_{i}$, of $W_{r}$ satisfy the requirement
for stability, $0<\mu_{i}<1$, then $\left\Vert W_{r}\right\Vert <1$.
Combined with the fact that $\left\Vert \frac{dG(\vec{z})}{d\vec{z}}\right\Vert <1$
(which follows from the choice of the warping function advocated in
Section \ref{sec:The-Roots-of-RNN}), this yields:
\begin{align}
\left\Vert \frac{\partial\vec{\psi}[n]}{\partial\vec{\psi}[l]}\right\Vert  & \sim\left(\left\Vert W_{r}\right\Vert \cdot\left\Vert \frac{dG(\vec{z})}{d\vec{z}}\right\Vert \right)^{l-n}\\
 & \sim\left\Vert W_{r}\right\Vert ^{l-n}\cdot\left\Vert \frac{dG(\vec{z})}{d\vec{z}}\right\Vert ^{l-n}\approx0\label{eq:StandardRNNStepPrototypeDelEDelStateFlowRateLongRangeVanishingGradient}
\end{align}
Conversely, if at least one eigenvalue of $W_{r}$ violates the requirement
for stability, the term $\left\Vert W_{r}\right\Vert ^{l-n}$ will
grow exponentially. This can lead to two possible outcomes for the
RNN system in Equation \ref{eq:StandardRNNStepPrototype}. In one
scenario, as the state signal, $\vec{\tilde{s}}[n]$, grows, the elements
of the readout signal, $\vec{\tilde{r}}[n]$, eventually saturate
at the ``rails'' (the flat regions) of the warping function. Since
in the saturation regime, $\frac{dG(\vec{z})}{d\vec{z}}=\vec{0}$,
the result is again $\left\Vert \frac{\partial\vec{\psi}[n]}{\partial\vec{\psi}[l]}\right\Vert \approx0$.
In another, albeit rare, scenario, the state signal, $\vec{\tilde{s}}[n]$,
is initially biased in the quasi-linear region of the warping function,
where $\frac{dG(\vec{z})}{d\vec{z}}\neq\vec{0}$. If the input, $\vec{\tilde{x}}_{m}[n]$,
then guides the system to stay in this mode for a large number of
steps, $\left\Vert \frac{\partial\vec{\psi}[n]}{\partial\vec{\psi}[l]}\right\Vert $
will grow, potentially resulting in an overflow. Consequently, training
the standard RNN system on windows spanning many data samples using
Gradient Descent is hampered by either vanishing or exploding gradients,
regardless of whether or not the system is large-signal stable. In
either case, as long as Gradient Descent optimization is used for
training the RNN, regulating $\vec{\psi}[n]$ will be challenging
in practice, leaving no reliable mechanism for updating the parameters
of the system, $\Theta$, in a way that would enable the trained RNN
model to infer both $\vec{\tilde{r}}[n]$ and $\vec{\tilde{r}}[l\gg n]$
optimally\footnote{A detailed treatment of the difficulties encountered in training RNNs
is presented in \citep{pascanu2013difficulty}. The problem is defined
using Equation \ref{eq:StandardRNNStepPrototypeObjectiveFunctionSumOfImmediateWarpedStateComponents},
which leads to the formulas for the gradients of the individual objective
function at each separate step with respect to the model's parameters.
Then the behavior of these formulas as a function of the index of
the step is analyzed, following the approach in \citep{williams:recurrent}.
In contrast, the present analysis follows the method described in
\citep{Hochreiter01gradientflow} and \citep{Werbos:88gasmarket,werbos:bptt1990,graves:05ijcnn,phd/de/Graves2008}.
The total gradient of the objective function with respect to the state
signal at each step is pre-computed using Equation \ref{eq:StandardRNNStepPrototypeDelEDelStateExpanded}.
Then the behavior of the members of this sequence as a function of
the number of steps separating them is analyzed. The results and conclusions
of these dual approaches are, of course, identical.}. The most effective solution so far is the Long Short-Term Memory
(LSTM) cell architecture \citep{hochreiter1997long,phd/de/Graves2008,pascanu2013difficulty,phd/ca/pascanu2014}.

\section{From RNN to Vanilla LSTM Network\label{sec:From-RNN-to-Vanilla-LSTM-Network}}

The Long Short-Term Memory (LSTM) network was invented with the goal
of addressing the vanishing gradients problem. The key insight in
the LSTM design was to incorporate nonlinear, data-dependent controls
into the RNN cell, which can be trained to ensure that the gradient
of the objective function with respect to the state signal (the quantity
directly proportional to the parameter updates computed during training
by Gradient Descent) does not vanish \citep{hochreiter1997long}.
The LSTM cell can be rationalized from the canonical RNN cell by reasoning
about Equation \ref{eq:CanonicalRNN} and introducing changes that
make the system robust and versatile. 

In the RNN system, the observable readout signal of the cell is the
warped version of the cell's state signal itself. A weighted copy
of this warped state signal is fed back from one step to the next
as part of the update signal to the cell's state. This tight coupling
between the readout signal at one step and the state signal at the
next step directly impacts the gradient of the objective function
with respect to the state signal. This impact is compounded during
the training phase, culminating in the vanishing/exploding gradients.

Several modifications to the cell's design can be undertaken to remedy
this situation. As a starting point, it is useful to separate the
right hand side of Equation \ref{eq:CanonicalRNN} (the cell's updated
state signal at a step with the index, $n$) into two parts\footnote{In the remainder of this document, the prototype segment notation
of Equation \ref{eq:StandardRNNStepPrototype} is being omitted for
simplicity. Unless otherwise specified, it is assumed that all operations
are intended to be on the domain of a segment, where the sample index,
$n,$ traverses the steps of the segment, $0\leq n\leq K_{m}-1$,
for every segment with the index, $m$, in $0\leq m\leq M-1$.}:
\begin{align}
\vec{s}[n] & =\vec{\mathcal{F}_{s}}\left(\vec{s}[n-1]\right)+\vec{\mathcal{F}_{u}}\left(\vec{r}[n-1],\vec{x}[n]\right)\label{eq:CanonicalRNNRHSPartsStateUpdate}\\
\vec{r}[n] & =G_{d}(\vec{s}[n])\label{eq:CanonicalRNNRHSWarpedState}\\
\vec{\mathcal{F}_{s}}\left(\vec{s}[n-1]\right) & =W_{s}\vec{s}[n-1]\label{eq:CanonicalRNNRHSPartState}\\
\vec{\mathcal{F}_{u}}\left(\vec{r}[n-1],\vec{x}[n]\right) & =W_{r}\vec{r}[n-1]+W_{x}\vec{x}[n]+\vec{\theta}_{s}\label{eq:CanonicalRNNRHSPartReadoutUpdate}
\end{align}
where $G_{d}(\vec{z})$ is the hyperbolic tangent as before\footnote{Throughout this document, the subscript ``c'' stands for ``control'',
while the subscript ``d'' stands for ``data''.}. The first part, $\vec{\mathcal{F}_{s}}\left(\vec{s}[n-1]\right)$,
carries forward the contribution from the state signal at the previous
step. The second part, $\vec{\mathcal{F}_{u}}\left(\vec{r}[n-1],\vec{x}[n]\right)$,
represents the update information, consisting of the combination of
the readout signal from the previous step and the input signal (the
external driving force) at the current step (plus the bias vector,
$\vec{\theta}_{s}$)\footnote{In discrete systems, the concept of time is only of historical significance.
The proper terminology would be to use the word ``adjacent'' when
referring to quantities at the neighboring steps. Here, the terms
``previous'', ``current'', and ``next'' are sometimes used only
for convenience purposes. The RNN systems can be readily unrolled
in the opposite direction, in which case all indices are negated and
the meaning of ``previous'' and ``next'' is reversed. As a matter
of fact, improved performance has been attained with bi-directional
processing in a variety of applications \citep{schuster97bidirectional,phd/de/Graves2008,conf/interspeech/SakSB14,DBLP:journals/corr/Lipton15,Greff2015LSTMAS,journals/corr/ZhouCWLX16}.
Moreover, if the input data is prepared ahead of time and is made
available to the system in its entirety, then the causality restriction
can be relaxed altogether. This can be feasible in applications, where
the entire training data set or a collection of independent training
data segments is gathered before processing commences. Non-causal
processing (i.e., a technique characterized by taking advantage of
the input data ``from the future'') can be advantageous in detecting
the presence of ``context'' among data samples. Utilizing the information
at the ``future'' steps as part of context for making decisions
at the ``current'' step is often beneficial for analyzing audio,
speech, and text. }. According to Equation \ref{eq:CanonicalRNNRHSPartsStateUpdate},
the state signal blends both sources of information in equal proportions
at every step. These proportions can be made adjustable by multiplying
the two quantities by the special ``gate'' signals,\\ $\vec{g}_{cs}[n]$
(``control state'') and $\vec{g}_{cu}[n]$ (``control update''),
respectively:
\begin{align}
\vec{s}[n] & =\vec{g}_{cs}[n]\odot\vec{\mathcal{F}_{s}}\left(\vec{s}[n-1]\right)+\vec{g}_{cu}[n]\odot\vec{\mathcal{F}_{u}}\left(\vec{r}[n-1],\vec{x}[n]\right)\label{eq:CanonicalRNNPartsStateUpdateGated}\\
\vec{0} & \leq\vec{g}_{cs}[n],\vec{g}_{cu}[n]\leq\vec{1}\label{eq:CanonicalRNNControlStateControlUpdateGateRange}
\end{align}

\begin{figure}[tph]
\includegraphics[scale=0.47]{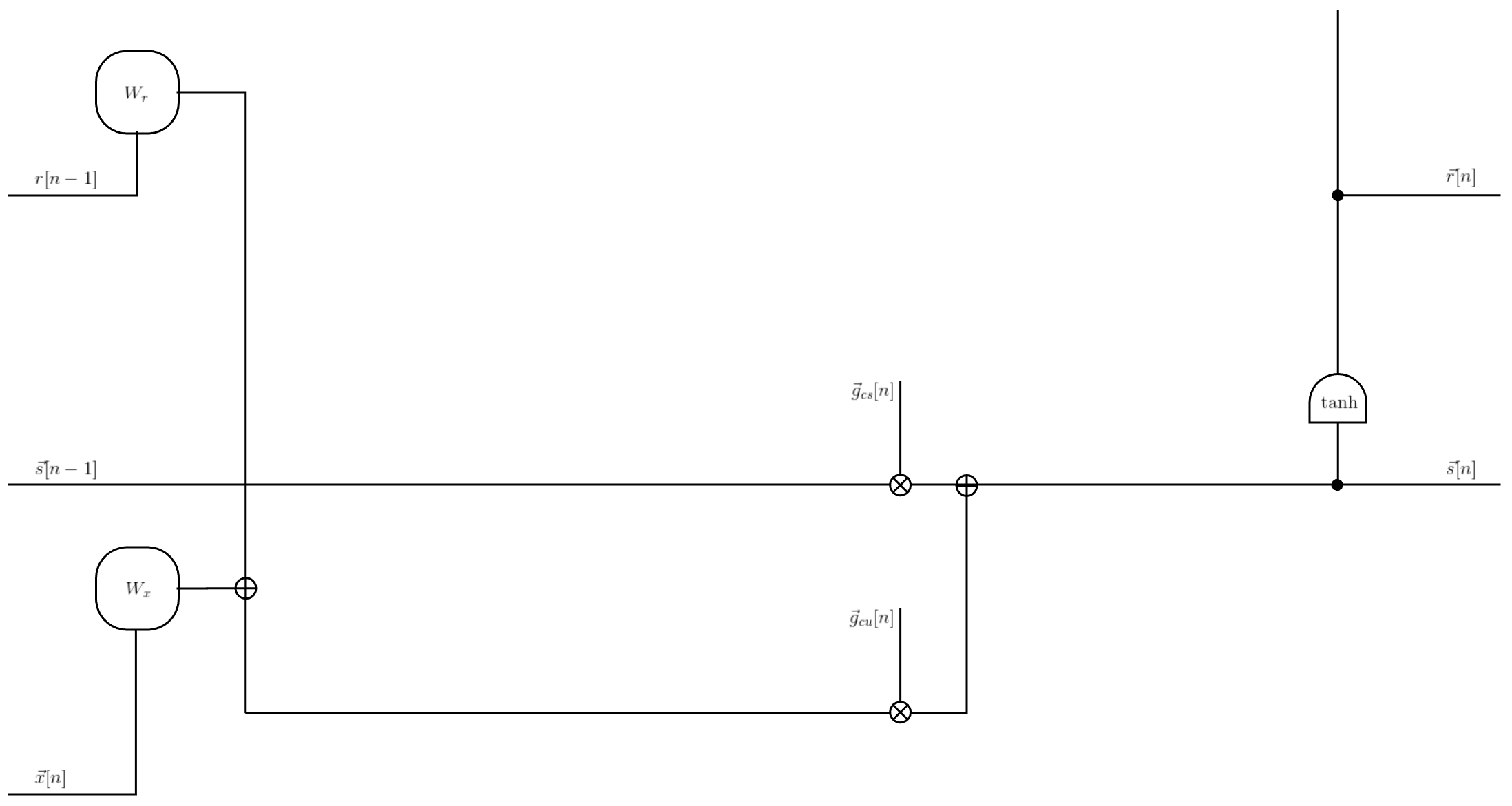}

\caption{Expanding the canonical RNN system by adding the \textquotedblleft control
state\textquotedblright{} gate, $\vec{g}_{cs}[n]$, to control the
amount of the state signal, retained from the previous step and the
\textquotedblleft control update\textquotedblright{} gate, $\vec{g}_{cu}[n]$,
to regulate the amount of the update signal -- to be injected into
the state signal at the current step.\label{fig:CanonicalRNNToLSTMCellTransformationGcsGcuAdded}}
\end{figure}
The elements of gate signals are non-negative fractions. The shorthand
notation, $\vec{g}[n]\in[\vec{0},\vec{1}]$ (alternatively, $\vec{0}\leq\vec{g}[n]\leq\vec{1}$),
means that the values of all elements of a vector-valued gate signal,
$\vec{g}[n]$, at a step with the index, $n$, lie on a closed segment
between $0$ and $1$.

The gate signals, $\vec{g}_{cs}[n]$ and $\vec{g}_{cu}[n]$, in Equation
\ref{eq:CanonicalRNNPartsStateUpdateGated} and Equation \ref{eq:CanonicalRNNControlStateControlUpdateGateRange}
provide a mechanism for exercising a fine-grained control of the two
types of contributions to the state signal at every step. Specifically,
$\vec{g}_{cs}[n]$ makes it possible to control the amount of the
state signal, retained from the previous step, and $\vec{g}_{cu}[n]$
regulates the amount of the update signal -- to be injected into
the state signal at the current step\footnote{The significance of the element-wise control of the ``content''
signals (here called the ``data'' signals) exerted by the gates
in the LSTM network has been independently recognized and researched
by others \citep{journals/corr/abs-1805-03716}. }.

From the derivation of the standard RNN system in Section \ref{sec:The-Roots-of-RNN},
$W_{s}$ in Equation \ref{eq:CanonicalRNNRHSPartState} is a diagonal
matrix with positive fractions, $\frac{1}{\left|a_{ii}\right|}$,
on its main diagonal. Hence, since the elements of $\vec{g}_{cs}[n]$
are also fractions, setting: 
\begin{align}
W_{s} & =I\label{eq:CanonicalRNNRHSIdentityWs}
\end{align}
 in $\vec{g}_{cs}[n]\odot W_{s}$ is acceptable as long as the gate
functions are parametrizable and their parameters are learned during
training. Under these conditions, Equation \ref{eq:CanonicalRNNRHSPartState}
can be simplified to: 
\begin{align}
\vec{\mathcal{F}_{s}}\left(\vec{s}[n-1]\right) & =\vec{s}[n-1]\label{eq:CanonicalRNNRHSPartStateIdentityWs}
\end{align}
so that Equation \ref{eq:CanonicalRNNPartsStateUpdateGated} becomes:
\begin{align}
\vec{s}[n] & =\vec{g}_{cs}[n]\odot\vec{\mathcal{F}_{s}}\left(\vec{s}[n-1]\right)+\vec{g}_{cu}[n]\odot\vec{\mathcal{F}_{u}}\left(\vec{r}[n-1],\vec{x}[n]\right)\nonumber \\
 & =\vec{g}_{cs}[n]\odot\vec{s}[n-1]+\vec{g}_{cu}[n]\odot\vec{\mathcal{F}_{u}}\left(\vec{r}[n-1],\vec{x}[n]\right)\label{eq:CanonicalRNNPartsIdentityWsStateUpdateGated}
\end{align}
Hence, the contribution from the state signal at the previous step
remains fractional, insuring the stability of the overall system.
Diagrammatically, the insertion of the expanded controls from Equation
\ref{eq:CanonicalRNNPartsIdentityWsStateUpdateGated} into the canonical
RNN system of Equation \ref{eq:CanonicalRNN} transforms Figure \ref{fig:CanonicalRNNCellSchematics}
into Figure \ref{fig:CanonicalRNNToLSTMCellTransformationGcsGcuAdded}.

While the update term, $\vec{\mathcal{F}_{u}}\left(\vec{r}[n-1],\vec{x}[n]\right)$,
as a whole is now controlled by $\vec{g}_{cu}[n]$ , the internal
composition of $\vec{\mathcal{F}_{u}}\left(\vec{r}[n-1],\vec{x}[n]\right)$
itself needs to be examined. According to Equation \ref{eq:CanonicalRNNRHSPartReadoutUpdate},
the readout signal from the previous step and the input signal at
the current step constitute the update candidate signal on every step
with the index, $n$, with both of these terms contributing in equal
proportions. The issue with always utilizing $W_{r}\vec{r}[n-1]$
in its entirety is that when $\vec{g}_{cu}[n]\sim1$, $\vec{s}[n-1]$
and $\vec{s}[n]$ become connected through $W_{r}$ and the warping
function. Based on Equation \ref{eq:StandardRNNStepPrototypeDelEDelStateFlowRateLongRange},
this link constrains the gradient of the objective function with respect
to the state signal, thus predisposing the system to the vanishing/exploding
gradients problem. To throttle this feedback path, the readout signal,
$\vec{r}[n]$, will be apportioned by another gate signal, $\vec{g}_{cr}[n]$
(``control readout''), as follows:
\begin{align}
\vec{v}[n] & =\vec{g}_{cr}[n]\odot\vec{r}[n]\label{eq:CanonicalRNNRHSGatedReadout}\\
\vec{0} & \leq\vec{g}_{cr}[n]\leq\vec{1}\label{eq:CanonicalRNNControlReadoutGateRange}
\end{align}
The gating control, $\vec{g}_{cr}[n]$, determines the fractional
amount of the readout signal that becomes the cell's observable value
signal at the step with the index, $n$. Thus, using $\vec{v}[n-1]$
in place of $\vec{r}[n-1]$ in Equation \ref{eq:CanonicalRNNRHSPartReadoutUpdate}
transforms it into: 
\begin{align}
\vec{\mathcal{F}_{u}}\left(\vec{v}[n-1],\vec{x}[n]\right) & =W_{r}\vec{v}[n-1]+W_{x}\vec{x}[n]+\vec{\theta}_{s}\label{eq:CanonicalRNNRHSPartValueUpdate}
\end{align}
The RNN cell schematic diagram, expanded to accommodate the control
readout gate, introduced in Equation \ref{eq:CanonicalRNNRHSGatedReadout},
and the modified recurrence relationship, employed in Equation \ref{eq:CanonicalRNNRHSPartValueUpdate},
appears in Figure \ref{fig:CanonicalRNNToLSTMCellTransformationGcsGcuGcrAdded}.

\begin{figure}[tph]
\includegraphics[scale=0.47]{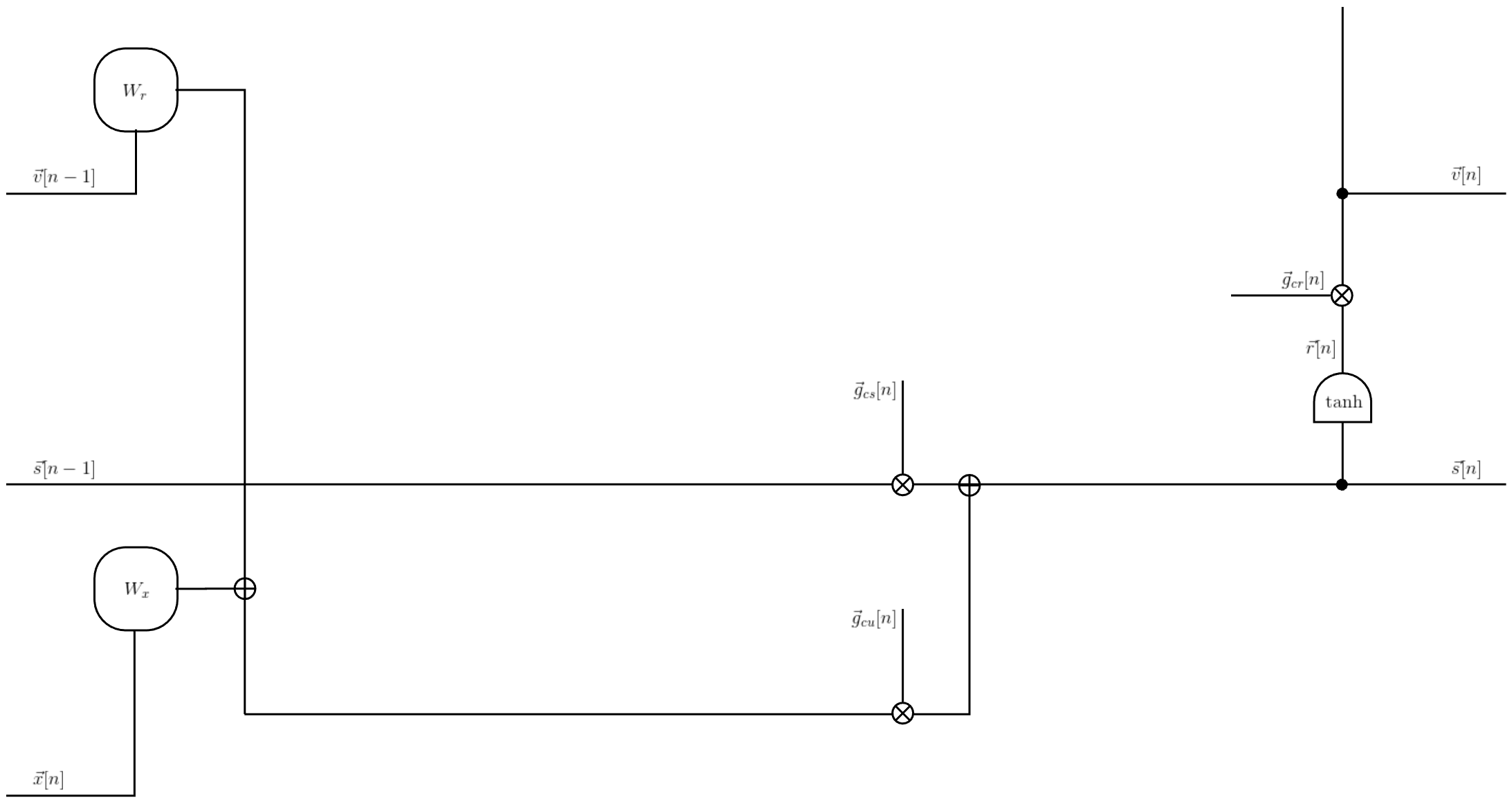}

\caption{The \textquotedblleft control readout\textquotedblright{} gate, $\vec{g}_{cr}[n]$,
determines the fractional amount of the readout signal that becomes
the cell's observable value signal at the current step.\label{fig:CanonicalRNNToLSTMCellTransformationGcsGcuGcrAdded}}

\end{figure}

Even though the external input does not affect the system's stability
or impact its susceptibility to vanishing/exploding gradients, pairing
the input with its own ``control input'' gate makes the system more
flexible. 

Multiplying the external input signal, $\vec{x}[n]$, in Equation
\ref{eq:CanonicalRNNRHSPartReadoutUpdate} by a dedicated gate signal,
$\vec{g}_{cx}[n]$, turns Equation \ref{eq:CanonicalRNNRHSPartValueUpdate}
into: 
\begin{align}
\vec{\mathcal{F}_{u}}\left(\vec{v}[n-1],\vec{x}[n]\right) & =W_{r}\vec{v}[n-1]+\vec{g}_{cx}[n]\odot W_{x}\vec{x}[n]+\vec{\theta}_{s}\label{eq:CanonicalRNNRHSPartValueInputUpdate}
\end{align}

According to Equation \ref{eq:CanonicalRNNRHSGatedReadout} and Equation
\ref{eq:CanonicalRNNRHSPartValueInputUpdate}, utilizing both the
control readout gate, $\vec{g}_{cr}[n]$, and the control input gate,
$\vec{g}_{cx}[n]$, allows for the update term, $\vec{\mathcal{F}_{u}}\left(\vec{v}[n-1],\vec{x}[n]\right)$,
to contain an arbitrary mix of the readout signal and the external
input. The control input gate signal, $\vec{g}_{cx}[n]$, will be
later incorporated as part of extending the Vanilla LSTM cell. For
now, it is assumed for simplicity that $\vec{g}_{cx}[n]=\vec{1}$,
so Equation \ref{eq:CanonicalRNNRHSPartValueInputUpdate} reduces
to Equation \ref{eq:CanonicalRNNRHSPartValueUpdate}.

The dynamic range of the value signal of the cell, $\vec{v}[n]$,
is determined by the readout signal, $\vec{r}[n]$, which is bounded
by the warping nonlinearity, $G_{d}(z)$. In order to maintain the
same dynamic range while absorbing the contributions from the input
signal, $\vec{x}[n]$ (or $\vec{g}_{cx}[n]\odot\vec{x}[n]$ if the
control input gate is part of the system architecture), the aggregate
signal, $\vec{\mathcal{F}_{u}}\left(\vec{v}[n-1],\vec{x}[n]\right)$,
is tempered by the saturating warping nonlinearity, $G_{d}(z)$, so
as to produce the update candidate signal, $\vec{u}[n]$:
\begin{align}
\vec{u}[n] & =G_{d}\left(\vec{\mathcal{F}_{u}}\left(\vec{v}[n-1],\vec{x}[n]\right)\right)\label{eq:CanonicalRNNRHSWarpedUpdateTermUpdateCandidate}
\end{align}
Thus, replacing the update term in Equation \ref{eq:CanonicalRNNPartsIdentityWsStateUpdateGated}
with $\vec{u}[n]$, given by Equation \ref{eq:CanonicalRNNRHSWarpedUpdateTermUpdateCandidate},
finally yields\footnote{Note the notation change: in the rest of the document, the symbol,
$\vec{u}[n]$, has the meaning of the update candidate signal (not
the vector-valued unit step function).}:
\begin{align}
\vec{s}[n] & =\vec{g}_{cs}[n]\odot\vec{s}[n-1]+\vec{g}_{cu}[n]\odot\vec{u}[n]\label{eq:CanonicalRNNPartsStateUpdateCandidateGatedLSTM}
\end{align}
which is a core constituent of the set of formulas defining the cell
of the Vanilla LSTM network. According to Equation \ref{eq:CanonicalRNNPartsStateUpdateCandidateGatedLSTM},
the state signal of the cell at the current step is a weighted combination
of the state signal of the cell at the previous step and the aggregation
of historical and novel update information available at the present
step. The complete data path of the Vanilla LSTM cell, culminating
from fortifying the canonical RNN system with gating controls and
signal containment, is illustrated in Figure \ref{fig:CanonicalRNNToLSTMCellTransformationGcsGcuGcrUAdded}.

\vspace{-3mm}
\begin{figure}[tph]
\includegraphics[scale=0.47]{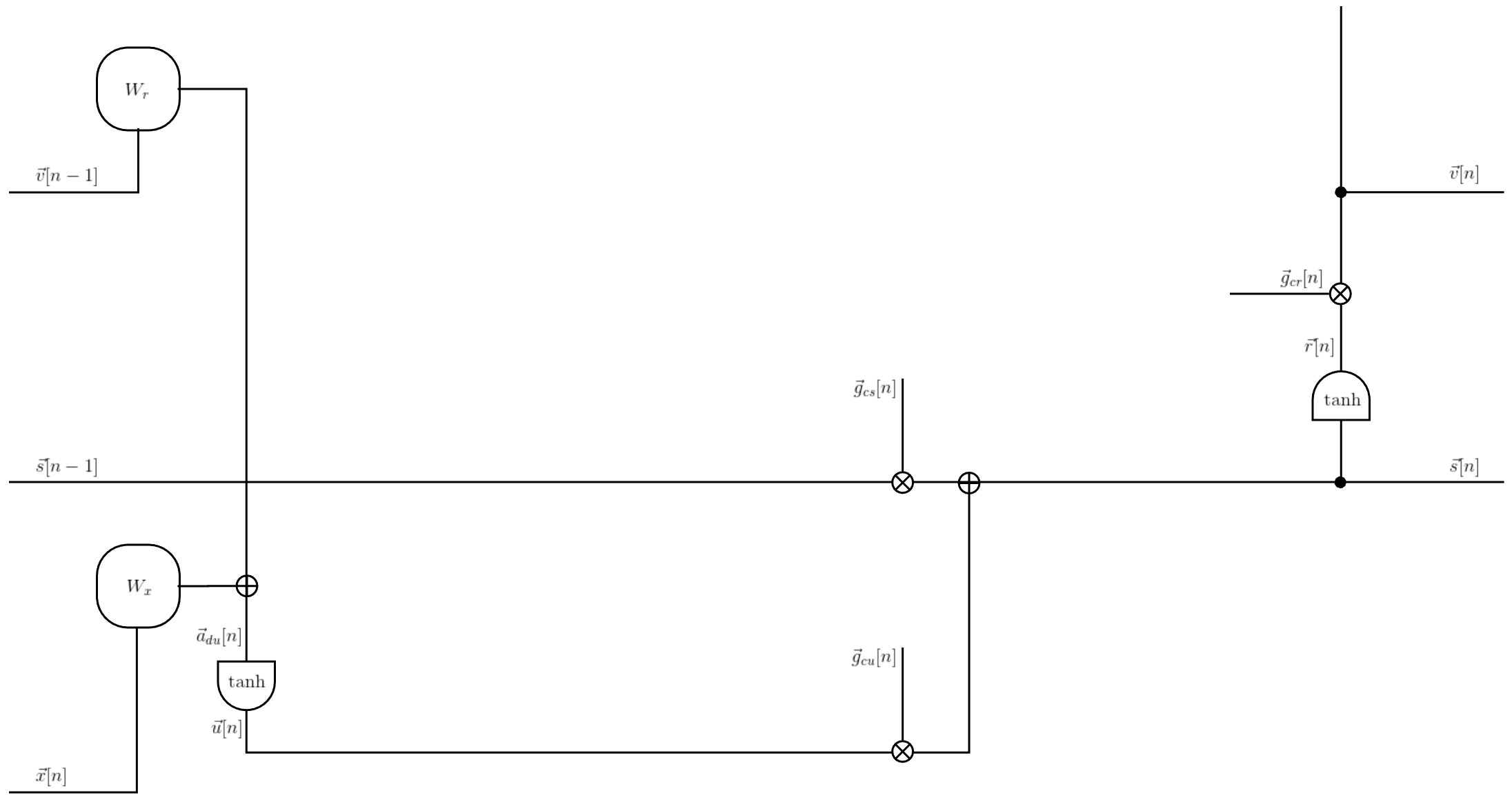}

\caption{In the Vanilla LSTM network, the state signal of the cell at the current
step is a weighted combination of the state signal of the cell at
the previous step and the aggregation of historical and novel update
information available at the present step.\label{fig:CanonicalRNNToLSTMCellTransformationGcsGcuGcrUAdded}}
\end{figure}

\vspace{2mm}

\begin{example}
\label{example:VanillaLSTMAsymptoticCEC}For an idealized illustration
of the ability of the LSTM cell to propagate the error gradient unattenuated,
set $\vec{g}_{cs}[n]$ to $\vec{1}$ and both, $\vec{g}_{cu}[n]$
and $\vec{g}_{cr}[n]$, to $\vec{0}$ for all steps in the segment.
Then $\vec{s}[n]=\vec{s}[n-1]$ and $\vec{\psi}[n]=\vec{\psi}[n+1]$
for all steps in the segment. The inventors of the LSTM network named
this mode the ``Constant Error Carousel'' (CEC) to underscore that
the error gradient is recirculated and the state signal of the cell
is refreshed on every step\footnote{However, the design of the LSTM network does not address explicitly
the exploding gradients problem. During training, the derivatives
can still become excessively large, leading to numerical problems.
To prevent this, all derivatives of the objective function with respect
to the state signal are renormalized to lie within a predefined range
\citep{phd/de/Graves2008,journals/corr/Graves13,pascanu2013difficulty}.}. Essentially, the multiplicative gate units open and close access
to constant error gradient flow through CEC as part of the operation
of the LSTM cell\footnote{Because of the access control functionality provided by the gates,
the LSTM cell is sometimes interpreted as a differentiable version
of the digital static random access memory cell \citep{phd/de/Graves2008}.} \hspace{-0.1mm}\citep{hochreiter1997long,Gers01longshortterm}.
\end{example}
In Section \ref{sec:RNN-Training-Difficulties}, we saw that the error
gradient determines the parameter updates for training the standard
RNN by Gradient Descent. It will become apparent in Section \ref{subsec:Vanilla-LSTM-System-Derivatives-(Backward-Pass)}
that the same relationship holds for the Vanilla LSTM network as well.
The difference is that because of the gates, the function for the
error gradient of the LSTM network accommodates Gradient Descent better
than that of the standard RNN does. As will be shown in Section \ref{subsec:Error-Gradient-Sequences-In-Vanilla-LSTM-System},
under certain provisions regarding the model parameters, the unrolled
Vanilla LSTM cell operates in the CEC mode. If such a parameter combination
emerges during training, then the parameter update information, embedded
in the error gradient signal, will be back-propagated over a large
number of steps of a training subsequence, imparting sensitivity to
the long-range dependencies to the model parameters through the parameter
update step of Gradient Descent. If the training process steers the
model parameters toward causing $\left\Vert \vec{g}_{cs}[n]\right\Vert =1$
(as in Example \ref{example:VanillaLSTMAsymptoticCEC}), then the
LSTM network circumvents the vanishing gradient problem in this asymptotic
case.

\pagebreak Analogously to the standard RNN, the Vanilla LSTM network,
trained by Gradient Descent, can also learn the short-range dependencies
among the samples of the subsequences, comprising the training data.
Suppose that during training the model parameters cause $\left\Vert \vec{g}_{cs}[n]\right\Vert <1$
(unlike in Example \ref{example:VanillaLSTMAsymptoticCEC}). Then,
as will be elaborated in Section \ref{subsec:Error-Gradient-Sequences-In-Vanilla-LSTM-System},
the error gradient signal will decline, eventually vanishing over
a finite number of steps, even if during training $\left\Vert \vec{g}_{cu}[n]\right\Vert >0$
and/or $\left\Vert \vec{g}_{cr}[n]\right\Vert >0$ so as to admit
(by Equation \ref{eq:CanonicalRNNPartsStateUpdateCandidateGatedLSTM})
the contributions from the update candidate signal, $\vec{u}[n]$,
into the composition of the state signal.

It remains to define the expressions for the gate signals, $\vec{g}_{cs}[n]$,
$\vec{g}_{cr}[n]$, and $\vec{g}_{cu}[n]$. Assuming that the system
will be trained with BPTT, all of its constituent functions, including
the functions for the gate signals, must be differentiable. A convenient
function that is continuous, differentiable, monotonically increasing,
and maps the domain $(-\infty,\infty)$ into the range $(0,1)$ is
the logistic function:\vspace{-1mm}
\begin{align}
G_{c}(z) & \equiv\sigma(z)\equiv\frac{1}{1+e^{-z}}\label{eq:LogisticFunctionExponentialForm}\\
 & =\frac{1+\tanh(\frac{z}{2})}{2}\label{eq:LogisticFunctionShiftedScaledReparameterizedTanH}\\
 & =\frac{1+G_{d}(\frac{z}{2})}{2}\label{eq:LogisticFunctionShiftedScaledReparameterizedDataWarpingFunction}
\end{align}
which is a shifted, scaled, and re-parameterized replica of the hyperbolic
tangent, used as the warping function, $G_{d}(z)$, for the data signals
in RNN and LSTM systems. When operating on vector arguments, $G_{c}(\vec{z})$
is computed by applying Equation \ref{eq:LogisticFunctionExponentialForm}
to each element of the vector, $\vec{z}$, separately; the same rule
applies to $G_{d}(\vec{z})$. 

\pagebreak In order to determine the fractional values of the control
signals, $\vec{g}_{cs}[n]$, $\vec{g}_{cu}[n]$, and $\vec{g}_{cr}[n]$,
at the step with the index, $n$, all the data signals, from as close
as possible to the index of the current step, are utilized. Specifically,
for both, $\vec{g}_{cs}[n]$, which determines the fraction of the
state signal, $\vec{s}[n-1]$, from the previous step and $\vec{g}_{cu}[n]$,
which determines the fraction of the update candidate signal, $\vec{u}[n]$,
from the current step, the available data signals are $\vec{s}[n-1]$,
$\vec{v}[n-1]$, and $\vec{x}[n]$. However, note that for $\vec{g}_{cr}[n]$,
which determines the fraction of the readout signal, $\vec{r}[n]$,
from the current step, the available data signals are $\vec{s}[n]$,
$\vec{v}[n-1]$, and $\vec{x}[n]$. This is because by Equation \ref{eq:CanonicalRNNRHSGatedReadout},
$\vec{r}[n]$ is available at the junction of the cell, where $\vec{g}_{cr}[n]$
is computed, and hence, by Equation \ref{eq:CanonicalRNNRHSWarpedState},
$\vec{s}[n]$ is necessarily available. The input to each gate is
presented as a linear combination of all the data signals available
to it:\vspace{-2mm}
\begin{align}
\vec{z}_{cs}[n] & =W_{x_{cs}}\vec{x}[n]+W_{s_{cs}}\vec{s}[n-1]+W_{v_{cs}}\vec{v}[n-1]+\vec{\theta}_{cs}\label{eq:CanonicalRNNControlStateAccumulationLSTM}\\
\vec{z}_{cu}[n] & =W_{x_{cu}}\vec{x}[n]+W_{s_{cu}}\vec{s}[n-1]+W_{v_{cu}}\vec{v}[n-1]+\vec{\theta}_{cu}\label{eq:CanonicalRNNControlUpdateAccumulationLSTM}\\
\vec{z}_{cr}[n] & =W_{x_{cr}}\vec{x}[n]+W_{s_{cr}}\vec{s}[n]+W_{v_{cr}}\vec{v}[n-1]+\vec{\theta}_{cr}\label{eq:CanonicalRNNControlReadoutAccumulationLSTM}
\end{align}
Accumulating the available data signals linearly makes the application
of the chain rule for BPTT straightforward, while providing a rich
representation of the system's data as an input to each gate at every
step. As the model parameters, $\left\{ W_{x_{cr}},W_{x_{cu}},W_{x_{cs}},W_{s_{cs}},W_{v_{cs}},\vec{\theta}_{cs},W_{s_{cu}},W_{v_{cu}},\vec{\theta}_{cu},W_{s_{cr}},W_{v_{cr}},\vec{\theta}_{cr}\right\} $,
in Equation \ref{eq:CanonicalRNNControlStateAccumulationLSTM}, Equation
\ref{eq:CanonicalRNNControlUpdateAccumulationLSTM}, and Equation
\ref{eq:CanonicalRNNControlReadoutAccumulationLSTM} are being trained,
the gate functions, given by:\vspace{-2mm}
\begin{align}
\vec{g}_{cs}[n] & =G_{c}(\vec{z}_{cs}[n])\label{eq:CanonicalRNNControlStateLSTM}\\
\vec{g}_{cu}[n] & =G_{c}(\vec{z}_{cu}[n])\label{eq:CanonicalRNNControlUpdateLSTM}\\
\vec{g}_{cr}[n] & =G_{c}(\vec{z}_{cr}[n])\label{eq:CanonicalRNNControlReadoutLSTM}
\end{align}
become attuned to the flow of and the variations in the training data
through the system at every step. During inference, this enables the
gates to modulate their corresponding data signals adaptively, utilizing
all the available information at every step. In particular, the gates
help to detect and mitigate the detrimental ramifications of artificial
boundaries, which arise in the input sequences, due to the implicit
truncation, caused by unrolling \citep{rabiner1971,parksmcclellanrabiner1973,yamamoto1223586ieee092003}.
The gates make the LSTM system a robust model that compensates for
the imperfections in the external data and is capable of generating
high quality output sequences. 

This concludes the derivation of the Vanilla LSTM network. The next
section presents a formal self-contained summary of the Vanilla LSTM
system, including the equations for training it using BPTT.

\section{The Vanilla LSTM Network Mechanism in Detail\label{sec:The-Vanilla-LSTM-Network-Mechanism-in-Detail}}

\subsection{Overview\label{subsec:Overview}}

Suppose that an LSTM cell is unrolled for $K$ steps. The LSTM cell
at the step with the index, $n$ (in the sequence of $K$ steps),
accepts the input signal, $\vec{x}[n]$, and computes the externally-accessible
(i.e., observable) signal, $\vec{v}[n]$. The internal state signal
of the cell at the step with the index, $n$, is maintained in $\vec{s}[n]$,
which is normally not observable by entities external to the cell\footnote{In certain advanced RNN and LSTM configurations, such as Attention
Networks, the state signal is externally observable and serves as
an important component of the objective function.}. However, the computations, associated with the cell at the next
adjacent step in the increasing order of the index, $n$ (i.e., the
LSTM step at the index, $n+1$), are allowed to access $\vec{s}[n]$,
the state signal of the LSTM cell at the step with the index, $n$.

The key principle of the LSTM cell centers around organizing its internal
operations according to two qualitatively different, yet cooperating,
objectives: data and the control of data. The data components prepare
the candidate data signals (ranging between $-1$ and $1$), while
the control components prepare the ``throttle'' signals (ranging
between $0$ and $1$). Multiplying the candidate data signal by the
control signal apportions the fractional amount of the candidate data
that is allowed to propagate to its intended nodes in the cell. Hence,
if the control signal is $0$, then $0\%$ of the candidate data amount
will propagate. Conversely, if the control signal is 1, then 10$0\%$
of the candidate data amount will propagate. Analogously, for intermediate
values of the control signal (in the range between $0$ and $1$),
the corresponding percentage of the candidate data amount will be
made available to the next function in the cell.

As depicted in Figure \ref{fig:VanillaLSTMCellSchematics}, the Vanilla
LSTM cell contains three candidate-data/control stages: update, state,
and readout.

\pagebreak{}

\subsection{Notation\label{subsec:Notation}}

The following notation is used consistently throughout this section
to define the Vanilla LSTM cell:
\begin{itemize}
\item $n$ -- index of a step in the segment (or subsequence); $n=0,\ldots,K-1$
\item $K$ -- number of steps in the unrolled segment (or subsequence)
\item $G_{c}$ -- monotonic, bipolarly-saturating warping function for
control/throttling purposes (acts as a ``gate'')
\item $G_{d}$ -- monotonic, negative-symmetric, bipolarly-saturating warping
function for data bounding purposes
\item $d_{x}$ -- dimensionality of the input signal to the cell
\item $d_{s}$ -- dimensionality of the state signal of the cell
\item $\vec{x}\in\mathbb{R}^{d_{x}}$ -- the input signal to the cell
\item $\vec{s}\in\mathbb{R}^{d_{s}}$ -- the state signal of the cell
\item $\vec{v}\in\mathbb{R}^{d_{s}}$ -- the observable value signal of
the cell for external purposes (e.g., for connecting one step to the
next adjacent step of the same cell in the increasing order of the
step index, $n$; as input to another cell in the cascade of cells;
for connecting to the signal transformation filter for data output;
etc.)
\item $\vec{a}\in\mathbb{R}^{d_{s}}$ -- an accumulation node of the cell
(linearly combines the signals from the preceding step and the present
step as net input to a warping function at the present step; each
cell contains several purpose-specific control and data accumulation
nodes)
\item $\vec{u}\in\mathbb{R}^{d_{s}}$ -- the update candidate signal for
the state signal of the cell
\item $\vec{r}\in\mathbb{R}^{d_{s}}$ --  the readout candidate signal
of the cell
\item $g\in\mathbb{R}^{d_{s}}$ -- a gate output signal of the cell for
control/throttling purposes
\item $E\in\mathbb{R}$ -- objective (cost) function to be minimized as
part of the model training procedure
\item $\vec{x}^{T}\vec{v}$ -- vector-vector inner product (yields a scalar)
\item $\vec{x}\vec{v}^{T}$ -- vector-vector outer product (yields a matrix)
\item $W\vec{v}$ -- matrix-vector product (yields a vector)
\item $\vec{x}\odot\vec{v}$ -- element-wise vector product (yields a vector)
\end{itemize}

\subsection{Control/Throttling (\textquotedblleft Gate\textquotedblright ) Nodes\label{subsec:Control/Throttling-(Gate)-Nodes}}

The Vanilla LSTM cell uses three gate types:
\begin{itemize}
\item control of the fractional amount of the update candidate signal used
to comprise the state signal of the cell at the present step with
the index, $n$
\item control of the fractional amount of the state signal of the cell at
the adjacent lower-indexed step, $n-1$, used to comprise the state
signal of the cell at the present step with the index, $n$
\item control of the fractional amount of the readout candidate signal used
to release as the externally-accessible (observable) signal of the
cell at the present step with the index, $n$
\end{itemize}

\subsection{Data Set Standardization\label{subsec:Data-Set-Standardization}}

Before the operation of the LSTM network (or its parent, RNN) can
commence, the external training data set, $\vec{x}_{0}[n]$, needs
to be standardized, such that all elements of the input to the network,
$\vec{x}[n]$, have the mean of $0$ and the standard deviation of
$1$ over the training set:
\begin{align}
\vec{\mu} & =\frac{1}{N}\sum_{n=0}^{N-1}\vec{x}_{0}[n]\label{eq:VanillaLSTMCellTrainingSetStandardizationMean}\\
\mathsf{\mathcal{V}} & =\frac{1}{N-1}\sum_{n=0}^{N-1}\left(\vec{x}_{0}[n]-\vec{\mu}\right)\left(\vec{x}_{0}[n]-\vec{\mu}\right)^{T}\label{eq:VanillaLSTMCellTrainingSetStandardizationVariance}\\
\vec{x}[n] & =\left[diag\left(\sqrt{\mathcal{V}_{ii}}\right)\right]^{-1}\left(\vec{x}_{0}[n]-\vec{\mu}\right)\label{eq:VanillaLSTMCellTrainingSetStandardizationInput}
\end{align}

Applying the transformations in Equation \ref{eq:VanillaLSTMCellTrainingSetStandardizationMean},
Equation \ref{eq:VanillaLSTMCellTrainingSetStandardizationVariance},
and Equation \ref{eq:VanillaLSTMCellTrainingSetStandardizationInput}
to the external training samples, $\vec{x}_{0}[n]$, accomplishes
this task. In these equations, $N$ is the number of samples in the
training set, $\vec{\mu}$ is the sample mean, and $\mathsf{\mathcal{V}}$
is the sample auto-covariance matrix of the training set\footnote{The test and validation sets should be standardized with the mean
and standard deviation of the training set \citep{journals/corr/Graves13}.}.

\subsection{Warping (Activation) Functions\label{subsec:Warping-(Activation)-Functions}}

As described in Section \ref{subsec:Overview}, the warping function
for control needs to output a value between $0$ and $1$. The sigmoidal
(also known as ``logistic'') nonlinearity is a good choice, because
it is bipolarly-saturating between these values and is monotonic,
continuous, and differentiable:
\begin{align}
G_{c}(z) & \equiv\sigma(z)\equiv\frac{1}{1+e^{-z}}=\frac{1+\tanh(\frac{z}{2})}{2}=\frac{1+G_{d}(\frac{z}{2})}{2}\label{eq:VanillaLSTMCellWarpingFunctionControl}
\end{align}
Related to this function, the hyperbolic tangent is a suitable choice
for the warping function for data bounding purposes:
\begin{align}
G_{d}(z) & \equiv\tanh(z)=\frac{e^{z}-e^{-z}}{e^{z}+e^{-z}}=2\sigma(2z)-1=2G_{c}(2z)-1\label{eq:VanillaLSTMCellWarpingFunctionData}
\end{align}
because it is monotonic, negative-symmetric, and bipolarly-saturating
at $-1$ and $1$ (i.e., one standard deviation of $\vec{x}[n]$ in
each direction). This insures that the data warping function, $G_{d}(z)$,
will support both negative and positive values of the standardized
incoming data signal, $\vec{x}[n]$, in Equation \ref{eq:VanillaLSTMCellTrainingSetStandardizationInput},
and keep it bounded within that range (the ``squashing'' property).

\subsection{Vanilla LSTM Cell Model Parameters\label{subsec:Vanilla-LSTM-Cell-Model-Parameters}}

The Vanilla LSTM cell model uses the following fifteen (15) parameter
entities (with their respective dimensions and\\ designations as
indicated below):

\subsubsection{Parameters of the accumulation node, $\vec{a}_{cu}[n]$, of the gate
that controls the fractional amount of the update candidate signal,
$\vec{u}[n]$, used to comprise the state signal of the cell at the
present step with the index, $n$}
\begin{itemize}
\item $W_{x_{cu}}\in\mathbb{R}^{d_{s}\times d_{x}}$ -- the matrix of weights
connecting the input signal, $\vec{x}[n]$, at the present step with
the index, $n$, to the ``control update'' accumulation node, $\vec{a}_{cu}[n]$,
of the cell at the present step with the index, $n$
\item $W_{s_{cu}}\mathbb{\in R}^{d_{s}\times d_{s}}$ -- the matrix of
weights connecting the state signal, $\vec{s}[n-1]$, at the adjacent
lower-indexed step with the index, $n-1$, to the ``control update''
accumulation node, $\vec{a}_{cu}[n]$, of the cell at the present
step with the index, $n$
\item $W_{v_{cu}}\mathbb{\in R}^{d_{s}\times d_{s}}$ -- the matrix of
weights connecting the externally-accessible (observable) value signal,
$\vec{v}[n-1]$, at the adjacent lower-indexed step with the index,
$n-1$, to the ``control update'' accumulation node, $\vec{a}_{cu}[n]$,
of the cell at the present step with the index, $n$
\item $\vec{b}_{cu}\in\mathbb{R}^{d_{s}}$ -- the vector of bias elements
for the ``control update'' accumulation node, $\vec{a}_{cu}[n]$,
of the cell at the present step with the index, $n$
\end{itemize}

\subsubsection{Parameters of the accumulation node, $\vec{a}_{cs}[n]$, of the gate
that controls the fractional amount of the state signal of the cell,
$\vec{s}[n-1]$, at the adjacent lower-indexed step, $n-1$, used
to comprise the state signal of the cell at the present step with
the index, $n$}
\begin{itemize}
\item $W_{x_{cs}}\in\mathbb{R}^{d_{s}\times d_{x}}$ -- the matrix of weights
connecting the input signal, $\vec{x}[n]$, at the present step with
the index, $n$, to the ``control state'' accumulation node, $\vec{a}_{cs}[n]$,
of the cell at the present step with the index, $n$
\item $W_{s_{cs}}\mathbb{\in R}^{d_{s}\times d_{s}}$ -- the matrix of
weights connecting the state signal, $\vec{s}[n-1]$, at the adjacent
lower-indexed step with the index, $n-1$, to the ``control state''
accumulation node, $\vec{a}_{cs}[n]$, of the cell at the present
step with the index, $n$
\item $W_{v_{cs}}\mathbb{\in R}^{d_{s}\times d_{s}}$ -- the matrix of
weights connecting the externally-accessible (observable) value signal,
$\vec{v}[n-1]$, at the adjacent lower-indexed step with the index,
$n-1$, to the ``control state'' accumulation node, $\vec{a}_{cs}[n]$,
of the cell at the present step with the index, $n$
\item $\vec{b}_{cs}\in\mathbb{R}^{d_{s}}$ -- the vector of bias elements
for the ``control state'' accumulation node, $\vec{a}_{cs}[n]$,
of the cell at the present step with the index, $n$
\end{itemize}

\subsubsection{Parameters of the accumulation node, $\vec{a}_{cr}[n]$, of the gate
that controls the fractional amount of the readout candidate signal,
$\vec{r}[n]$, used to release as the externally-accessible (observable)
value signal of the cell at the present step with the index, $n$}
\begin{itemize}
\item $W_{x_{cr}}\in\mathbb{R}^{d_{s}\times d_{x}}$ -- the matrix of weights
connecting the input signal, $\vec{x}[n]$, at the present step with
the index, $n$, to the ``control readout'' accumulation node, $\vec{a}_{cr}[n]$,
of the cell at the present step with the index, $n$
\item $W_{s_{cr}}\mathbb{\in R}^{d_{s}\times d_{s}}$ -- the matrix of
weights connecting the state signal, $\vec{s}[n]$, at the present
step with the index, $n$, to the ``control readout'' accumulation
node, $\vec{a}_{cr}[n]$, of the cell at the present step with the
index, $n$
\item $W_{v_{cr}}\mathbb{\in R}^{d_{s}\times d_{s}}$ -- the matrix of
weights connecting the externally-accessible (observable) value signal,
$\vec{v}[n-1]$, at the adjacent lower-indexed step with the index,
$n-1$, to the ``control readout'' accumulation node, $\vec{a}_{cr}[n]$,
of the cell at the present step with the index, $n$
\item $\vec{b}_{cr}\in\mathbb{R}^{d_{s}}$ -- the vector of bias elements
for the ``control readout'' accumulation node, $\vec{a}_{cr}[n]$,
of the cell at the present step with the index, $n$
\end{itemize}

\subsubsection{Parameters of the accumulation node, $\vec{a}_{du}[n]$, for the
data warping function that produces the update candidate signal, $\vec{u}[n]$,
of the cell at the present step with the index, $n$}
\begin{itemize}
\item $W_{x_{du}}\in\mathbb{R}^{d_{s}\times d_{x}}$ -- the matrix of weights
connecting the input signal, $\vec{x}[n]$, at the present step with
the index, $n$, to the ``data update'' accumulation node, $\vec{a}_{du}[n]$,
of the cell at the present step with the index, $n$
\item $W_{v_{du}}\mathbb{\in R}^{d_{s}\times d_{s}}$ -- the matrix of
weights connecting the externally-accessible (observable) value signal,
$\vec{v}[n-1]$, at the adjacent lower-indexed step with the index,
$n-1$, to the ``data update'' accumulation node, $\vec{a}_{du}[n]$,
of the cell at the present step with the index, $n$
\item $\vec{b}_{du}\in\mathbb{R}^{d_{s}}$ -- the vector of bias elements
for the ``data update'' accumulation node, $\vec{a}_{du}[n]$, of
the cell at the present step with the index, $n$
\end{itemize}

\subsubsection{All model parameters, which must be learned, combined (for notational
convenience)}
\begin{itemize}
\item All parameters of the LSTM network are commonly concatenated and represented
as a whole by $\Theta$:
\begin{align}
\Theta & \equiv\left\{ W_{x_{cu}},W_{s_{cu}},W_{v_{cu}},\vec{b}_{cu},W_{x_{cs}},W_{s_{cs}},W_{v_{cs}},\vec{b}_{cs},W_{x_{cr}},W_{s_{cr}},W_{v_{cr}},\vec{b}_{cr}W_{x_{du}},W_{v_{du}},\vec{b}_{du}\right\} \label{eq:VanillaLSTMCellModelParametersConcatenated}
\end{align}
\item Arranged ``thematically'' (attributed by the type of an accumulation),
$\Theta$ can be written as:
\begin{align}
\begin{gathered}\Theta\equiv\end{gathered}
 & \begin{Bmatrix}W_{x_{cu}}, & W_{s_{cu}}, & W_{v_{cu}}, & \vec{b}_{cu},\\
W_{x_{cs}}, & W_{s_{cs}}, & W_{v_{cs}}, & \vec{b}_{cs},\\
W_{x_{cr}}, & W_{s_{cr}}, & W_{v_{cr}}, & \vec{b}_{cr},\\
W_{x_{du}}, &  & W_{v_{du}}, & \vec{b}_{du}
\end{Bmatrix}.\label{eq:VanillaLSTMCellModelParametersThematicallyArranged}
\end{align}
\end{itemize}

\subsection{Summary of the main entities (generalized)\label{subsec:Summary-of-the-main-entities-(generalized)}}

The following glossary lists the main entities of the model in a generalized
way (i.e., without the subscripts, indices, etc.). Note that the special
quantities $\vec{\psi,}\vec{\chi},\vec{\alpha},\vec{\rho},\vec{\gamma}$
will be defined in Section \ref{subsec:Vanilla-LSTM-System-Derivatives-(Backward-Pass)}.

\begin{equation}\label{eq:VanillaLSTMCellNotationalEntities}\begin{rcases} \vec{x}\in\mathbb{R}^{d_{x}}\\\vec{s},\vec{v},\vec{a},\vec{u},\vec{r},\vec{g}\in\mathbb{R}^{d_{s}}\\\vec{\psi},\vec{\chi},\vec{\alpha},\vec{\rho},\vec{\gamma}\in\mathbb{R}^{d_{s}}\\W_{x}\in\mathbb{R}^{d_{s}\times d_{x}}\\W_{s},W_{v}\in\mathbb{R}^{d_{s}\times d_{s}}\\\vec{b}\in\mathbb{R}^{d_{s}}\\E\\N\\K\\n=0,\ldots,K-1 \end{rcases}\end{equation}

\pagebreak{}

\subsection{Vanilla LSTM System Equations (\textquotedblleft Forward Pass\textquotedblright )\label{subsec:Vanilla-LSTM-System-Equations-(Forward-Pass)}}

It is important to highlight the general pattern of computations that
govern the processes, according to which any RNN cell, and the LSTM
network cell in particular, unrolled for $K$ steps, generates sequences
of samples. Namely, the quantities that characterize the step of the
cell at the index, $n$, of the sequence depend on the quantities
that characterize the step of the cell at the index, $n-1$, of the
sequence\footnote{If the cell is unrolled in the opposite direction, then $n-1$ is
replaced by $n+1$, and the direction of evaluating the steps is reversed.
For the bi-directional unrolling, the steps in both the positive and
the negative directions of the index, $n$, need to be evaluated \citep{schuster97bidirectional}.
Here, only the positive direction is considered.}. The following equations fully define the Vanilla LSTM cell:
\begin{align}
\vec{a}_{cu}[n] & =W_{x_{cu}}\vec{x}[n]+W_{s_{cu}}\vec{s}[n-1]+W_{v_{cu}}\vec{v}[n-1]+\vec{b}_{cu}\label{eq:VanillaLSTMCellAccumulationControlUpdate}\\
\vec{a}_{cs}[n] & =W_{x_{cs}}\vec{x}[n]+W_{s_{cs}}\vec{s}[n-1]+W_{v_{cs}}\vec{v}[n-1]+\vec{b}_{cs}\label{eq:VanillaLSTMCellAccumulationControlState}\\
\vec{a}_{cr}[n] & =W_{x_{cr}}\vec{x}[n]+W_{s_{cr}}\vec{s}[n]+W_{v_{cr}}\vec{v}[n-1]+\vec{b}_{cr}\label{eq:VanillaLSTMCellAccumulationControlReadout}\\
\vec{a}_{du}[n] & =W_{x_{du}}\vec{x}[n]+W_{v_{du}}\vec{v}[n-1]+\vec{b}_{du}\label{eq:VanillaLSTMCellAccumulationDataUpdate}\\
\vec{u}[n] & =G_{d}(\vec{a}_{du}[n])\label{eq:VanillaLSTMCellUpdate}\\
\vec{g}_{cu}[n] & =G_{c}(\vec{a}_{cu}[n])\label{eq:VanillaLSTMCellGateControlUpdate}\\
\vec{g}_{cs}[n] & =G_{c}(\vec{a}_{cs}[n])\label{eq:VanillaLSTMCellGateControlState}\\
\vec{g}_{cr}[n] & =G_{c}(\vec{a}_{cr}[n])\label{eq:VanillaLSTMCellGateControlReadout}\\
\vec{s}[n] & =\vec{g}_{cs}[n]\odot\vec{s}[n-1]+\vec{g}_{cu}[n]\odot\vec{u}[n]\label{eq:VanillaLSTMCellState}\\
\vec{r}[n] & =G_{d}(\vec{s}[n])\label{eq:VanillaLSTMCellReadout}\\
\vec{v}[n] & =\vec{g}_{cr}[n]\odot\vec{r}[n]\label{eq:VanillaLSTMCellValue}
\end{align}

The schematic diagram of the Vanilla LSTM cell, defined by Equations
\ref{eq:VanillaLSTMCellAccumulationControlUpdate} -- \ref{eq:VanillaLSTMCellValue},
is presented in Figure \ref{fig:VanillaLSTMCellSchematics}, and the
snapshot of unrolling it (for only $4$ steps as an illustration)
appears in Figure \ref{fig:LSTMCellUnrolled}. In order to make it
easier to isolate the specific functions performed by the components
of the Vanilla LSTM cell, its schematic diagram is redrawn in Figure
\ref{fig:VanillaLSTMCellSchematicsAnnotatedStages}, with the major
stages comprising the cell\textquoteright s architecture marked by
dashed rectangles annotated by the names of the respective enclosed
stages.

\begin{figure}[tph]
\includegraphics[scale=0.47]{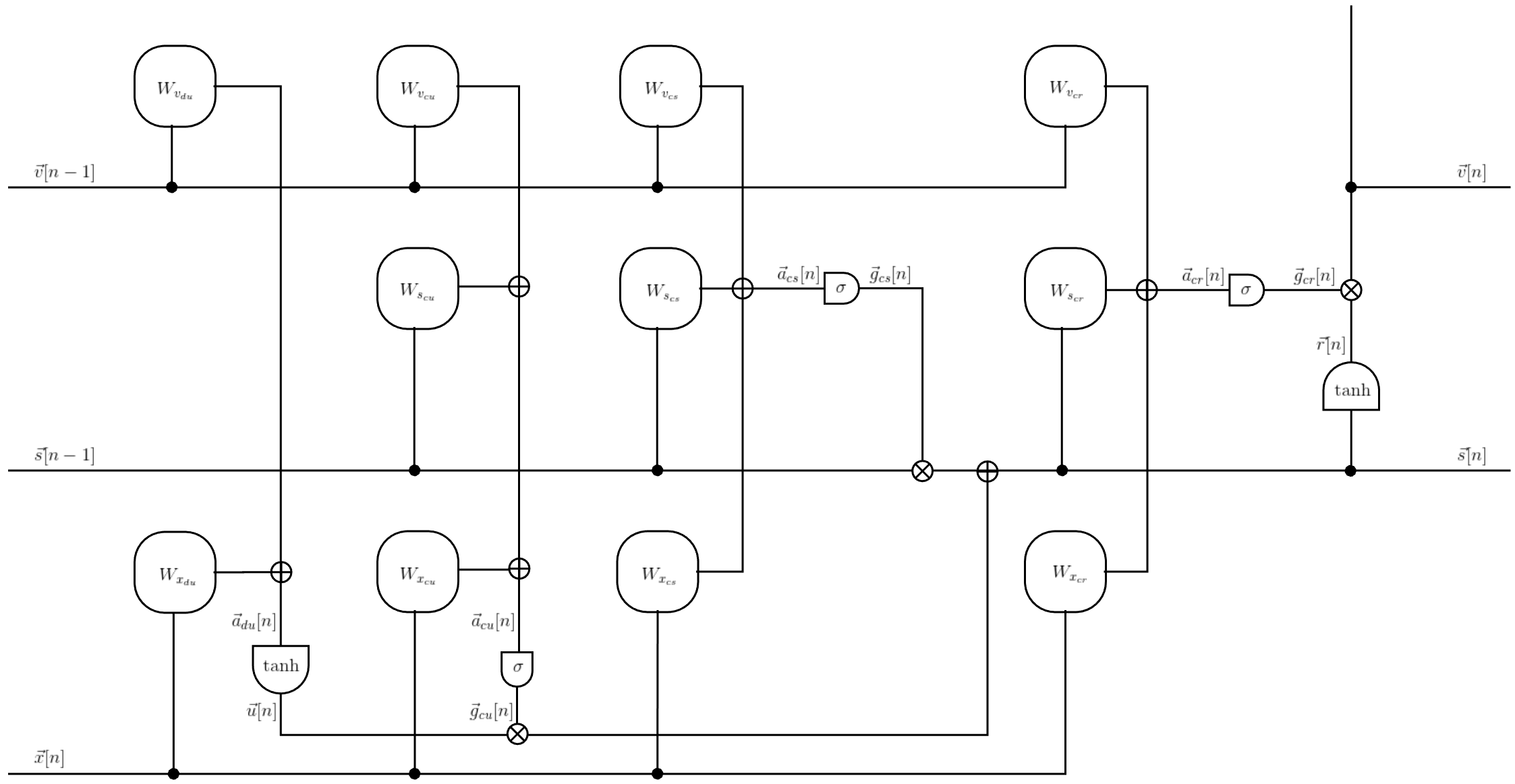}

\caption{Vanilla LSTM network cell. The bias parameters, $\vec{b},$ have been
omitted from the figure for brevity. They can be assumed to be included
without the loss of generality by appending an additional element,
always set to $1$, to the input signal vector, $\vec{x}[n]$, and
increasing the row dimensions of all corresponding weight matrices
by $1$.\label{fig:VanillaLSTMCellSchematics}}
\end{figure}

\begin{figure}[tph]
\includegraphics[scale=0.44]{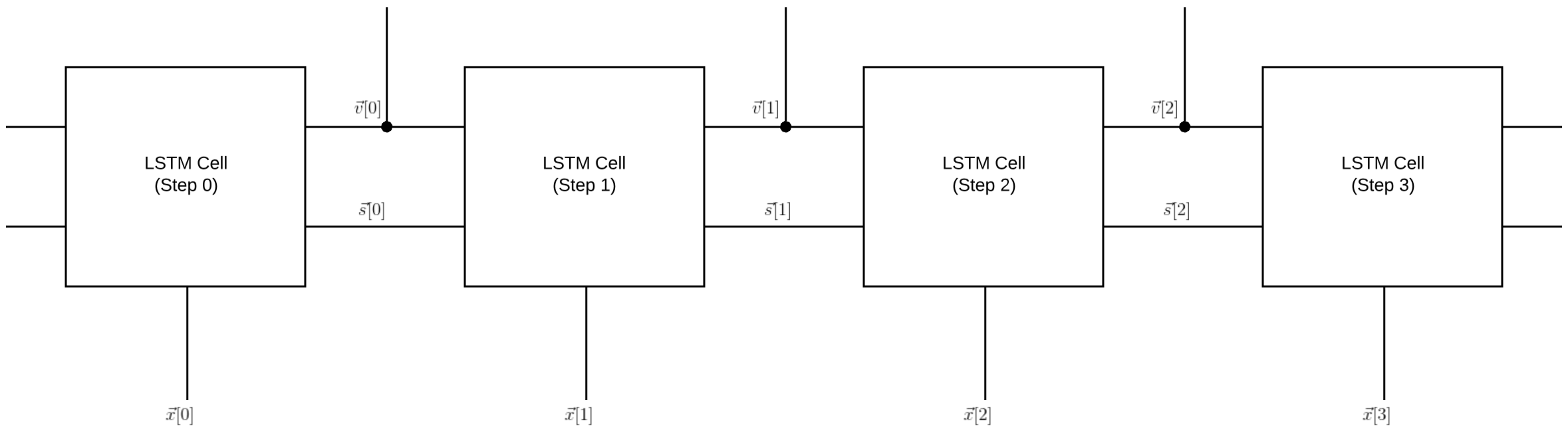}\caption{Sequence of steps generated by unrolling a cell of the LSTM network
(displaying $4$ steps for illustration).\label{fig:LSTMCellUnrolled}}
\end{figure}

\begin{figure}[tph]
\includegraphics[scale=0.42]{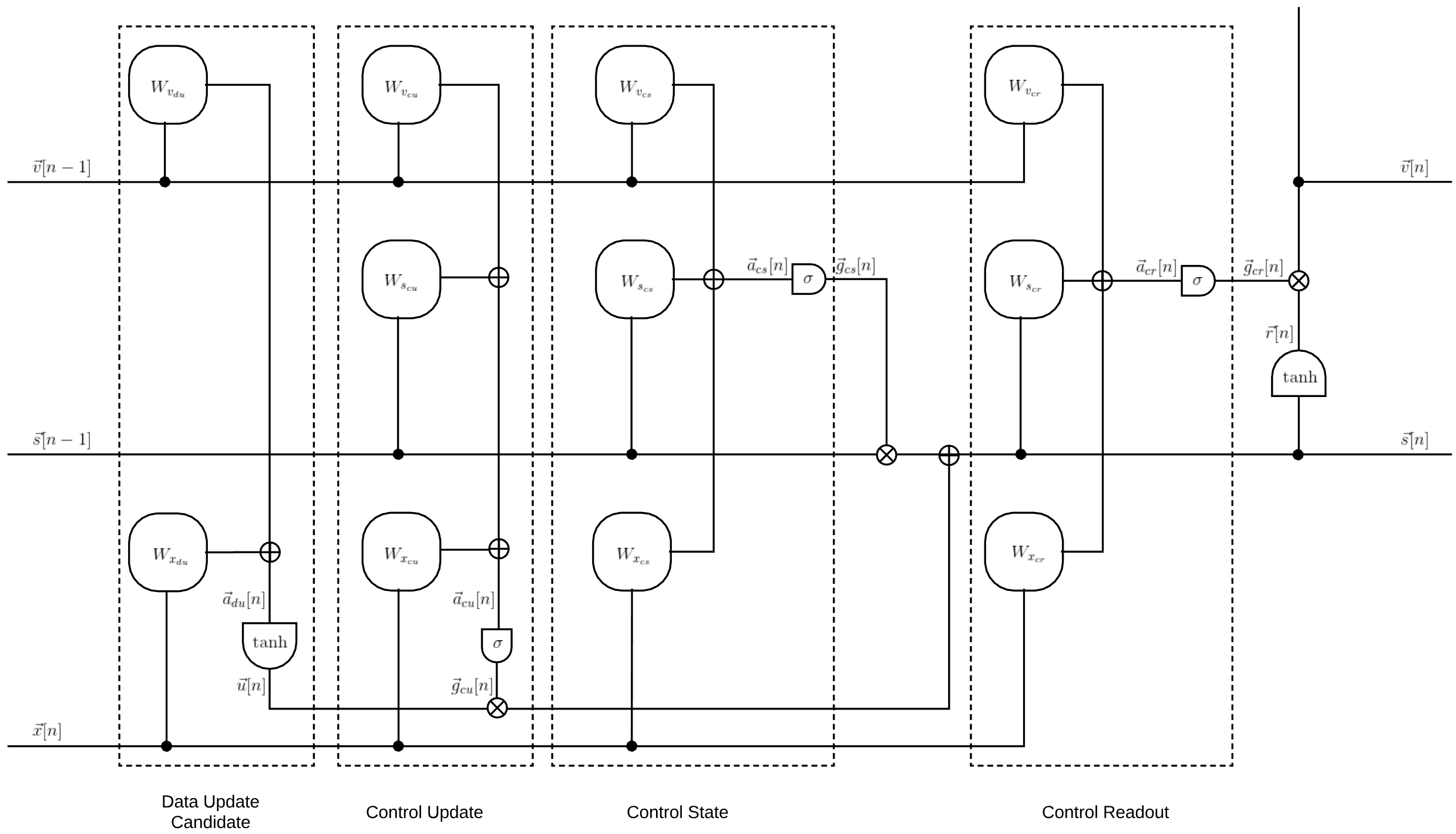}

\caption{Vanilla LSTM network cell from Figure \ref{fig:VanillaLSTMCellSchematics},
with the stages of the system delineated by dashed rectangles and
annotations that depict the function of each stage. As before, the
bias parameters, $\vec{b},$ have been omitted from the figure for
brevity. (They can be assumed to be included without the loss of generality
by appending an additional element, always set to $1$, to the input
signal vector, $\vec{x}[n]$, and increasing the row dimensions of
all corresponding weight matrices by $1$.)\label{fig:VanillaLSTMCellSchematicsAnnotatedStages}}
\end{figure}

\vphantom{}

\hphantom{}\pagebreak{}

\subsection{Vanilla LSTM System Derivatives (\textquotedblleft Backward Pass\textquotedblright )\label{subsec:Vanilla-LSTM-System-Derivatives-(Backward-Pass)}}

This section derives the equations that are necessary for training
the Vanilla LSTM network cell, unrolled for $K$ steps, using Back
Propagation Through Time (BPTT). To obtain the update equations for
the parameters of the system, two auxiliary ``backward-moving''
gradient sequences, indexed by $n$, are computed first: $\vec{\chi}[n]$,
the total partial derivative of the objective function, $E$, with
respect to the externally-accessible (observable) value signal, $\vec{v}[n]$,
and $\vec{\psi}[n]$, the total partial derivative of the objective
function, $E$, with respect to the state signal, $\vec{s}[n]$. The
decision to ``anchor'' the chain rule at the border of the cell
is made judiciously, guided by the principles of modular design. Expressing
every intra-cell total partial derivative in terms of $\vec{\chi}[n]$
(instead of explicitly computing the total partial derivative of the
objective function, $E$, with respect to each variable of the cell
\citep{journals/corr/PalangiDSGHCSW15}) reduces the number of intermediate
variables. This makes the equations for the backward pass straightforward
and a natural fit for an implementation as a pluggable module \citep{Werbos:88gasmarket,werbos:bptt1990,phd/de/Graves2008}.

Due to the backward-moving recursion of $\vec{\chi}[n]$ and $\vec{\psi}[n]$
(the gradient sequences propagate in the direction opposite to that
of the state signal, $\vec{s}[n]$, as a function of the step index,
$n$), the values of $\vec{\chi}[n]$ and $\vec{\psi}[n]$ at the
index, $n$, depend on the values of the same quantities at the index,
$n+1$, subject to the initial conditions. Once $\vec{\chi}[n]$ and
$\vec{\psi}[n]$ are known, they are used to compute the total partial
derivatives of the objective function, $E$, with respect to the accumulation
nodes for each value of the index, $n$. These intermediate gradient
sequences, named $\vec{\alpha}_{cs}[n]$, $\vec{\alpha}_{cu}[n]$,
$\vec{\alpha}_{cr}[n]$, and $\vec{\alpha}_{du}[n]$, allocate the
amounts contributed by the signals associated with the step at the
index, $n$, to the total partial derivatives of the objective function,
$E$, with respect to the model parameters. By the definition of the
total derivative, these contributions have to be summed across all
steps, $0\leq n\leq K-1$, to produce the total partial derivatives
of the objective function, $E$, with respect to the model parameters.

During the inference phase of the LSTM system, only $\vec{x}[n]$
(the input signal) and $\vec{v}[n]$ (the value signal) are externally
accessible (i.e., observable). The cell accepts the input signal at
each step and computes the value signal for all steps. All the other
intermediate signals are available only to the internal components
and nodes of the cell, with the exception of $\vec{s}[n]$ (state
signal) and $\vec{v}[n]$ (value signal), which serve both the inter-
and the intra-step purposes throughout the unrolled sequence.

The cell's value signal, $\vec{v}[n]$, at the step with the index,
$n$, can be further transformed to produce the output signal, $\vec{y}[n]$
(e.g., a commonly used form of $\vec{y}[n]$ may be obtained by computing
a linear transformation of $\vec{v}[n]$, followed by a softmax operator,
or a different decision function). Likewise, the input signal, too,
may result from the transformation of the original raw data. For example,
one kind of input pre-processing can convert the vocabulary ``one-hot''
vector into a more compact representation. Also, for applications
where the input data set can be collected for the entire segment at
once, input samples that lie within a small window surrounding the
given step can be combined so as to enhance the system's ``attention''
to context. A non-causal input filter, designed for this purpose,
will be introduced in Section \ref{subsec:Exernal-Input-Context-Windows}
as part of extending the Vanilla LSTM cell.

We start by computing the derivatives of the warping functions from
their definitions in Equation \ref{eq:VanillaLSTMCellWarpingFunctionControl}
and Equation \ref{eq:VanillaLSTMCellWarpingFunctionData}, respectively:
\begin{align}
\frac{dG_{c}(z)}{dz} & =G_{c}(z)(1-G_{c}(z))\label{eq:VanillaLSTMCellDerivativeOfWarpingFunctionControl}\\
\frac{dG_{d}(z)}{dz} & =1-(G_{d}(z))^{2}\label{eq:VanillaLSTMCellDerivativeOfWarpingFunctionData}
\end{align}

Next, we anchor the chain rule at the border of the cell by defining
$\vec{\chi}[n]$ as the total partial derivative of the objective
function, $E$, with respect to the externally-accessible (observable)
value signal, $\vec{v}[n]$, as follows:
\begin{align}
\vec{\chi}[n] & \equiv\vec{\nabla}_{\vec{v}[n]}E=\frac{\partial E}{\partial\vec{v}[n]}\label{eq:VanillaLSTMCellDelEDelValueStepTotalDefinition}
\end{align}
As will become imminently evident, having $\vec{\chi}[n]$ not only
makes training equations for the Vanilla LSTM cell amenable for a
modular implementation at the step level, but also greatly simplifies
them.\pagebreak{}

We also define the total partial derivatives of the objective function,
$E$, with respect to three intermediate (i.e., away from the border)
variables and another border variable of the Vanilla LSTM cell:
\begin{align}
\vec{\rho}[n] & \equiv\vec{\nabla}_{\vec{r}[n]}E=\frac{\partial E}{\partial\vec{r}[n]}\label{eq:VanillaLSTMCellDelEDelReadoutStepTotalDefinition}\\
\vec{\gamma}[n] & \equiv\vec{\nabla}_{\vec{g}[n]}E=\frac{\partial E}{\partial\vec{g}[n]}\label{eq:VanillaLSTMCellDelEDelGateControlGenericStepTotalDefinition}\\
\vec{\alpha}[n] & \equiv\vec{\nabla}_{\vec{a}[n]}E=\frac{\partial E}{\partial\vec{a}[n]}\label{eq:VanillaLSTMCellDelEDelAccumulationGenericStepTotalDefinition}\\
\vec{\psi}[n] & \equiv\vec{\nabla}_{\vec{s}[n]}E=\frac{\partial E}{\partial\vec{s}[n]}\label{eq:VanillaLSTMCellDelEDelStateStepTotalDefinition}
\end{align}

The border quantity in Equation \ref{eq:VanillaLSTMCellDelEDelStateStepTotalDefinition},
$\vec{\psi}[n]$, is of special significance as it is the total partial
derivative of the objective function, $E$, with respect to the state
signal, $\vec{s}[n]$, at the index, $n$, of the Vanilla LSTM cell.
As in the standard RNN, all parameter updates in the Vanilla LSTM
network depend on $\vec{\psi}[n]$, making it the most important error
gradient sequence of the system.

The backward pass equations are obtained by utilizing these border
and intermediate derivatives in the application of the chain rule
to the Vanilla LSTM cell, defined by Equations \ref{eq:VanillaLSTMCellAccumulationControlUpdate}
-- \ref{eq:VanillaLSTMCellValue}: 
\begin{align}
\vec{\chi}[n] & =\left(\frac{\partial\vec{y}[n]}{\partial\vec{v}[n]}\right)^{T}\left(\frac{\partial E}{\partial\vec{y}[n]}\right)+\vec{f_{\chi}}[n+1]\label{eq:VanillaLSTMCellDelEDelValueStepTotal}\\
\vec{\rho}[n] & =\left(\frac{\partial\vec{v}[n]}{\partial\vec{r}[n]}\right)^{T}\left(\frac{\partial E}{\partial\vec{v}[n]}\right)=\vec{(\nabla}_{\vec{v}[n]}E)\odot\vec{g}_{cr}[n]=\vec{\chi}[n]\odot\vec{g}_{cr}[n]\label{eq:VanillaLSTMCellDelEDelReadoutStepTotal}\\
\vec{\gamma}_{cr}[n] & =\frac{\partial E}{\partial\vec{v}[n]}\frac{\partial\vec{v}[n]}{\partial\vec{g}_{cr}[n]}=\vec{(\nabla}_{\vec{v}[n]}E)\odot\vec{r}[n]=\vec{\chi}[n]\odot\vec{r}[n]\label{eq:VanillaLSTMCellDelEDelGateControlReadoutStepTotal}\\
\vec{\alpha}_{cr}[n] & =\vec{\gamma}_{cr}[n]\odot\frac{\partial\vec{g}_{cr}[n]}{\partial\vec{a}_{cr}[n]}=\vec{\gamma}_{cr}[n]\odot\frac{dG_{c}(z)}{dz}\rfloor_{z=\vec{a}_{cr}[n]}=\vec{\chi}[n]\odot\vec{r}[n]\odot\frac{dG_{c}(z)}{dz}\rfloor_{z=\vec{a}_{cr}[n]}\label{eq:VanillaLSTMCellDelEDelAccumulationControlReadoutStepTotal}\\
\vec{\psi}[n] & =\vec{\rho}[n]\odot\frac{\partial\vec{r}[n]}{\partial\vec{s}[n]}+\frac{\partial\vec{a}_{cr}[n]}{\partial\vec{s}[n]}\vec{\alpha}_{cr}[n]+\vec{f_{\psi}}[n+1]\\
 & =\vec{\rho}[n]\odot\frac{dG_{d}(\vec{z})}{d\vec{z}}\rfloor_{z=\vec{s}[n]}+W_{s_{cr}}\vec{\alpha}_{cr}[n]+\vec{f_{\psi}}[n+1]\\
 & =\vec{\chi}[n]\odot\vec{g}_{cr}[n]\odot\frac{dG_{d}(\vec{z})}{d\vec{z}}\rfloor_{z=\vec{s}[n]}+W_{s_{cr}}\vec{\alpha}_{cr}[n]+\vec{f_{\psi}}[n+1]\label{eq:VanillaLSTMCellDelEDelStateStepTotal}\\
\vec{\alpha}_{cs}[n] & =\vec{\psi}[n]\odot\frac{\partial\vec{s}[n]}{\partial\vec{g}_{cs}[n]}\odot\frac{\partial\vec{g}_{cs}[n]}{\partial\vec{a}_{cs}[n]}=\vec{\psi}[n]\odot\vec{s}[n-1]\odot\frac{dG_{c}(\vec{z})}{d\vec{z}}\rfloor_{z=\vec{a}_{cs}[n]}\label{eq:VanillaLSTMCellDelEDelAccumulationControlStateStepTotal}\\
\vec{\alpha}_{cu}[n] & =\vec{\psi}[n]\odot\frac{\partial\vec{s}[n]}{\partial\vec{g}_{cu}[n]}\odot\frac{\partial\vec{g}_{cu}[n]}{\partial\vec{a}_{cu}[n]}=\vec{\psi}[n]\odot\vec{u}[n]\odot\frac{dG_{c}(\vec{z})}{d\vec{z}}\rfloor_{z=\vec{a}_{cu}[n]}\label{eq:VanillaLSTMCellDelEDelAccumulationControlUpdateStepTotal}\\
\vec{\alpha}_{du}[n] & =\vec{\psi}[n]\odot\frac{\partial\vec{s}[n]}{\partial\vec{u}[n]}\odot\frac{dG_{d}(\vec{z})}{d\vec{z}}\rfloor_{z=\vec{a}_{du}[n]}=\vec{\psi}[n]\odot\vec{g}_{cu}[n]\odot\frac{dG_{d}(\vec{z})}{d\vec{z}}\rfloor_{z=\vec{a}_{du}[n]}\label{eq:VanillaLSTMCellDelEDelAccumulationDataUpdateStepTotal}
\end{align}
where:
\begin{align}
\vec{f_{\chi}}[n+1] & =W_{v_{cu}}^{T}\vec{\alpha}_{cu}[n+1]+W_{v_{cs}}^{T}\vec{\alpha}_{cs}[n+1]+W_{v_{cr}}^{T}\vec{\alpha}_{cr}[n+1]+W_{v_{du}}^{T}\vec{\alpha}_{du}[n+1]\label{eq:VanillaLSTMCellDelEDelValueNextStep}\\
\vec{f_{\psi}}[n+1] & =W_{s_{cu}}^{T}\vec{\alpha}_{cu}[n+1]+W_{s_{cs}}^{T}\vec{\alpha}_{cs}[n+1]+\vec{g}_{cs}[n+1]\odot\vec{\psi}[n+1]\label{eq:VanillaLSTMCellDelEDelStateNextStep}
\end{align}
are the portions of the total derivative of the objective function,
$E$, with respect to the cell's value signal and the cell's state
signal, respectively, contributed by the quantities evaluated at the
step with the index, $n+1$.

The total partial derivatives of the objective function, $E$, with
respect to the model parameters at the step with the index, $n$,
are directly proportional to the ``accumulation derivatives''\footnote{As noted in Section \ref{sec:From-RNN-to-Vanilla-LSTM-Network}, during
training, the total derivatives of the objective function, $E$, with
respect to the cell's accumulation signals (``accumulation derivatives'')
can become excessively large. In order to prevent these kinds of numerical
problems, all accumulation derivatives are clipped to lie between
$-1$ and $1$, a range that is suitable for the particular choices
of the control and data warping functions \citep{phd/de/Graves2008,journals/corr/Graves13,pascanu2013difficulty}.}, given by Equation \ref{eq:VanillaLSTMCellDelEDelAccumulationControlReadoutStepTotal},
Equation \ref{eq:VanillaLSTMCellDelEDelAccumulationControlStateStepTotal},
Equation \ref{eq:VanillaLSTMCellDelEDelAccumulationControlUpdateStepTotal},
and \\\pagebreak\begin{samepage}

Equation \ref{eq:VanillaLSTMCellDelEDelAccumulationDataUpdateStepTotal}.
Hence, by referring once again to the definition of the Vanilla LSTM
cell in Equations \ref{eq:VanillaLSTMCellAccumulationControlUpdate}
-- \ref{eq:VanillaLSTMCellValue}, we obtain:

\vspace{-5mm}

\begin{align}
\frac{\partial E}{\partial W_{x_{cu}}}[n] & =\vec{\alpha}_{cu}[n]\vec{x}^{T}[n]\label{eq:VanillaLSTMCellDelEDelInputMatrixControlUpdateStepTotal}\\
\frac{\partial E}{\partial W_{s_{cu}}}[n] & =\vec{\alpha}_{cu}[n]\vec{s}^{T}[n-1]\label{eq:VanillaLSTMCellDelEDelStateMatrixControlUpdateStepTotal}\\
\frac{\partial E}{\partial W_{v_{cu}}}[n] & =\vec{\alpha}_{cu}[n]\vec{v}^{T}[n-1]\label{eq:VanillaLSTMCellDelEDelValueMatrixControlUpdateStepTotal}\\
\frac{\partial E}{\partial\vec{b}_{cu}}[n] & =\vec{\alpha}_{cu}[n]\label{eq:VanillaLSTMCellDelEDelBiasVectorControlUpdateStepTotal}\\
\frac{\partial E}{\partial W_{x_{cs}}}[n] & =\vec{\alpha}_{cs}[n]\vec{x}^{T}[n]\label{eq:VanillaLSTMCellDelEDelInputMatrixControlStateStepTotal}\\
\frac{\partial E}{\partial W_{s_{cs}}}[n] & =\vec{\alpha}_{cs}[n]\vec{s}^{T}[n-1]\label{eq:VanillaLSTMCellDelEDelStateMatrixControlStateStepTotal}\\
\frac{\partial E}{\partial W_{v_{cs}}}[n] & =\vec{\alpha}_{cs}[n]\vec{v}^{T}[n-1]\label{eq:VanillaLSTMCellDelEDelValueMatrixControlStateStepTotal}\\
\frac{\partial E}{\partial\vec{b}_{cs}}[n] & =\vec{\alpha}_{cs}[n]\label{eq:VanillaLSTMCellDelEDelBiasVectorControlStateStepTotal}\\
\frac{\partial E}{\partial W_{x_{cr}}}[n] & =\vec{\alpha}_{cr}[n]\vec{x}^{T}[n]\label{eq:VanillaLSTMCellDelEDelInputMatrixControlReadoutStepTotal}\\
\frac{\partial E}{\partial W_{s_{cr}}}[n] & =\vec{\alpha}_{cr}[n]\vec{s}^{T}[n]\label{eq:VanillaLSTMCellDelEDelStateMatrixControlReadoutStepTotal}\\
\frac{\partial E}{\partial W_{v_{cr}}}[n] & =\vec{\alpha}_{cr}[n]\vec{v}^{T}[n-1]\label{eq:VanillaLSTMCellDelEDelValueMatrixControlReadoutStepTotal}\\
\frac{\partial E}{\partial\vec{b}_{cr}}[n] & =\vec{\alpha}_{cr}[n]\label{eq:VanillaLSTMCellDelEDelBiasVectorControlReadoutStepTotal}\\
\frac{\partial E}{\partial W_{x_{du}}}[n] & =\vec{\alpha}_{du}[n]\vec{x}^{T}[n]\label{eq:VanillaLSTMCellDelEDelInputMatrixDataUpdateStepTotal}\\
\frac{\partial E}{\partial W_{v_{du}}}[n] & =\vec{\alpha}_{du}[n]\vec{v}^{T}[n-1]\label{eq:VanillaLSTMCellDelEDelValueMatrixDataUpdateStepTotal}\\
\frac{\partial E}{\partial\vec{b}_{du}}[n] & =\vec{\alpha}_{du}[n]\label{eq:VanillaLSTMCellDelEDelBiasVectorDataUpdateStepTotal}
\end{align}
Arranged congruently with Equation \ref{eq:VanillaLSTMCellModelParametersThematicallyArranged},
the total partial derivative of the objective function, $E$, with
respect to the model parameters, $\Theta$, at the step with the index,
$n$, is:
\begin{align}
\begin{gathered}\frac{\partial E}{\partial\Theta}[n]=\end{gathered}
 & \begin{Bmatrix}\frac{\partial E}{\partial W_{x_{cu}}}[n], & \frac{\partial E}{\partial W_{s_{cu}}}[n], & \frac{\partial E}{\partial W_{v_{cu}}}[n], & \frac{\partial E}{\partial\vec{b}_{cu}}[n],\\
\frac{\partial E}{\partial W_{x_{cs}}}[n], & \frac{\partial E}{\partial W_{s_{cs}}}[n], & \frac{\partial E}{\partial W_{v_{cs}}}[n], & \frac{\partial E}{\partial\vec{b}_{cs}}[n],\\
\frac{\partial E}{\partial W_{x_{cr}}}[n], & \frac{\partial E}{\partial W_{s_{cr}}}[n], & \frac{\partial E}{\partial W_{v_{cr}}}[n], & \frac{\partial E}{\partial\vec{b}_{cr}}[n],\\
\frac{\partial E}{\partial W_{x_{du}}}[n], &  & \frac{\partial E}{\partial W_{v_{du}}}[n], & \frac{\partial E}{\partial\vec{b}_{du}}[n]
\end{Bmatrix}.\label{eq:VanillaLSTMCellDelEDelModelParametersStepTotalThematicallyArranged}
\end{align}
When the Vanilla LSTM cell is unrolled for $K$ steps in order to
cover one full segment of training samples, the same set of the model
parameters, $\Theta$, is shared by all the steps. This is because
$\Theta$ is the parameter of the Vanilla LSTM cell as a whole. Consequently,
the total derivative of the objective function, $E$, with respect
to the model parameters, $\Theta$, has to include the contributions
from all steps of the unrolled sequence:
\begin{align}
\frac{dE}{d\Theta} & =\sum_{n=0}^{K-1}\frac{\partial E}{\partial\Theta}[n]\label{eq:VanillaLSTMCellDEDModelParametersSegmentTotal}
\end{align}
The result from Equation \ref{eq:VanillaLSTMCellDEDModelParametersSegmentTotal}
can now be used as part of optimization by Gradient Descent. In practice,
Equation \ref{eq:VanillaLSTMCellDEDModelParametersSegmentTotal} is
computed for a batch of segments\footnote{Depending on the application, the batch sizes typically range between
$16$ and $128$ segments.}, and the sum of the parameter gradients over all segments in the
batch is then supplied to the Gradient Descent algorithm for updating
the model parameters\footnote{Regularization, while outside of the scope of the present article,
is an essential aspect of machine learning model training process
\citep{zaremba2014recurrent}.}.

\end{samepage}

\mbox{}

\begin{samepage}

\subsection{Error Gradient Sequences in Vanilla LSTM System\label{subsec:Error-Gradient-Sequences-In-Vanilla-LSTM-System}}

Section \ref{sec:From-RNN-to-Vanilla-LSTM-Network} mentions that
because of the action of the gates, the LSTM network is more compatible
with the Gradient Descent training procedure than the standard RNN
system is. As discussed in Section \ref{sec:RNN-Training-Difficulties}
and Section \ref{subsec:Vanilla-LSTM-System-Derivatives-(Backward-Pass)},
for Gradient Descent to be effective, the elements of $\frac{\partial E}{\partial\Theta}[n]$
in Equation \ref{eq:VanillaLSTMCellDelEDelModelParametersStepTotalThematicallyArranged}
must be well-behaved numerically. In particular, this implies that
the intermediate gradient sequences, $\vec{\alpha}_{cs}[n]$, $\vec{\alpha}_{cu}[n]$,
$\vec{\alpha}_{cr}[n]$, and $\vec{\alpha}_{du}[n]$, and hence the
border gradient sequences, $\vec{\chi}[n]$ and $\vec{\psi}[n]$,
must be able to sustain a steady flow of information over long ranges
of the step index, $n$. Expanding Equation \ref{eq:VanillaLSTMCellDelEDelStateStepTotal}
produces:
\begin{align}
\vec{\psi}[n] & =\vec{\chi}[n]\odot\vec{g}_{cr}[n]\odot\frac{dG_{d}(\vec{z})}{d\vec{z}}\rfloor_{z=\vec{s}[n]}+W_{s_{cr}}\vec{\chi}[n]\odot\vec{r}[n]\odot\frac{dG_{c}(z)}{dz}\rfloor_{z=\vec{a}_{cr}[n]}+\vec{f_{\psi}}[n+1]\label{eq:VanillaLSTMCellDelEDelStateStepTotalExpandedDelEDelAccumulationControlReadoutStepTotal}\\
 & =\left(\left(\frac{\partial\vec{y}[n]}{\partial\vec{v}[n]}\right)^{T}\left(\frac{\partial E}{\partial\vec{y}[n]}\right)+\vec{f_{\chi}}[n+1]\right)\odot\vec{g}_{cr}[n]\odot\frac{dG_{d}(\vec{z})}{d\vec{z}}\rfloor_{z=\vec{s}[n]}\nonumber \\
 & +W_{s_{cr}}\left(\left(\frac{\partial\vec{y}[n]}{\partial\vec{v}[n]}\right)^{T}\left(\frac{\partial E}{\partial\vec{y}[n]}\right)+\vec{f_{\chi}}[n+1]\right)\odot\vec{r}[n]\odot\frac{dG_{c}(z)}{dz}\rfloor_{z=\vec{a}_{cr}[n]}+\vec{f_{\psi}}[n+1]\label{eq:VanillaLSTMCellDelEDelStateStepTotalExpandedDelEDelValueStepTotal}
\end{align}

\vspace{-3mm}
According to Equation \ref{eq:VanillaLSTMCellDelEDelValueNextStep}
and Equation \ref{eq:VanillaLSTMCellDelEDelStateNextStep}, both $\vec{\chi}[n]$
and $\vec{\psi}[n]$ depend on $\vec{\psi}[n+1]$. Hence, we can follow
the approach in Section \ref{sec:RNN-Training-Difficulties} to analyze
the dependence of $\vec{\psi}[n]$ on $\left\langle \vec{\psi}[k]\right\rangle \left\rfloor _{n<k\leq K-1}\right.$
in order to gauge the sensitivity of the LSTM system to factors conducive
to gradient decay. Applying the change of indices, $n\longrightarrow k-1$,
and the  chain rule to Equation \ref{eq:VanillaLSTMCellDelEDelStateStepTotalExpandedDelEDelValueStepTotal}
yields\footnote{The notation, $\mathit{\mathrm{diag}}\bigg[\vec{z}\bigg]$, represents
a diagonal matrix, in which the elements of the vector, $\vec{z}$,
occupy the main diagonal.}:
\begin{align}
\frac{\partial\vec{\psi}[k-1]}{\partial\vec{\psi}[k]} & =\left(\frac{\partial\vec{\psi}[k-1]}{\partial\vec{f_{\chi}}[k]}\right)\left(\frac{\partial\vec{f_{\chi}}[k]}{\partial\vec{\psi}[k]}\right)+\left(\frac{\partial\vec{\psi}[k-1]}{\partial\vec{f_{\psi}}[k]}\right)\left(\frac{\partial\vec{f_{\psi}}[k]}{\partial\vec{\psi}[k]}\right)\label{eq:VanillaLSTMCellDelEDelStateFlowRateOneStepChainRuleStatement}\\
= & \left(\frac{\partial\vec{\psi}[k-1]}{\partial\vec{f_{\chi}}[k]}\right)\left\{ \left(\frac{\partial\vec{f_{\chi}}[k]}{\partial\vec{\alpha}_{cu}[k]}\right)\left(\frac{\partial\vec{\alpha}_{cu}[k]}{\partial\vec{\psi}[k]}\right)+\left(\frac{\partial\vec{f_{\chi}}[k]}{\partial\vec{\alpha}_{cs}[k]}\right)\left(\frac{\partial\vec{\alpha}_{cs}[k]}{\partial\vec{\psi}[k]}\right)\right.\nonumber \\
 & \left.+\left(\frac{\partial\vec{f_{\chi}}[k]}{\partial\vec{\alpha}_{du}[k]}\right)\left(\frac{\partial\vec{\alpha}_{du}[k]}{\partial\vec{\psi}[k]}\right)\right\} +\left(\frac{\partial\vec{\psi}[k-1]}{\partial\vec{f_{\psi}}[k]}\right)\left\{ \left(\frac{\partial\vec{f_{\psi}}[k]}{\partial\vec{\alpha}_{cu}[k]}\right)\left(\frac{\partial\vec{\alpha}_{cu}[k]}{\partial\vec{\psi}[k]}\right)\right.\nonumber \\
 & \left.+\left(\frac{\partial\vec{f_{\psi}}[k]}{\partial\vec{\alpha}_{cs}[k]}\right)\left(\frac{\partial\vec{\alpha}_{cs}[k]}{\partial\vec{\psi}[k]}\right)+\mathit{\mathrm{diag}}\bigg[\vec{g}_{cs}[k]\bigg]\right\} \label{eq:VanillaLSTMCellDelEDelStateFlowRateOneStepExpandedNextStepDerivatives}\\
 & =\left(\mathrm{diag}\left[\vec{g}_{cr}[k-1]\odot\frac{dG_{d}(\vec{z})}{d\vec{z}}\rfloor_{z=\vec{s}[k-1]}\right]+W_{s_{cr}}\mathrm{diag}\left[\vec{r}[k-1]\odot\frac{dG_{c}(z)}{dz}\rfloor_{z=\vec{a}_{cr}[k-1]}\right]\right)\nonumber \\
 & \times\left\{ W_{v_{cu}}\mathit{\mathrm{diag}}\left[\vec{u}[k]\odot\frac{dG_{c}(\vec{z})}{d\vec{z}}\rfloor_{z=\vec{a}_{cu}[k]}\right]+W_{v_{cs}}\mathit{\mathrm{diag}}\left[\vec{s}[k-1]\odot\frac{dG_{c}(\vec{z})}{d\vec{z}}\rfloor_{z=\vec{a}_{cs}[k]}\right]\right.\nonumber \\
 & \left.+W_{v_{du}}\mathit{\mathrm{diag}}\left[\vec{g}_{cu}[k]\odot\frac{dG_{d}(\vec{z})}{d\vec{z}}\rfloor_{z=\vec{a}_{du}[k]}\right]\right\} +\left(W_{s_{cu}}\mathit{\mathrm{diag}}\left[\vec{u}[k]\odot\frac{dG_{c}(\vec{z})}{d\vec{z}}\rfloor_{z=\vec{a}_{cu}[k]}\right]\right.\nonumber \\
 & \left.+W_{s_{cs}}\mathit{\mathrm{diag}}\left[\vec{s}[k-1]\odot\frac{dG_{c}(\vec{z})}{d\vec{z}}\rfloor_{z=\vec{a}_{cs}[k]}\right]+\mathit{\mathrm{diag}}\bigg[\vec{g}_{cs}[k]\bigg]\right)\label{eq:VanillaLSTMCellDelEDelStateFlowRateOneStepExpandedDelEDelValueStepTotal}\\
 & =\mathit{\mathrm{diag}}\left[\vec{u}[k]\odot\frac{dG_{c}(\vec{z})}{d\vec{z}}\rfloor_{z=\vec{a}_{cu}[k]}\right]\nonumber \\
 & \times\left\{ W_{v_{cu}}\left(\mathrm{diag}\left[\vec{g}_{cr}[k-1]\odot\frac{dG_{d}(\vec{z})}{d\vec{z}}\rfloor_{z=\vec{s}[k-1]}\right]+W_{s_{cr}}\mathrm{diag}\left[\vec{r}[k-1]\odot\frac{dG_{c}(z)}{dz}\rfloor_{z=\vec{a}_{cr}[k-1]}\right]\right)+W_{s_{cu}}\right\} \nonumber \\
 & +\mathit{\mathrm{diag}}\left[\vec{s}[k-1]\odot\frac{dG_{c}(\vec{z})}{d\vec{z}}\rfloor_{z=\vec{a}_{cs}[k]}\right]\nonumber \\
 & \times\left\{ W_{v_{cs}}\left(\mathrm{diag}\left[\vec{g}_{cr}[k-1]\odot\frac{dG_{d}(\vec{z})}{d\vec{z}}\rfloor_{z=\vec{s}[k-1]}\right]+W_{s_{cr}}\mathrm{diag}\left[\vec{r}[k-1]\odot\frac{dG_{c}(z)}{dz}\rfloor_{z=\vec{a}_{cr}[k-1]}\right]\right)+W_{s_{cs}}\right\} \nonumber \\
 & +\mathit{\mathrm{diag}}\left[\vec{g}_{cu}[k]\odot\frac{dG_{d}(\vec{z})}{d\vec{z}}\rfloor_{z=\vec{a}_{du}[k]}\right]\nonumber \\
 & \times\left\{ W_{v_{du}}\left(\mathrm{diag}\left[\vec{g}_{cr}[k-1]\odot\frac{dG_{d}(\vec{z})}{d\vec{z}}\rfloor_{z=\vec{s}[k-1]}\right]+W_{s_{cr}}\mathrm{diag}\left[\vec{r}[k-1]\odot\frac{dG_{c}(z)}{dz}\rfloor_{z=\vec{a}_{cr}[k-1]}\right]\right)\right\} +\mathit{\mathrm{diag}}\bigg[\vec{g}_{cs}[k]\bigg]\label{eq:VanillaLSTMCellDelEDelStateFlowRateOneStepTermsCollected}\\
 & =Q\left(k-1,k;\tilde{\Theta}\right)+\mathit{\mathrm{diag}}\bigg[\vec{g}_{cs}[k]\bigg]\label{eq:VanillaLSTMCellDelEDelStateFlowRateOneStepEssentialTerms}
\end{align}

\vspace{-3mm}
\end{samepage}

\pagebreak{}

\begin{samepage}

where $\tilde{\Theta}=\left\{ W_{s_{cu}},W_{v_{cu}},W_{s_{cs}},W_{v_{cs}},W_{s_{cr}},W_{v_{du}}\right\} $,
and $Q\left(k-1,k;\tilde{\Theta}\right)$ subsumes all the terms in
$\frac{\partial\vec{\psi}[k-1]}{\partial\vec{\psi}[k]}$, excluding
$\mathit{\mathrm{diag}}\bigg[\vec{g}_{cs}[k]\bigg]$.

\vspace{-2mm}

Extrapolating $\frac{\partial\vec{\psi}[k-1]}{\partial\vec{\psi}[k]}$
from the step with the index, $n,$ to the step with the index, $l\leq K-1$,
where $l\gg n$, gives:\vspace{-3mm}
\begin{align}
\frac{\partial\vec{\psi}[n]}{\partial\vec{\psi}[l]} & =\prod_{k=n+1}^{l}\frac{\partial\vec{\psi}[k-1]}{\partial\vec{\psi}[k]}\label{eq:VanillaLSTMCellDelEDelStateFlowRateLongRange}
\end{align}

Assuming that the issue of ``exploding gradients'' is handled as
a separate undertaking, the present focus is on the effectiveness
of the LSTM network at assuaging the ``vanishing gradients'' problem.
If the value of the total partial derivative of the objective function,
$E$, with respect to the state signal, $\vec{s}[l]$, at the index,
$l\leq K-1$, where $l\gg n$, is considered to be an impulse of the
error gradient, then Equation \ref{eq:VanillaLSTMCellDelEDelStateFlowRateLongRange}
computes the fractional amount of this corrective stimulus that did
not dissipate across the large number ($l-n$) of steps and is preserved
in $\vec{\psi}[n]$, thereby able to contribute to updating the model
parameters.

The propensity of the LSTM system toward diminishing error gradients
during training can be assessed by evaluating the different modes
of Equation \ref{eq:VanillaLSTMCellDelEDelStateFlowRateOneStepEssentialTerms}
that can cause $\left\Vert \frac{\partial\vec{\psi}[n]}{\partial\vec{\psi}[l]}\right\Vert \approx0$
in Equation \ref{eq:VanillaLSTMCellDelEDelStateFlowRateLongRange}
when $l-n$ is large. A sufficient condition for driving the residual
$\left\Vert \frac{\partial\vec{\psi}[n]}{\partial\vec{\psi}[l]}\right\Vert $
to zero is maintaining $\left\Vert \frac{\partial\vec{\psi}[k-1]}{\partial\vec{\psi}[k]}\right\Vert <1$
at each step with the index, $k$. There are three possibilities for
this outcome:
\begin{itemize}
\item $Q\left(k-1,k;\tilde{\Theta}\right)=\big[0\big]$ and $\vec{g}_{cs}[k]=\vec{0}$
for all values of the step index, $k$. This is the case of the network
being perpetually ``at rest'' (i.e., in a trivial state), which
is not interesting from the practical standpoint.
\item $\vec{g}_{cs}[k]\approx\vec{1}$ and $Q\left(k-1,k;\tilde{\Theta}\right)=-\mathit{\mathrm{diag}}\bigg[\vec{g}_{cs}[k]\bigg]$;
in other words, $Q\left(k-1,k;\tilde{\Theta}\right)$ and $\mathit{\mathrm{diag}}\bigg[\vec{g}_{cs}[k]\bigg]$
``cancel each other out'' for some value of the step index, $k$.
However, satisfying this condition would require a very careful orchestration
of all signals, which is highly unlikely to occur in practice, making
this pathological case unrealistic.
\item The spectral radius of $\left[Q\left(k-1,k;\tilde{\Theta}\right)+\mathit{\mathrm{diag}}\bigg[\vec{g}_{cs}[k]\bigg]\right]$
in Equation \ref{eq:VanillaLSTMCellDelEDelStateFlowRateOneStepEssentialTerms}
is less than unity for all values of the step index, $k$. In this
situation, the error gradient will degrade to negligible levels after
a sufficiently large number of steps. Nevertheless, this behavior
would not be due to a degenerate mode of the system, but as a consequence
of the particular patterns, occurring in the training data. In other
words, some dependencies are naturally short-range.
\end{itemize}
For all remaining cases, the magnitude of $\frac{\partial\vec{\psi}[k-1]}{\partial\vec{\psi}[k]}$
is governed by the triangle inequality:
\begin{align}
\left\Vert \frac{\partial\vec{\psi}[k-1]}{\partial\vec{\psi}[k]}\right\Vert  & \leq\left\Vert Q\left(k-1,k;\tilde{\Theta}\right)\right\Vert +\left\Vert \mathit{\mathrm{diag}}\bigg[\vec{g}_{cs}[k]\bigg]\right\Vert \label{eq:VanillaLSTMCellDelEDelStateFlowRateOneStepEssentialTermsNormTriangleInequality}
\end{align}
The most emblematic regime of the LSTM network arises when $\left\Vert Q\left(k-1,k;\tilde{\Theta}\right)\right\Vert <1$.
Examining the terms in\\ Equation \ref{eq:VanillaLSTMCellDelEDelStateFlowRateOneStepTermsCollected}
exposes multiple ways of restricting signals and parameters that would
create favorable circumstances for this to hold. The following list
prescribes several plausible alternatives (all conditions in each
arrangement must be satisfied)\footnote{Note that all data signals in Equation \ref{eq:VanillaLSTMCellDelEDelStateFlowRateOneStepTermsCollected}
($\vec{u}$, $\vec{s}$, and $\vec{r}$) have the dynamic range of
$2$, because they are at the output node of the warping function,
$G_{d}(\vec{z})$, which is the hyperbolic tangent. Hence, for the
respective term to have the norm of $<1$, the associated parameter
matrices must have the norm of $<\frac{1}{2}$.}:
\begin{itemize}
\item $\left\Vert W_{s_{cu}}\right\Vert <\frac{1}{2}$, $\left\Vert W_{v_{cu}}\right\Vert <\frac{1}{2}$,
$\left\Vert W_{s_{cs}}\right\Vert <\frac{1}{2}$, $\left\Vert W_{v_{cs}}\right\Vert <\frac{1}{2}$,
$\left\Vert W_{v_{du}}\right\Vert <1$
\item the state signal saturates the readout data warping function, $\left\Vert W_{s_{cr}}\right\Vert <\frac{1}{2}$,
$\left\Vert W_{s_{cu}}\right\Vert <\frac{1}{2}$, $\left\Vert W_{s_{cs}}\right\Vert <\frac{1}{2}$
\item the state signal saturates the readout data warping function, the
accumulation signal for the control readout gate saturates its control
warping function, $\left\Vert W_{s_{cu}}\right\Vert <\frac{1}{2}$,
$\left\Vert W_{s_{cs}}\right\Vert <\frac{1}{2}$
\item the control readout gate is turned off, $\left\Vert W_{s_{cu}}\right\Vert <\frac{1}{2}$,
$\left\Vert W_{s_{cs}}\right\Vert <\frac{1}{2}$
\item the accumulation signals for the control update gate and the control
state gate saturate their respective control warping functions, the
update candidate accumulation signal saturates the update candidate
data warping function
\item the control update gate is turned off, the control state gate is turned
off
\end{itemize}
Since the difference between the step indices, $l-n$, is large when
the network is trained to represent long-range dependencies, the powers
of the $Q\left(k-1,k;\tilde{\Theta}\right)$ terms become negligible,
ultimately leading to:\vspace{-6mm}

\begin{align}
\left\Vert \frac{\partial\vec{\psi}[n]}{\partial\vec{\psi}[l]}\right\Vert  & \sim\prod_{k=n+1}^{l}\left\Vert \mathit{\mathrm{diag}}\bigg[\vec{g}_{cs}[k]\bigg]\right\Vert \leqslant1\label{eq:VanillaLSTMCellDelEDelStateFlowRateLongRangeSustainedGradient}
\end{align}
\vspace{-1mm}
Unlike $Q\left(k-1,k;\tilde{\Theta}\right)$, $\mathit{\mathrm{diag}}\bigg[\vec{g}_{cs}[k]\bigg]$
in Equation \ref{eq:VanillaLSTMCellDelEDelStateFlowRateOneStepEssentialTerms}
has no attenuating factors (the multiplier of $\vec{g}_{cs}[n+1]\odot\vec{\psi}[n+1]$
in Equation \ref{eq:VanillaLSTMCellDelEDelStateNextStep} is the identity
matrix). As long as the elements of $\vec{g}_{cs}[n]$ are fractions,
the error gradient will naturally decay. However, if the model is
trained to saturate $\vec{g}_{cs}[n]$ at $\vec{1}$, then the error
gradient is recirculated through Constant Error Carousel.

\end{samepage}

\pagebreak{}

\section{Extensions to the Vanilla LSTM Network\label{sec:Extensions-to-the-Vanilla-LSTM-Network}}

Since its invention, many variants and extensions of the original
LSTM network model have been researched and utilized in practice.
In this section, we will evolve the Vanilla LSTM architecture, derived
in Section \ref{sec:From-RNN-to-Vanilla-LSTM-Network} and explained
in depth in Section \ref{sec:The-Vanilla-LSTM-Network-Mechanism-in-Detail},
along three avenues. Based on the analysis in Section \ref{sec:The-Roots-of-RNN}
as well as the discussions in Section \ref{sec:From-RNN-to-Vanilla-LSTM-Network}
and Section \ref{subsec:Vanilla-LSTM-System-Derivatives-(Backward-Pass)},
we will expand the input from consisting of a single sample to combining
multiple samples within a small context window. In addition, as proposed
in Section \ref{sec:From-RNN-to-Vanilla-LSTM-Network}, we will introduce
a new gate for controlling this richer input signal. Besides these
two novel extensions, we will also include the ``recurrent projection
layer'' in the augmented model, because it proved to be advantageous
in certain sequence modeling applications \citep{conf/interspeech/SakSB14}.

\subsection{External Input Context Windows\label{subsec:Exernal-Input-Context-Windows}}

We will represent the external input context windows by linear filters
that have matrix-valued coefficients and operate on the sequence of
input samples along the dimension of the steps of the sequence produced
by unrolling the LSTM cell. In\\ Equations \ref{eq:VanillaLSTMCellAccumulationControlUpdate}
-- \ref{eq:VanillaLSTMCellAccumulationDataUpdate}, the matrix-vector
products, $W_{x_{cu}}\vec{x}[n]$, $W_{x_{cs}}\vec{x}[n]$, $W_{x_{cr}}\vec{x}[n]$,
and $W_{x_{du}}\vec{x}[n]$, respectively, which involve a single
input sample, $\vec{x}[n]$, will be replaced by the convolutions
of the context window filters, $W_{x_{cu}}[n]$, $W_{x_{cs}}[n]$,
$W_{x_{cr}}[n]$, and $W_{x_{du}}[n]$, respectively, with the input
signal, $\vec{x}[n]$, thereby involving all input samples within
the context window in the computation of the respective accumulation
signal. We choose the context window filters to be non-causal (i.e.,
with the\\ non-zero coefficients defined only for $n\leq0$). This
will enable the accumulation signals to utilize the input samples
from the ``future'' steps of the unrolled LSTM cell without excessively
increasing the number of parameters to be learned, since the input
samples from the ``past'' steps will be already absorbed by the
state signal, $\vec{s}[n]$, due to recurrence. After making these
substitutions, Equations \ref{eq:VanillaLSTMCellAccumulationControlUpdate}
-- \ref{eq:VanillaLSTMCellAccumulationDataUpdate} become:
\begin{align}
\vec{a}_{cu}[n] & =W_{x_{cu}}[n]\ast\vec{x}[n]+W_{s_{cu}}\vec{s}[n-1]+W_{v_{cu}}\vec{v}[n-1]+\vec{b}_{cu}\label{eq:InputConvolutionLSTMCellAccumulationControlUpdate}\\
\vec{a}_{cs}[n] & =W_{x_{cs}}[n]\ast\vec{x}[n]+W_{s_{cs}}\vec{s}[n-1]+W_{v_{cs}}\vec{v}[n-1]+\vec{b}_{cs}\label{eq:InputConvolutionLSTMCellAccumulationControlState}\\
\vec{a}_{cr}[n] & =W_{x_{cr}}[n]\ast\vec{x}[n]+W_{s_{cr}}\vec{s}[n]+W_{v_{cr}}\vec{v}[n-1]+\vec{b}_{cr}\label{eq:InputConvolutionLSTMCellAccumulationControlReadout}\\
\vec{a}_{du}[n] & =W_{x_{du}}[n]\ast\vec{x}[n]+W_{v_{du}}\vec{v}[n-1]+\vec{b}_{du}\label{eq:InputConvolutionLSTMCellAccumulationDataUpdate}
\end{align}
To examine the convolutional terms in more detail, let every context
window filter, $W_{x}[n]$ (with the respective subscript), have $L$
non-zero matrix-valued terms. For example, if $L=4$, then $W_{x}[0]$,
$W_{x}[-1]$, $W_{x}[-2]$, and $W_{x}[-3]$ will be non-zero\footnote{As the members of the expanded LSTM model's parameter set, $\Theta$,
the new matrices and bias vectors are learned during the training
phase.}. 

By the definition of the discrete convolution,
\begin{align}
W_{x}[n]\ast\vec{x}[n] & =\sum_{l=-L+1}^{0}W_{x}[l]\vec{x}[n-l]=\sum_{l=0}^{L-1}W_{x}[-l]\vec{x}[n+l]\label{eq:InputConvolutionLSTMCellConvolutionSum}
\end{align}
In the above example, the result on the left hand side of Equation
\ref{eq:InputConvolutionLSTMCellConvolutionSum} for each step with
the index, $n$, will be influenced by the window spanning $4$ input
samples: $\vec{x}[0]$, $\vec{x}[1]$, $\vec{x}[2]$, and $\vec{x}[3]$.

If we redefine $W_{x}[n]$ to be non-zero for $n\geq0$, then Equation
\ref{eq:InputConvolutionLSTMCellConvolutionSum} simplifies to:
\begin{align}
W_{x}[n]\ast\vec{x}[n] & =\sum_{l=0}^{L-1}W_{x}[l]\vec{x}[n+l]\label{eq:InputConvolutionLSTMCellConvolutionSumIndexSignChange}
\end{align}
The dependence of the left hand side of Equation \ref{eq:InputConvolutionLSTMCellConvolutionSumIndexSignChange}
on the input samples from the ``future'' steps of the unrolled LSTM
cell is readily apparent from the expression for the convolution sum
on the right hand side of Equation \ref{eq:InputConvolutionLSTMCellConvolutionSumIndexSignChange}.

By taking advantage of the available input samples within a small
window surrounding each step of the sequence, the system can learn
to ``discern'' the context in which the given step occurs. The inspiration
for this ``look-ahead'' extension comes from the way people sometimes
find it beneficial to read forward to the end of the sentence in order
to better understand a phrase occurring in the earlier part of the
sentence. It would be interesting to explore the relative trade-offs
between the cost of adding a small number of parameter matrices to
$\Theta$ so as to accommodate the input context windows with the
computational burden of training a bi-directional LSTM network \citep{schuster97bidirectional},
and to compare the performance of the two architectures on several
data sets\footnote{The non-causal input context windows can be readily employed as part
of the bi-directional RNN or LSTM network. The number of additional
parameter matrices to support input convolutions will double compared
to the uni-directional case, because the samples covered by the windows
are situated on the positive side along the direction of the sequence.}.

\subsection{Recurrent Projection Layer\label{subsec:Recurrent-Projection-Layer}}

Another modification of the Vanilla LSTM cell redefines the cell's
value signal to be the product of an additional matrix of weights\footnote{This new matrix of weights, to be learned as part of training, is
known as the ``recurrent projection layer'' \citep{conf/interspeech/SakSB14}.} with the LSTM cell's value signal from Equation \ref{eq:VanillaLSTMCellValue}.
To insert the recurrent projection layer into the Vanilla LSTM cell,
we adjust Equation \ref{eq:VanillaLSTMCellValue} as follows:
\begin{align}
\vec{q}[n] & =\vec{g}_{cr}[n]\odot\vec{r}[n]\label{eq::RecurrentProjectionLayerLSTMCellQuantity}\\
\vec{v}[n] & =W_{q_{dr}}\vec{q}[n]\label{eq:RecurrentProjectionLayerLSTMCellValue}
\end{align}
where $W_{q_{dr}}$ implements the recurrent projection layer\footnote{In our nomenclature, the matrix of weights, $W_{q_{dr}}$, links the
cell's ``qualifier'' signal, $\vec{q}[n]$, with the cell's value
signal, $\vec{v}[n]$, along the readout data path of the cell.}, and the intermediate cell quantity, $\vec{q}[n]$ (which we will
call the cell's ``qualifier'' signal), replaces what used to be
the cell's value signal in the Vanilla LSTM cell. The new value signal
of the cell from Equation \ref{eq:RecurrentProjectionLayerLSTMCellValue}
will now be used for computing the accumulation signals in Equations
\ref{eq:InputConvolutionLSTMCellAccumulationControlUpdate} -- \ref{eq:InputConvolutionLSTMCellAccumulationDataUpdate}.

Let $d_{v}$ denote the dimensionality of the observable value signal
of the cell; then $\vec{v}\in\mathbb{R}^{d_{v}}$ and $W_{q_{dr}}\mathbb{\in R}^{d_{v}\times d_{s}}$.
The degree to which the dimensionality reduction of the cell's value
signal can be tolerated for the given application directly contributes
to speeding up the training phase of the system. By allowing $d_{v}<d_{s}$,
the matrix multipliers of all the terms involving $\vec{v}[n-1]$,
which dominate Equations \ref{eq:InputConvolutionLSTMCellAccumulationControlUpdate}
-- \ref{eq:InputConvolutionLSTMCellAccumulationDataUpdate} (or Equations
\ref{eq:VanillaLSTMCellAccumulationControlUpdate} -- \ref{eq:VanillaLSTMCellAccumulationDataUpdate}
in the absence of the external input context windows), will contain
correspondingly fewer columns. In contrast, in the Vanilla LSTM cell
as covered in Section \ref{sec:The-Vanilla-LSTM-Network-Mechanism-in-Detail}
(i.e., without the recurrent projection layer), $d_{v}$ must equal
$d_{s}$, since $\vec{v}[n]$ is on the same data path as $\vec{s}[n]$,
with no signal transformations along the data path between them. Hence,
the addition of the recurrent projection layer to the Vanilla LSTM
cell brings about the flexibility of trading off the representational
capacity of the cell with the computational cost of learning its parameters
\citep{conf/interspeech/SakSB14}.

\subsection{Controlling External Input with a New Gate\label{subsec:Controling-Exernal-Input-With-A-New-Gate}}

Section \ref{sec:From-RNN-to-Vanilla-LSTM-Network} argues that the
two components of the data update accumulation node, $\vec{a}_{du}[n]$,
in the Vanilla LSTM cell are not treated the same way from the standpoint
of control. While the readout candidate signal is throttled by the
control readout gate, $\vec{g}_{cr}[n]$, the external input is always
injected at the full $100\%$ contribution of its signal strength.
This is not as much of an issue for the control accumulation nodes
($\vec{a}_{cu}[n]$, $\vec{a}_{cs}[n]$, and $\vec{a}_{cr}[n]$),
because they influence the scaling of the data signals, not the relative
mixing of the data signals themselves. However, since the data update
accumulation node, $\vec{a}_{du}[n]$, is directly in the path of
the cell's state signal, the ability to regulate both components of
$\vec{a}_{du}[n]$ can improve the cell's capacity to adapt to the
nuances and fluctuations in the training data. For instance, the control
readout gate, $\vec{g}_{cr}[n]$, can diminish the effect of the cell's
readout signal from the adjacent step in favor of making the external
input signal at the given step more prominent in the make up of the
cell's state signal. Likewise, having the additional flexibility to
fine-tune the external input component of $\vec{a}_{du}[n]$ at the
same granularity as its readout component (i.e., at the level of the
individual step with the index, $n$) would provide a means for training
the LSTM cell to suppress interference due to noisy or spurious input
samples.

As a mechanism for adjusting the contribution of the external input,
Equation \ref{eq:CanonicalRNNRHSPartValueInputUpdate} introduced
the control input gate, $\vec{g}_{cx}[n]$, which we will apply to
the convolution term in Equation \ref{eq:InputConvolutionLSTMCellAccumulationDataUpdate}.
Analogously to the other gates, $\vec{g}_{cx}[n]$ is computed by
taking the warping function for control, given by Equation \ref{eq:VanillaLSTMCellWarpingFunctionControl},
of the accumulation signal for controlling the input, element by element.

From Section \ref{subsec:Vanilla-LSTM-Cell-Model-Parameters} and
Equation \ref{eq:InputConvolutionLSTMCellAccumulationDataUpdate},
the data update accumulation node, $\vec{a}_{du}[n]$, is followed
by the data warping function that produces the update candidate signal,
$\vec{u}[n]$, of the cell at the present step with the index, $n$.
The new control input gate, $\vec{g}_{cx}[n]$, will throttle $W_{x_{du}}[n]\ast\vec{x}[n]$,
the term representing the composite external input signal in Equation
\ref{eq:InputConvolutionLSTMCellAccumulationDataUpdate} for $\vec{a}_{du}[n]$.
Letting $\vec{\xi}_{x_{du}}[n]\equiv W_{x_{du}}[n]\ast\vec{x}[n]$,
where $\vec{\xi}_{x_{du}}[n]$ denotes the composite external input
signal for the data update accumulation node, $\vec{a}_{du}[n]$,
this gating operation will be accomplished through the element-wise
multiplication, $\vec{g}_{cx}[n]\odot\vec{\xi}_{x_{du}}[n]$, the
method used by all the other gates to control the fractional amount
of their designated data signals.

The equations for accommodating the new control input gate, $\vec{g}_{cx}[n]$,
as part of the LSTM cell design as well as the equation for the data
update accumulation node, $\vec{a}_{du}[n]$, modified to take advantage
of this new gate, are provided below:
\begin{align}
\vec{\xi}_{x_{cx}}[n] & =W_{x_{cx}}[n]\ast\vec{x}[n]\label{eq:GateControlInputLSTMCellControlInputTermInputConvolution}\\
\vec{a}_{cx}[n] & =\vec{\xi}_{x_{cx}}[n]+W_{s_{cx}}\vec{s}[n-1]+W_{v_{cx}}\vec{v}[n-1]+\vec{b}_{cx}\label{eq:GateControlInputLSTMCellAccumulationControlInput}\\
\vec{g}_{cx}[n] & =G_{c}(\vec{a}_{cx}[n])\label{eq:GateControlInputLSTMCellGateWarpFunctionControlOfAccumulationControlInput}\\
\vec{\xi}_{x_{du}}[n] & =W_{x_{du}}[n]\ast\vec{x}[n]\label{eq:GateControlInputLSTMCellDataInputTermInputConvolution}\\
\vec{a}_{du}[n] & =\vec{g}_{cx}[n]\odot\vec{\xi}_{x_{du}}[n]+W_{v_{du}}\vec{v}[n-1]+\vec{b}_{du}\label{eq:GateControlInputLSTMCellAccumulationDataUpdate}
\end{align}

where the additional parameters, needed to characterize the accumulation
node, $\vec{a}_{cx}[n]$, of the new control input gate, $\vec{g}_{cx}[n]$,
have the following interpretation:
\begin{itemize}
\item $W_{x_{cx}}[l]\in\mathbb{R}^{d_{s}\times d_{x}}$ -- the matrices
of weights (for $0\leq l\leq L-1$) connecting the input signal, $\vec{x}[n+l]$,
at the step with the index, $n+l$, to the ``control input'' accumulation
node, $\vec{a}_{cx}[n]$, of the cell at the present step with the
index, $n$
\item $W_{s_{cx}}\mathbb{\in R}^{d_{s}\times d_{s}}$ -- the matrix of
weights connecting the state signal, $\vec{s}[n-1]$, at the adjacent
lower-indexed step with the index, $n-1$, to the ``control input''
accumulation node, $\vec{a}_{cx}[n]$, of the cell at the present
step with the index, $n$
\item $W_{v_{cx}}\mathbb{\in R}^{d_{s}\times d_{v}}$ -- the matrix of
weights connecting the externally-accessible (observable) value signal,
$\vec{v}[n-1]$, at the adjacent lower-indexed step with the index,
$n-1$, to the ``control input'' accumulation node, $\vec{a}_{cx}[n]$,
of the cell at the present step with the index, $n$
\item $\vec{b}_{cx}\in\mathbb{R}^{d_{s}}$ -- the vector of bias elements
for the ``control input'' accumulation node, $\vec{a}_{cx}[n]$,
of the cell at the present step with the index, $n$
\item $W_{x_{du}}[l]\in\mathbb{R}^{d_{s}\times d_{x}}$ -- the matrices
of weights (for $0\leq l\leq L-1$) connecting the input signal, $\vec{x}[n+l]$,
at the step with the index, $n+l$, to the ``data update'' accumulation
node, $\vec{a}_{du}[n]$, of the cell at the present step with the
index, $n$
\end{itemize}

\subsection{Augmented LSTM System Equations (\textquotedblleft Forward Pass\textquotedblright )\label{subsec:Augmented-LSTM-System-Equations-(Forward-Pass)}}

We are now ready to assemble the equations for the Augmented LSTM
system by enhancing the Vanilla LSTM network with the new functionality,
presented earlier in this section -- the recurrent projection layer,
the non-causal input context windows, and the input gate:
\begin{align}
\vec{\xi}_{x_{cu}}[n] & =W_{x_{cu}}[n]\ast\vec{x}[n]\label{eq:AugmentedLSTMCellControlUpdateTermInputConvolution}\\
\vec{a}_{cu}[n] & =\vec{\xi}_{x_{cu}}[n]+W_{s_{cu}}\vec{s}[n-1]+W_{v_{cu}}\vec{v}[n-1]+\vec{b}_{cu}\label{eq:AugmentedLSTMCellAccumulationControlUpdate}\\
\vec{g}_{cu}[n] & =G_{c}(\vec{a}_{cu}[n])\label{eq:AugmentedLSTMCellGateControlUpdate}\\
\vec{\xi}_{x_{cs}}[n] & =W_{x_{cs}}[n]\ast\vec{x}[n]\label{eq:AugmentedLSTMCellControlStateTermInputConvolution}\\
\vec{a}_{cs}[n] & =\vec{\xi}_{x_{cs}}[n]+W_{s_{cs}}\vec{s}[n-1]+W_{v_{cs}}\vec{v}[n-1]+\vec{b}_{cs}\label{eq:AugmentedLSTMCellAccumulationControlState}\\
\vec{g}_{cs}[n] & =G_{c}(\vec{a}_{cs}[n])\label{eq:AugmentedLSTMCellGateControlState}\\
\vec{\xi}_{x_{cr}}[n] & =W_{x_{cr}}[n]\ast\vec{x}[n]\label{eq:AugmentedLSTMCellControlReadoutTermInputConvolution}\\
\vec{a}_{cr}[n] & =\vec{\xi}_{x_{cr}}[n]+W_{s_{cr}}\vec{s}[n]+W_{v_{cr}}\vec{v}[n-1]+\vec{b}_{cr}\label{eq:AugmentedLSTMCellAccumulationControlReadout}\\
\vec{g}_{cr}[n] & =G_{c}(\vec{a}_{cr}[n])\label{eq:AugmentedLSTMCellGateControlReadout}\\
\vec{\xi}_{x_{cx}}[n] & =W_{x_{cx}}[n]\ast\vec{x}[n]\label{eq:AugmentedLSTMCellControlInputTermInputConvolution}\\
\vec{a}_{cx}[n] & =\vec{\xi}_{x_{cx}}[n]+W_{s_{cx}}\vec{s}[n-1]+W_{v_{cx}}\vec{v}[n-1]+\vec{b}_{cx}\label{eq:AugmentedLSTMCellAccumulationControlInput}\\
\vec{g}_{cx}[n] & =G_{c}(\vec{a}_{cx}[n])\label{eq:AugmentedLSTMCellGateControlInput}\\
\vec{\xi}_{x_{du}}[n] & =W_{x_{du}}[n]\ast\vec{x}[n]\label{eq:AugmentedLSTMCellDataUpdateTermInputConvolution}\\
\vec{a}_{du}[n] & =\vec{g}_{cx}[n]\odot\vec{\xi}_{x_{du}}[n]+W_{v_{du}}\vec{v}[n-1]+\vec{b}_{du}\label{eq:AugmentedLSTMCellAccumulationDataUpdate}\\
\vec{u}[n] & =G_{d}(\vec{a}_{du}[n])\label{eq:AugmentedLSTMCellUpdate}\\
\vec{s}[n] & =\vec{g}_{cs}[n]\odot\vec{s}[n-1]+\vec{g}_{cu}[n]\odot\vec{u}[n]\label{eq:AugmentedLSTMCellState}\\
\vec{r}[n] & =G_{d}(\vec{s}[n])\label{eq:AugmentedLSTMCellReadout}\\
\vec{q}[n] & =\vec{g}_{cr}[n]\odot\vec{r}[n]\label{eq:AugmentedLSTMCellQuantity}\\
\vec{v}[n] & =W_{q_{dr}}\vec{q}[n]\label{eq:AugmentedLSTMCellValue}
\end{align}
with the dimensions of the parameters adjusted to take into account
the recurrent projection layer:\\ $W_{x_{cu}}[l]\in\mathbb{R}^{d_{s}\times d_{x}}$,
$W_{s_{cu}}\in\mathbb{R}^{d_{s}\times d_{s}}$, $W_{v_{cu}}\in\mathbb{R}^{d_{s}\times d_{v}}$,
$\vec{b}_{cu}\in\mathbb{R}^{d_{s}}$,\\ $W_{x_{cs}}[l]\in\mathbb{R}^{d_{s}\times d_{x}}$,
$W_{s_{cs}}\in\mathbb{R}^{d_{s}\times d_{s}}$, $W_{v_{cs}}\in\mathbb{R}^{d_{s}\times d_{v}}$,
$\vec{b}_{cs}\in\mathbb{R}^{d_{s}}$,\\ $W_{x_{cr}}[l]\in\mathbb{R}^{d_{s}\times d_{x}}$,
$W_{s_{cr}}\in\mathbb{R}^{d_{s}\times d_{s}}$, $W_{v_{cr}}\in\mathbb{R}^{d_{s}\times d_{v}}$,
$\vec{b}_{cr}\in\mathbb{R}^{d_{s}}$,\\ $W_{x_{cx}}[l]\in\mathbb{R}^{d_{s}\times d_{x}}$,
$W_{s_{cx}}\in\mathbb{R}^{d_{s}\times d_{s}}$, $W_{v_{cx}}\in\mathbb{R}^{d_{s}\times d_{v}}$,
$\vec{b}_{cx}\in\mathbb{R}^{d_{s}}$,\\ $W_{x_{du}}[l]\in\mathbb{R}^{d_{s}\times d_{x}}$,
$W_{v_{du}}\in\mathbb{R}^{d_{s}\times d_{v}}$, $\vec{b}_{du}\in\mathbb{R}^{d_{s}}$,
and $W_{q_{dr}}\in\mathbb{R}^{d_{v}\times d_{s}}$, where $0\leq l\leq L-1$
and $d_{v}\leq d_{s}$.

The schematic diagram of the Augmented LSTM cell appears in Figure
\ref{fig:AugmentedLSTMCellSchematicsAndInputMatrixFilters}.

Combining all the matrix and vector parameters of the Augmented LSTM
cell, described by Equations \ref{eq:AugmentedLSTMCellControlUpdateTermInputConvolution}
-- \ref{eq:AugmentedLSTMCellValue}, into:
\begin{align}
\begin{gathered}\Theta\equiv\end{gathered}
 & \begin{Bmatrix}W_{x_{cu}}[l], & W_{s_{cu}}, & W_{v_{cu}}, & \vec{b}_{cu},\\
W_{x_{cs}}[l], & W_{s_{cs}}, & W_{v_{cs}}, & \vec{b}_{cs},\\
W_{x_{cr}}[l], & W_{s_{cr}}, & W_{v_{cr}}, & \vec{b}_{cr},\\
W_{x_{cx}}[l], & W_{s_{cx}}, & W_{v_{cx}}, & \vec{b}_{cx},\\
W_{x_{du}}[l], &  & W_{v_{du}}, & \vec{b}_{du},\\
 &  & W_{q_{dr}}
\end{Bmatrix}\label{eq:AugmentedLSTMCellModelParametersThematicallyArranged}
\end{align}
completes the definition of the inference phase (forward pass) of
the Augmented LSTM cell.

\begin{figure}[tph]
\subfloat[{Augmented LSTM network cell. The bias parameters, $\vec{b},$ have
been omitted from the figure for brevity. They can be assumed to be
included without \newline the loss of generality by appending an
additional element, always set to $1$, to the input signal vector,
$\vec{x}[n]$, and increasing the row dimensions of all \newline
corresponding weight matrices by $1$. To minimize clutter in the
diagram, the context windows are shown separately in Figure \ref{fig:AugmentedLSTMCellInputMatrixFilters}
below).\label{fig:AugmentedLSTMCellSchematics}}]%
{\includegraphics[scale=0.39]{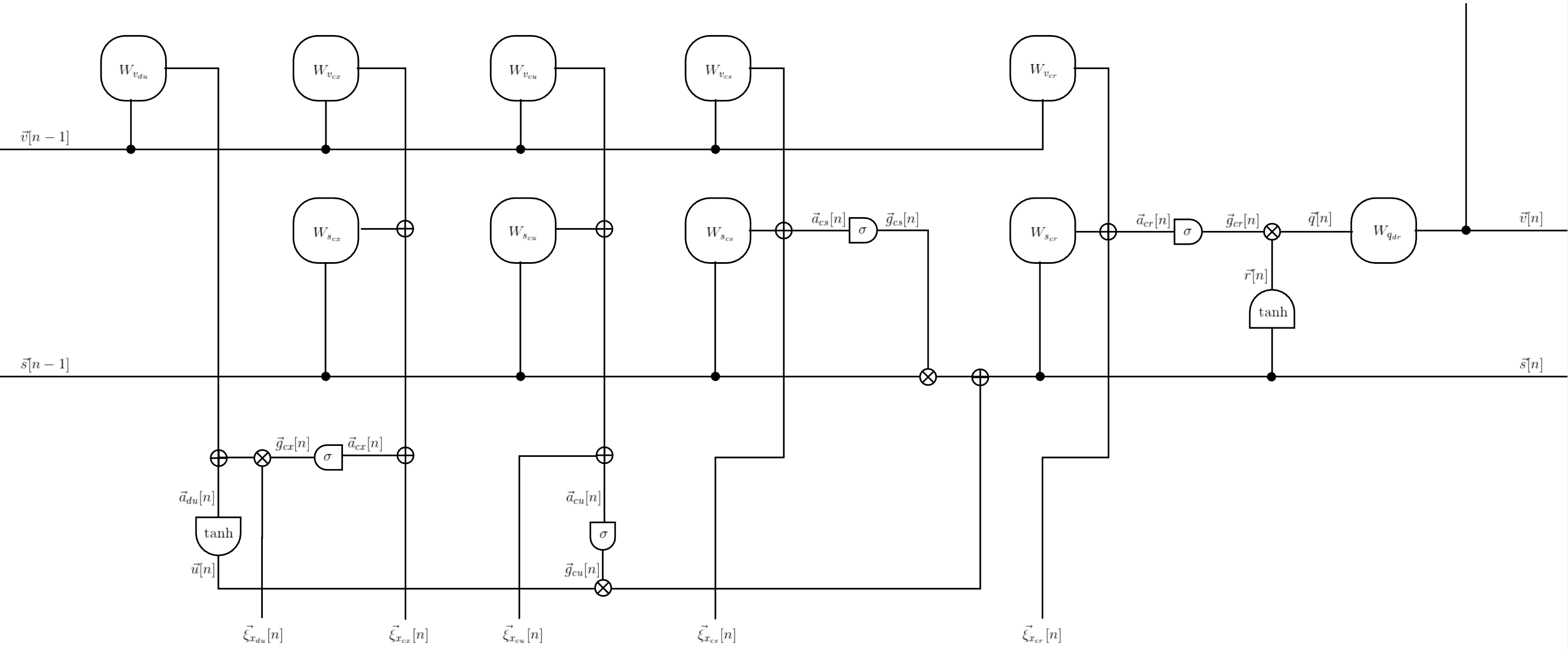}

}\vfill{}
\subfloat[{The non-causal context windows, implemented as the convolutions (along
the dimension of the sequence index, $n$) of the external input,
$\vec{x}[n]$, with linear filters, whose coefficients are matrix-valued.
The operations in this figure compute the quantities of the form ${\vec{\xi}_{x}[n]=W_{x}[n]\ast\vec{x}[n]}$,
corresponding to their respective subscripts.\label{fig:AugmentedLSTMCellInputMatrixFilters}}]%
{\includegraphics[scale=0.5]{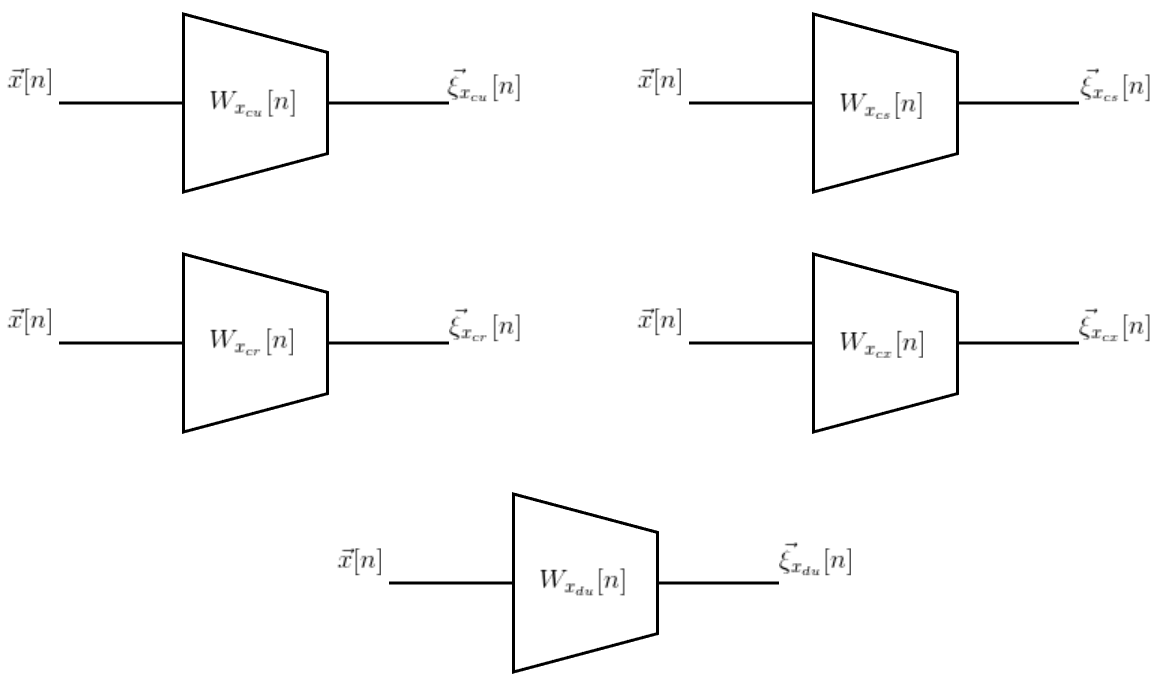}

}\caption{Augmented LSTM network cell system schematics.\label{fig:AugmentedLSTMCellSchematicsAndInputMatrixFilters}}
\end{figure}

\subsection{Augmented LSTM System Derivatives: Backward Pass\label{subsec:Augmented-LSTM-System-Derivatives-(Backward-Pass)}}

The equations for training the Augmented LSTM cell, unrolled for $K$
steps, using BPTT, are obtained by adopting the same method as was
used for the Vanilla LSTM cell in Section \ref{subsec:Vanilla-LSTM-System-Derivatives-(Backward-Pass)}.
We rely on the same border and intermediate total partial derivatives,
appearing in Equations \ref{eq:VanillaLSTMCellDelEDelValueStepTotalDefinition}
-- \ref{eq:VanillaLSTMCellDelEDelStateStepTotalDefinition}, with
the addition of the total partial derivative of the objective function,
$E$, with respect to the qualifier signal, $\vec{q}[n]$:
\begin{align}
\vec{\beta}[n] & \equiv\vec{\nabla}_{\vec{q}[n]}E=\frac{\partial E}{\partial\vec{q}[n]}\label{eq:AugmentedLSTMCellDelEDelQuantityStepTotalDefinition}
\end{align}
which is an intermediate total partial derivative that reflects the
insertion of the projection layer into the cell's data path. Applying
the chain rule to the Augmented LSTM cell, defined by Equations \ref{eq:AugmentedLSTMCellControlUpdateTermInputConvolution}
-- \ref{eq:AugmentedLSTMCellValue}, and judiciously utilizing all
of these border and intermediate total partial derivatives, yields
the backward pass equations for the Augmented LSTM cell:\vspace{-3mm}

\begin{align}
\vec{\chi}[n] & =\left(\frac{\partial\vec{y}[n]}{\partial\vec{v}[n]}\right)^{T}\left(\frac{\partial E}{\partial\vec{y}[n]}\right)+\vec{f_{\chi}}[n+1]\label{eq:AugmentLSTMCellDelEDelValueStepTotal}\\
\vec{\beta}[n] & =\left(\frac{\partial\vec{v}[n]}{\partial\vec{q}[n]}\right)^{T}\left(\frac{\partial E}{\partial\vec{v}[n]}\right)=W_{q_{dr}}^{T}\vec{\chi}[n]\label{eq:AugmentedLSTMCellDelEDelQuantityStepTotal}\\
\vec{\rho}[n] & =\left(\frac{\partial\vec{q}[n]}{\partial\vec{r}[n]}\right)^{T}\left(\frac{\partial E}{\partial\vec{q}[n]}\right)=\vec{(\nabla}_{\vec{q}[n]}E)\odot\vec{g}_{cr}[n]=\vec{\beta}[n]\odot\vec{g}_{cr}[n]=W_{q_{dr}}^{T}\vec{\chi}[n]\odot\vec{g}_{cr}[n]\label{eq:AugmentedLSTMCellDelEDelReadoutStepTotal}\\
\vec{\gamma}_{cr}[n] & =\frac{\partial E}{\partial\vec{q}[n]}\frac{\partial\vec{q}[n]}{\partial\vec{g}_{cr}[n]}=\vec{(\nabla}_{\vec{q}[n]}E)\odot\vec{r}[n]=\vec{\beta}[n]\odot\vec{r}[n]=W_{q_{dr}}^{T}\vec{\chi}[n]\odot\vec{r}[n]\label{eq:AugmentedLSTMCellDelEDelGateControlReadoutStepTotal}\\
\vec{\alpha}_{cr}[n] & =\vec{\gamma}_{cr}[n]\odot\frac{\partial\vec{g}_{cr}[n]}{\partial\vec{a}_{cr}[n]}=\vec{\gamma}_{cr}[n]\odot\frac{dG_{c}(z)}{dz}\rfloor_{z=\vec{a}_{cr}[n]}=W_{q_{dr}}^{T}\vec{\chi}[n]\odot\vec{r}[n]\odot\frac{dG_{c}(z)}{dz}\rfloor_{z=\vec{a}_{cr}[n]}\label{eq:AugmentedLSTMCellDelEDelAccumulationControlReadoutStepTotal}\\
\vec{\psi}[n] & =\vec{\rho}[n]\odot\frac{\partial\vec{r}[n]}{\partial\vec{s}[n]}+\frac{\partial\vec{a}_{cr}[n]}{\partial\vec{s}[n]}\vec{\alpha}_{cr}[n]+\vec{f_{\psi}}[n+1]\nonumber \\
 & =\vec{\rho}[n]\odot\frac{dG_{d}(\vec{z})}{d\vec{z}}\rfloor_{z=\vec{s}[n]}+W_{s_{cr}}\vec{\alpha}_{cr}[n]+\vec{f_{\psi}}[n+1]\nonumber \\
 & =W_{q_{dr}}^{T}\vec{\chi}[n]\odot\vec{g}_{cr}[n]\odot\frac{dG_{d}(\vec{z})}{d\vec{z}}\rfloor_{z=\vec{s}[n]}+W_{s_{cr}}\vec{\alpha}_{cr}[n]+\vec{f_{\psi}}[n+1]\label{eq:AugmentedLSTMCellDelEDelStateStepTotal}\\
\vec{\alpha}_{cs}[n] & =\vec{\psi}[n]\odot\frac{\partial\vec{s}[n]}{\partial\vec{g}_{cs}[n]}\odot\frac{\partial\vec{g}_{cs}[n]}{\partial\vec{a}_{cs}[n]}=\vec{\psi}[n]\odot\vec{s}[n-1]\odot\frac{dG_{c}(\vec{z})}{d\vec{z}}\rfloor_{z=\vec{a}_{cs}[n]}\label{eq:AugmentedLSTMCellDelEDelAccumulationControlStateStepTotal}\\
\vec{\alpha}_{cu}[n] & =\vec{\psi}[n]\odot\frac{\partial\vec{s}[n]}{\partial\vec{g}_{cu}[n]}\odot\frac{\partial\vec{g}_{cu}[n]}{\partial\vec{a}_{cu}[n]}=\vec{\psi}[n]\odot\vec{u}[n]\odot\frac{dG_{c}(\vec{z})}{d\vec{z}}\rfloor_{z=\vec{a}_{cu}[n]}\label{eq:AugmentedLSTMCellDelEDelAccumulationControlUpdateStepTotal}\\
\vec{\alpha}_{du}[n] & =\vec{\psi}[n]\odot\frac{\partial\vec{s}[n]}{\partial\vec{u}[n]}\odot\frac{dG_{d}(\vec{z})}{d\vec{z}}\rfloor_{z=\vec{a}_{du}[n]}=\vec{\psi}[n]\odot\vec{g}_{cu}[n]\odot\frac{dG_{d}(\vec{z})}{d\vec{z}}\rfloor_{z=\vec{a}_{du}[n]}\label{eq:AugmentedLSTMCellDelEDelAccumulationDataUpdateStepTotal}\\
\vec{\gamma}_{cx}[n] & =\vec{\alpha}_{du}[n]\frac{\partial\vec{u}[n]}{\partial\vec{g}_{cx}[n]}=\vec{\alpha}_{du}[n]\odot\vec{\xi}_{x_{du}}[n]\label{eq:AugmentedLSTMCellDelEDelGateControlInputStepTotal}\\
\vec{\alpha}_{cx}[n] & =\vec{\gamma}_{cx}[n]\odot\frac{\partial\vec{g}_{cx}[n]}{\partial\vec{a}_{cx}[n]}=\vec{\gamma}_{cx}[n]\odot\frac{dG_{c}(z)}{dz}\rfloor_{z=\vec{a}_{cx}[n]}=\vec{\alpha}_{du}[n]\odot\vec{\xi}_{x_{du}}[n]\odot\frac{dG_{c}(z)}{dz}\rfloor_{z=\vec{a}_{cx}[n]}\nonumber \\
 & =\vec{\psi}[n]\odot\vec{g}_{cu}[n]\odot\vec{\xi}_{x_{du}}[n]\odot\frac{dG_{d}(\vec{z})}{d\vec{z}}\rfloor_{z=\vec{a}_{du}[n]}\odot\frac{dG_{c}(z)}{dz}\rfloor_{z=\vec{a}_{cx}[n]}\label{eq:AugmentedLSTMCellDelEDelAccumulationControlInputStepTotal}
\end{align}
where:
\begin{align}
\vec{f_{\chi}}[n+1] & =W_{v_{cu}}^{T}\vec{\alpha}_{cu}[n+1]+W_{v_{cs}}^{T}\vec{\alpha}_{cs}[n+1]+W_{v_{cr}}^{T}\vec{\alpha}_{cr}[n+1]+W_{v_{cx}}^{T}\vec{\alpha}_{cx}[n+1]+W_{v_{du}}^{T}\vec{\alpha}_{du}[n+1]\label{eq:AugmentedLSTMCellDelEDelValueNextStep}\\
\vec{f_{\psi}}[n+1] & =W_{s_{cu}}^{T}\vec{\alpha}_{cu}[n+1]+W_{s_{cs}}^{T}\vec{\alpha}_{cs}[n+1]+W_{s_{cx}}^{T}\vec{\alpha}_{cx}[n+1]+\vec{g}_{cs}[n+1]\odot\vec{\psi}[n+1]\label{eq:AugmentedLSTMCellDelEDelStateNextStep}
\end{align}
Referring once again to the definition of the Augmented LSTM cell
in Equations \ref{eq:AugmentedLSTMCellControlUpdateTermInputConvolution}
-- \ref{eq:AugmentedLSTMCellValue}, we obtain:
\begin{align}
\frac{\partial E}{\partial W_{x_{cu}}[l]}[n] & =\vec{\alpha}_{cu}[n]\vec{x}^{T}[n+l]\label{eq:AugmentedLSTMCellDelEDelInputCoefficientMatrixControlUpdateStepTotal}\\
\frac{\partial E}{\partial W_{s_{cu}}}[n] & =\vec{\alpha}_{cu}[n]\vec{s}^{T}[n-1]\label{eq:AugmentedLSTMCellDelEDelStateMatrixControlUpdateStepTotal}\\
\frac{\partial E}{\partial W_{v_{cu}}}[n] & =\vec{\alpha}_{cu}[n]\vec{v}^{T}[n-1]\label{eq:AugmentedLSTMCellDelEDelValueMatrixControlUpdateStepTotal}\\
\frac{\partial E}{\partial\vec{b}_{cu}}[n] & =\vec{\alpha}_{cu}[n]\label{eq:AugmentedLSTMCellDelEDelBiasVectorControlUpdateStepTotal}\\
\frac{\partial E}{\partial W_{x_{cs}}[l]}[n] & =\vec{\alpha}_{cs}[n]\vec{x}^{T}[n+l]\label{eq:AugmentedLSTMCellDelEDelInputCoefficientMatrixControlStateStepTotal}\\
\frac{\partial E}{\partial W_{s_{cs}}}[n] & =\vec{\alpha}_{cs}[n]\vec{s}^{T}[n-1]\label{eq:AugmentedLSTMCellDelEDelStateMatrixControlStateStepTotal}\\
\frac{\partial E}{\partial W_{v_{cs}}}[n] & =\vec{\alpha}_{cs}[n]\vec{v}^{T}[n-1]\label{eq:AugmentedLSTMCellDelEDelValueMatrixControlStateStepTotal}\\
\frac{\partial E}{\partial\vec{b}_{cs}}[n] & =\vec{\alpha}_{cs}[n]\label{eq:AugmentedLSTMCellDelEDelBiasVectorControlStateStepTotal}\\
\frac{\partial E}{\partial W_{x_{cr}}[l]}[n] & =\vec{\alpha}_{cr}[n]\vec{x}^{T}[n+l]\label{eq:AugmentedLSTMCellDelEDelInputCoefficientMatrixControlReadoutStepTotal}\\
\frac{\partial E}{\partial W_{s_{cr}}}[n] & =\vec{\alpha}_{cr}[n]\vec{s}^{T}[n]\label{eq:AugmentedLSTMCellDelEDelStateMatrixControlReadoutStepTotal}\\
\frac{\partial E}{\partial W_{v_{cr}}}[n] & =\vec{\alpha}_{cr}[n]\vec{v}^{T}[n-1]\label{eq:AugmentedLSTMCellDelEDelValueMatrixControlReadoutStepTotal}\\
\frac{\partial E}{\partial\vec{b}_{cr}}[n] & =\vec{\alpha}_{cr}[n]\label{eq:AugmentedLSTMCellDelEDelBiasVectorControlReadoutStepTotal}\\
\frac{\partial E}{\partial W_{x_{cx}}[l]}[n] & =\vec{\alpha}_{cx}[n]\vec{x}^{T}[n+l]\label{eq:AugmentedLSTMCellDelEDelInputCoefficientMatrixControlInputStepTotal}\\
\frac{\partial E}{\partial W_{s_{cx}}}[n] & =\vec{\alpha}_{cx}[n]\vec{s}^{T}[n-1]\label{eq:AugmentedLSTMCellDelEDelStateMatrixControlInputStepTotal}\\
\frac{\partial E}{\partial W_{v_{cx}}}[n] & =\vec{\alpha}_{cx}[n]\vec{v}^{T}[n-1]\label{eq:AugmentedLSTMCellDelEDelValueMatrixControlInputStepTotal}\\
\frac{\partial E}{\partial\vec{b}_{cx}}[n] & =\vec{\alpha}_{cx}[n]\label{eq:AugmentedLSTMCellDelEDelBiasVectorControlInputStepTotal}\\
\frac{\partial E}{\partial W_{x_{du}}[l]}[n] & =\vec{\alpha}_{du}[n]\vec{x}^{T}[n+l]\label{eq:AugmentedLSTMCellDelEDelInputCoefficientMatrixDataUpdateStepTotal}\\
\frac{\partial E}{\partial W_{v_{du}}}[n] & =\vec{\alpha}_{du}[n]\vec{v}^{T}[n-1]\label{eq:AugmentedLSTMCellDelEDelValueMatrixDataUpdateStepTotal}\\
\frac{\partial E}{\partial\vec{b}_{du}}[n] & =\vec{\alpha}_{du}[n]\label{eq:AugmentedLSTMCellDelEDelBiasVectorDataUpdateStepTotal}\\
\frac{\partial E}{\partial W_{q_{dr}}}[n] & =\vec{\chi}[n]\vec{q}^{T}[n]\label{eq:AugmentedLSTMCellDelEDelQuantityMatrixDataReadoutStepTotal}
\end{align}
where $0\leq l\leq L-1$ and $0\leq n\leq K-1$. Arranged to parallel
the structure of $\Theta$, defined in Equation \ref{eq:AugmentedLSTMCellModelParametersThematicallyArranged},
the total partial derivative of the objective function, $E$, with
respect to the model parameters, $\Theta$, at the step with the index,
$n$, is:
\begin{align}
\begin{gathered}\frac{\partial E}{\partial\Theta}[n]=\end{gathered}
 & \begin{Bmatrix}\frac{\partial E}{\partial W_{x_{cu}}[l]}[n], & \frac{\partial E}{\partial W_{s_{cu}}}[n], & \frac{\partial E}{\partial W_{v_{cu}}}[n], & \frac{\partial E}{\partial\vec{b}_{cu}}[n],\\
\frac{\partial E}{\partial W_{x_{cs}}[l]}[n], & \frac{\partial E}{\partial W_{s_{cs}}}[n], & \frac{\partial E}{\partial W_{v_{cs}}}[n], & \frac{\partial E}{\partial\vec{b}_{cs}}[n],\\
\frac{\partial E}{\partial W_{x_{cr}}[l]}[n], & \frac{\partial E}{\partial W_{s_{cr}}}[n], & \frac{\partial E}{\partial W_{v_{cr}}}[n], & \frac{\partial E}{\partial\vec{b}_{cr}}[n],\\
\frac{\partial E}{\partial W_{x_{cx}}[l]}[n] & \frac{\partial E}{\partial W_{s_{cx}}}[n] & \frac{\partial E}{\partial W_{v_{cx}}}[n] & \frac{\partial E}{\partial\vec{b}_{cx}}[n]\\
\frac{\partial E}{\partial W_{x_{du}}[l]}[n], &  & \frac{\partial E}{\partial W_{v_{du}}}[n], & \frac{\partial E}{\partial\vec{b}_{du}}[n]\\
 &  & \frac{\partial E}{\partial W_{q_{dr}}}[n]
\end{Bmatrix}.\label{eq:AugmentedLSTMCellDelEDelModelParametersStepTotalThematicallyArranged}
\end{align}
Finally, $\frac{dE}{d\Theta}$, the total derivative of the objective
function, $E$, with respect to the model parameters, $\Theta$, for
the entire unrolled sequence is computed by Equation \ref{eq:VanillaLSTMCellDEDModelParametersSegmentTotal}.
Aggregated over a batch of segments, $\frac{dE}{d\Theta}$ is plugged
in to the Gradient Descent training algorithm for learning the model
parameters, $\Theta$.

\section{Conclusions and Future Work\label{sec:Conclusions-and-Future-Work}}

In this paper, we presented the fundamentals of the RNN and the LSTM
network using a principled approach. Starting with the differential
equations encountered in many branches of science and engineering,
we showed that the canonical formulation of the RNN can be obtained
by sampling delay differential equations used to model processes in
physics, life sciences, and neural networks. We proceeded to obtain
the standard RNN formulation by appropriately choosing the parameters
of the canonical RNN equations and applying stability considerations.
We then formally explained RNN unrolling within the framework of approximating
an IIR system by an FIR model and proved the sufficient conditions
for its applicability to learning sequences. Next, we presented the
training of the standard RNN using Back Propagation Through Time,
segueing to the review of the vanishing and exploding gradients, the
well-known numerical difficulties, associated with training the standard
RNN by Gradient Descent. We subsequently addressed the shortcomings
of the standard RNN by morphing the canonical RNN system into the
more robust LSTM network through a series of extensions and embellishments.
In addition to the logical construction of the Vanilla LSTM network
from the canonical RNN, we included a self-contained overview of the
Vanilla LSTM network, complete with the specifications of all principal
entities as well as clear, descriptive, yet concise, presentations
of the forward pass and, importantly, the backward pass, without skipping
any steps. The main contribution up to this point has been our unique
pedagogical approach for analyzing the RNN and Vanilla LSTM systems
from the Signal Processing perspective, a formal derivation of the
RNN unrolling procedure, and a thorough treatment using a descriptive
and meaningful notation, aimed at demystifying the underlying concepts.
Moreover, as an unexpected benefit of our analysis, we identified
two novel extensions to the Vanilla LSTM network: the convolutional
non-causal input context windows and the external input gate. We then
augmented the equations for the LSTM cell with these extensions (along
with the recurrent projection layer, previously introduced by another
researcher team). Candidate recommendations for future work include
implementing the Augmented LSTM system within a high-performance computing
environment and benchmarking its efficacy in multiple practical scenarios.
The use of the Augmented LSTM could potentially benefit the language
representation subsystems used in question answering and in automating
customer support. For these applications, it will be important to
evaluate the performance impact, attributed to the non-causal input
context windows, as compared to the different baselines, such the
Vanilla LSTM network, the bi-directional LSTM network, the Transformer,
and other state-of-the-art models. Also of particular relevance to
this use case will be to measure the effectiveness of the external
input gate in helping to eliminate the non-essential content from
the input sequences. Finally, adopting the Augmented LSTM network
to other practical domains and publishing the results is respectfully
encouraged.

\section*{Acknowledgments\label{sec:Acknowledgments}}

The author thanks Eugene Mandel for long-time collaboration and engaging
discussions, which were instrumental in clarifying the key concepts,
as well as for his encouragement and support. Tremendous gratitude
is expressed to Tom Minka for providing helpful critique and valuable
comments and to Mike Singer for reviewing the proposition and correcting
errors in the proof. Big thanks go to Varuna Jayasiri, Eduardo G.
Ponferrada, Flavia Sparacino, and Janet Cahn for proofreading the
manuscript.\pagebreak{}

\vphantom{}

\bibliographystyle{plain}
\nocite{*}
\bibliography{12_Users_alexsherstinsky_Documents_RNNAndLSTMFu___stems2020_Revision_arXiv07312023_references}

\end{document}